\documentclass[twosided]{article}
\usepackage[T1]{fontenc}
\usepackage{lmodern}

\title{Non-asymptotic and Accurate \\Learning of Nonlinear Dynamical Systems}
\author{Yahya~Sattar\thanks{Department of Electrical and Computer Engineering, University of California, Riverside, CA 92521, USA. Email: ysatt001@ucr.edu, oymak@ece.ucr.edu.} \and Samet~Oymak\footnotemark[1] 
}

\usepackage{hyperref,graphicx,amsmath,amsfonts,amssymb,bm,url,breakurl,epsfig,epsf,color,fullpage,MnSymbol,mathbbol,fmtcount,semtrans,cite,caption,subcaption,multirow,comment}
\usepackage[noend]{algpseudocode}

\usepackage{amssymb}

\usepackage[utf8]{inputenc} 
\usepackage[T1]{fontenc}    
\usepackage{url}            
\usepackage{booktabs}       
\usepackage{amsfonts}       
\usepackage{nicefrac}       
\usepackage{microtype}      

\makeatletter
\providecommand*{\boxast}{%
  \mathbin{
    \mathpalette\@boxit{*}%
  }%
}
\newcommand*{\@boxit}[2]{%
  \sbox0{$\m@th#1\Box$}%
  \ifx#1\displaystyle \ht0=\dimexpr\ht0+.05ex\relax \fi
  \ifx#1\textstyle \ht0=\dimexpr\ht0+.05ex\relax \fi
  \ifx#1\scriptstyle \ht0=\dimexpr\ht0+.04ex\relax \fi
  \ifx#1\scriptscriptstyle \ht0=\dimexpr\ht0+.065ex\relax \fi
  \sbox2{$#1\vcenter{}$}
  \rlap{%
    \hbox to \wd0{%
      \hfill
      \raisebox{%
        \dimexpr.5\dimexpr\ht0+\dp0\relax-\ht2\relax
      }{$\m@th#1#2$}%
      \hfill
    }%
  }%
  \Box
}
\makeatother

  \makeatletter
\def\BState{\State\hskip-\ALG@thistlm}
\makeatother

  \usepackage{mathtools}

\usepackage{titlesec}

\usepackage{tikz}
\usepackage{pgfplots}
\usetikzlibrary{pgfplots.groupplots}

\titleformat{\paragraph}
{\normalfont\normalsize\bfseries}{\theparagraph}{1em}{}
\titlespacing*{\paragraph}
{0pt}{3.25ex plus 1ex minus .2ex}{1.5ex plus .2ex}

\usepackage{movie15}

\usepackage{cite}
\usepackage{caption}
\usepackage[bottom,hang,flushmargin]{footmisc} 

\newcommand{\tsn}[1]{{\left\vert\kern-0.25ex\left\vert\kern-0.25ex\left\vert #1 
    \right\vert\kern-0.25ex\right\vert\kern-0.25ex\right\vert}}

\definecolor{darkred}{RGB}{150,0,0}
\definecolor{darkgreen}{RGB}{0,150,0}
\definecolor{darkblue}{RGB}{0,0,200}
\hypersetup{colorlinks=true, linkcolor=darkblue, citecolor=darkgreen, urlcolor=darkgreen}

\newtheorem{theorem}{Theorem}[section]

\newtheorem{assumption}{Assumption}

\newtheorem{lemma}[theorem]{Lemma}
\newtheorem{corollary}[theorem]{Corollary}

\newtheorem{definition}[theorem]{Definition}

\newtheorem{remark}[subsection]{Remark}

\newcommand{\vertiii}[1]{{\vert\kern-0.25ex\vert\kern-0.25ex\vert #1 \vert\kern-0.25ex\vert\kern-0.25ex\vert}}

\newcommand{\eps}{\varepsilon}

\newcommand{\beq}{\begin{equation}}

\newcommand{\eeq}{\end{equation}}

\newcommand{\var}[1]{{{\text{\bf{var}}}}[#1]}

\newcommand{\nn}{\nonumber}
\newcommand{\la}{\lambda}
\newcommand{\A}{{\mtx{A}}}

\newcommand{\bt}{\bteta}

\newcommand{\B}{{{\mtx{B}}}}

\newcommand{\Gb}{{\mtx{G}}}
\newcommand{\Fb}{{\mtx{F}}}

\newcommand{\Lc}{{\cal{L}}}
\newcommand{\Lch}{\hat{{\cal{L}}}}
\newcommand{\Lcp}{{\cal{L}}_{\cal{D}}}

\newcommand{\Ltr}{\hat{{\cal{L}}}^{\text{tr}}}

\newcommand{\ltr}{\hat{{\ell}}^{\text{tr}}}

\newcommand{\Dc}{{\cal{D}}}

\newcommand{\bSi}{{\boldsymbol{{\Sigma}}}}

\newcommand{\bmu}{{\boldsymbol{{\mu}}}}

\newcommand{\Iden}{{\mtx{I}}}
\newcommand{\M}{{\mtx{M}}}

\newcommand{\order}[1]{{\cal{O}}(#1)}

\newcommand{\lmn}[1]{{\lambda_{\min}(#1)}}
\newcommand{\lmx}[1]{{\lambda_{\max}(#1)}}
\newcommand{\z}{{\vct{z}}}

\newcommand{\tn}[1]{\|{#1}\|_{\ell_2}}
\newcommand{\tinf}[1]{\|{#1}\|_{\ell_\infty}}

\newcommand{\bad}{{\bar{d}}}

\newcommand{\tf}[1]{\|{#1}\|_{F}}
\newcommand{\te}[1]{\|{#1}\|_{\psi_1}}
\newcommand{\tsub}[1]{\|{#1}\|_{\psi_2}}

\newcommand{\Cc}{\mathcal{C}}

\newcommand{\Rc}{\mathcal{R}}

\newcommand{\bdel}{\boldsymbol{\delta}}
\newcommand{\bGma}{\boldsymbol{\Gamma}}
\newcommand{\omg}{\boldsymbol{\omega}}
\newcommand{\bteta}{\boldsymbol{\theta}}
\newcommand{\bTeta}{\boldsymbol{\Theta}}

\newcommand{\Bc}{\mathcal{B}}
\newcommand{\Sc}{\mathcal{S}}

\newcommand{\Nc}{\mathcal{N}}

\newcommand{\vb}{\vct{v}}

\newcommand{\vh}{\vct{\hat{v}}}

\newcommand{\w}{\vct{w}}

\newcommand{\wh}{{\hat{\mtx{w}}}}
\newcommand{\li}{\left<}
\newcommand{\ri}{\right>}

\newcommand{\ab}{\vct{a}}

\newcommand{\bb}{\vct{b}}
\newcommand{\ub}{{\vct{u}}}

\newcommand{\h}{\vct{h}}
\newcommand{\htt}{\tilde{\h}}
\newcommand{\bh}{\bar{\vct{h}}}

\newcommand{\g}{{\vct{g}}}
\newcommand{\q}{{\vct{q}}}
\newcommand{\hq}{\hat{\vct{q}}}

\newcommand{\distas}{\overset{\text{i.i.d.}}{\sim}}

\newcommand{\leqsym}[1]{\stackrel{\text{(#1)}}{\leq}}

\newcommand{\eqsym}[1]{\stackrel{\text{(#1)}}{=}}

\newcommand{\by}{\bar{\y}}

\newcommand{\x}{\vct{x}}

\newcommand{\y}{\vct{y}}

\newcommand{\bgl}{{~\big |~}}

\definecolor{emmanuel}{RGB}{255,127,0}

\newcommand{\Kb}{{\mtx{K}}}

\newcommand{\R}{\mathbb{R}}
\newcommand{\Pro}{\mathbb{P}}

\renewcommand{\P}{\operatorname{\mathbb{P}}}
\newcommand{\E}{\operatorname{\mathbb{E}}}

\newcommand{\grad}[1]{{\nabla\Lc(#1)}}

\newcommand{\e}{\mathrm{e}}

\newcommand{\vct}[1]{\bm{#1}}
\newcommand{\mtx}[1]{\bm{#1}}

\numberwithin{equation}{section} 

\def \endprf{\hfill {\vrule height6pt width6pt depth0pt}\medskip}
\newenvironment{proof}{\noindent {\bf Proof} }{\endprf\par}

\newcommand{\red}{\textcolor{black}}

\usepackage{enumitem}

\usepackage{multirow}
\usepackage{algorithm}
\usepackage{algpseudocode}

\newcommand{\RED}[1]{\textcolor{black}{#1}}

\newcommand{\btetas}{{\boldsymbol{\theta}_\star}}
\newcommand{\As}{{\boldsymbol{A}_\star}}
\newcommand{\Aps}{{\boldsymbol{A'}_\star}}
\newcommand{\Bs}{{\boldsymbol{B}_\star}}
\newcommand{\btetat}{\tilde{{\boldsymbol{\theta}}}}
\newcommand{\btetah}{\hat{{\boldsymbol{\theta}}}}

\newcommand{\bheta}{{\boldsymbol{\theta}}}

\newcommand{\bhetas}{{\boldsymbol{\theta}_\star}}
\newcommand{\bbalpha}{{\boldsymbol{\alpha}}}
\newcommand{\bphi}{\tilde{\phi}}

\newcommand{\Bheta}{{\boldsymbol{\Theta}}}

\newcommand{\Bhetas}{{\boldsymbol{\Theta}_\star}}

\date{}
\begin{document}
\maketitle
\begin{abstract}
	We consider the problem of learning \red{nonlinear dynamical} systems governed by nonlinear state equation $\h_{t+1}=\phi(\h_t,\ub_t;\bteta)+\w_t$. Here $\bteta$ is the unknown system dynamics, $\h_t $ is the state, $\ub_t$ is the input and $\w_t$ is the additive noise vector. We study gradient based algorithms to learn the system dynamics $\bteta$ from samples obtained from a single finite trajectory. \red{If the system is run by a stabilizing input policy, then using a mixing-time argument we show that temporally-dependent samples can be approximated by i.i.d.~samples}. We then develop new guarantees for the uniform convergence of the gradients of the empirical loss \red{induced by these i.i.d.~samples}. Unlike existing works, our bounds are noise sensitive which allows for learning ground-truth dynamics with high accuracy and small sample complexity. Together, our results facilitate efficient learning of \red{a broader class of nonlinear dynamical systems as compared to the prior works}. 
We specialize our guarantees to  entrywise nonlinear activations and verify our theory in various numerical experiments. 
\end{abstract}

\section{Introduction}
Dynamical systems are fundamental for modeling a wide range of problems appearing in complex physical processes, cyber-physical systems and machine learning. Contemporary neural network models for processing sequential data, such as recurrent networks and LSTMs, can be interpreted as nonlinear dynamical systems and establish state-of-the-art performance in machine translation and speech recognition  \cite{li2013accelerating,mikolov2010recurrent,graves2013speech,sak2014long,bahdanau2014neural}.
Classical optimal control literature heavily relies on modeling the underlying system as a linear dynamical system (LDS) to synthesize control policies leading to elegant solutions such as PID controller and Kalman filter \cite{aastrom1995pid,welch1995introduction,ho1966effective}. In many of these problems, we have to estimate or approximate the system dynamics from data, either because the system is initially unknown or because it is time-varying. This is alternatively known as the system identification problem which is the task of learning an unknown system from the time series of its trajectories \cite{aastrom1971system,chen1990non,hochreiter1997long,ljung1999system,pintelon2012system}.

In this paper, we aim to learn the dynamics of nonlinear systems which are governed by following state equation,
\begin{align}
	\h_{t+1}=\phi(\h_t,\ub_t;\btetas)+\w_t, \label{state eqn intro}
\end{align} 
where $\btetas \in \R^d$ is the system dynamics, $\h_t \in \R^n$ is the state vector, $\ub_t \in \R^p$ is the input and $\w_t \in \R^n$ is the additive noise at time $t$. Our goal is understanding the statistical and computational efficiency of gradient based algorithms for learning the system dynamics from a single finite trajectory. 

\noindent{\bf{Contributions:}} Although system identification is classically well-studied, obtaining non-asymptotic sample complexity bounds is challenging especially when it comes to nonlinear systems. We address this challenge by relating the system identification problem (which has temporally dependent samples) to classical statistical learning setup where data is independent and identically distributed (i.i.d). We build on this to provide the following contributions.

\noindent $\bullet$ {\bf{Learning nonlinear systems via gradient descent:}} We work with (properly defined) \red{stable} nonlinear systems and use \red{stability} in conjunction with mixing-time arguments to address the \red{problem of learning the system dynamics from a single finite trajectory}. Under proper and intuitive assumptions, this leads to sample complexity and convergence guarantees for learning nonlinear dynamical systems \eqref{state eqn intro} via gradient descent. Unlike the related results on nonlinear systems \cite{oymak2019stochastic,bahmani2019convex}, our analysis accounts for the noise, achieves optimal dependence and applies to a \red{broader class of} nonlinear systems.

\noindent $\bullet$ {\bf{Accurate statistical learning:}} Of independent interest, we develop new statistical guarantees for the uniform convergence of the gradients of the empirical loss. Improving over earlier works~\cite{mei2018landscape,foster2018uniform}, our bounds \red{properly capture the noise dependence} and allows for learning the ground-truth dynamics with high accuracy and small sample complexity (see \S\ref{sec:accu SLWG} for further discussion).

\noindent $\bullet$ {\bf{Applications:}} We specialize our results by establishing theoretical guarantees for learning linear~($\h_{t+1} = \As\h_t + \Bs\ub_t+\w_t$) as well as nonlinear~($\h_{t+1} = \phi(\Bhetas\h_t) +\z_t +\w_t$) dynamical systems via gradient descent which highlight the optimality of our guarantees. Lastly, we verify our theoretical results through various numerical experiments with nonlinear activations.

{\bf{Organization:}} We introduce the problem under consideration in \S\ref{sec:math_prelim} and provide uniform convergence guarantees for empirical gradients in \S\ref{sec:accu SLWG}. We relate the gradients of single trajectory loss and multiple trajectory loss in~\S\ref{sec: truncated results}. Our main results on learning nonlinear systems are presented in \S\ref{sec 5} and applied to two special cases in \S\ref{sec:applications}. \S\ref{sec:numerical_exp} provides numerical experiments to corroborate our theoretical results. \S\ref{sec:related} discusses the related works and finally \S\ref{sec:conclusion} concludes the paper. Finally, the proofs of our main results are provided in Appendices~\ref{sec: proofs main} and \ref{sec: appenix A}.

{\bf{Notations:}} We use boldface uppercase (lowercase) letters to denote matrices (vectors). For a vector $\vb$, we denote its Euclidean norm by $\tn{\vb}$. For a matrix $\M$, $\rho(\M),\|\M\|$ and $\tf{\M}$ denote the spectral radius, spectral norm and Frobenius norm respectively. $c,c_0,c_1,\dots, C,C_0$ denote positive absolute constants. $\Sc^{d-1}$ denotes the unit sphere while $\Bc^d(\ab,r)$ denotes the Euclidean ball of radius $r$, centered at $\ab$, in $\R^d$. The normal distribution is denoted by $\Nc(\mu,\sigma^2)$. For a random vector $\vb$, we denote its covariance matrix by $\bSi[\vb]$. We use $\gtrsim$ and $\lesssim$ for inequalities that hold up to a constant factor. We denote by $a \lor b$, the maximum of two scalars $a$ and $b$. Similarly, $a \land b$ denotes the minimum of the two scalars. Given a number $a$, $\lfloor a \rfloor$ denotes the largest integer less than or equal to $a$, whereas, $\lceil a \rceil$ denotes the smallest integer greater than or equal to $a$.

\section{Problem Setup}\label{sec:math_prelim}
We assume the system is driven by inputs $\ub_t=\mtx{\pi}(\h_t)+\z_t$, where $\mtx{\pi}(\cdot)$ is a fixed control policy and $\z_t$ is excitation for exploration. For statistical analysis, we assume the excitation and noise are random, that is, $(\z_{t})_{t\geq 0}\distas \Dc_z$ and $(\w_{t})_{t\geq 0} \distas \Dc_w$ for some distributions $\Dc_z$ and $\Dc_w$. With our choice of inputs, the state equation~\eqref{state eqn intro} becomes,
\begin{align}
	\h_{t+1}&=\phi(\h_t,\mtx{\pi}(\h_t)+\z_t;\btetas)+\w_t :=\bphi(\h_t,\z_t;\btetas)+\w_t, \label{general state eqn}
\end{align}
where $\bphi$ denotes the closed-loop nonlinear system. \red{Throughout, we assume the nonlinear functions $\phi(\cdot,\cdot;\bt)$ and $\bphi(\cdot,\cdot;\bt)$ are differentiable in $\bt$. For clarity of exposition, we will not explicitly state this assumption when it is clear from the context.} A special case of~\eqref{general state eqn} is a linear state equation with $\btetas=[\As~\Bs]$, $\mtx{\pi}(\h_t)=-\Kb \h_t$ and
\begin{align}
	\h_{t+1}=(\As-\Bs\Kb)\h_t+\Bs\z_t+\w_t,\label{linear state eqn}
\end{align}
To analyze~\eqref{general state eqn} in a non-asymptotic setup, we assume access to a finite trajectory $(\h_{t},\z_{t})_{t=0}^{T-1}$ generated by unknown dynamics $\btetas$. Towards estimating $\btetas$, we formulate an empirical risk minimization (ERM) problem over single finite trajectory as follows,
\begin{equation}\label{eqn:minimization problem}
	\begin{aligned}
		&\btetah = \arg \min_{\bteta \in \R^d}\hat{\Lc}(\bteta), \quad \text{subject to} \quad \hat{\Lc}(\bteta) :=\frac{1}{2(T-L)}\sum_{t=L}^{T-1}\tn{\h_{t+1}-\bphi(\h_t,\z_t;\bteta)}^2, 
	\end{aligned}
\end{equation}
where $L \geq 1$ is a churn period which is useful for simplifying the notation later on, as $L$ will also stand for the approximate mixing-time of the system. To solve \eqref{eqn:minimization problem}, we investigate the properties of the gradient descent algorithm, given by the following iterate
\begin{align}	
	\bteta_{\tau+1} &= \bteta_{\tau} - \eta \nabla{\hat{\Lc}(\bteta_{\tau})}, \label{eqn:gradient descent} 
\end{align}
where $\eta >0$ is the fixed learning rate. ERM with i.i.d.~samples is a fairly well-understood topic in classical machine learning. However, samples obtained from a single trajectory of a dynamical system are temporally dependent. For \red{stable} systems~(see Def.~\ref{def:non-linear stability}), it can be shown that this dependence decays exponentially over the time. Capitalizing on this, we show that one can obtain almost i.i.d.~samples from a given trajectory $(\h_{t},\z_{t})_{t=0}^{T-1}$. This will in turn allow us to leverage techniques developed for i.i.d.~data to solve problems with sequential data.
\subsection{Assumptions on the System and the Inputs}
We assume that \red{the closed-loop system $\bphi$ is stable. Stability} in linear dynamical systems is connected to the spectral radius of the closed-loop system~\cite{krauth2019finite,simchowitz2018learning}. The definition below provides a natural generalization of stability to nonlinear systems. 
\begin{definition}[$(C_\rho,\rho)$-stability] \label{def:non-linear stability}
	Given excitation $(\z_t)_{t\geq 0}$ and  noise $(\w_t)_{t\geq 0}$, denote the state sequence~\eqref{general state eqn} resulting from initial state $\h_0=\bbalpha$, $(\z_{\tau})_{\tau = 0}^{t-1}$ and $(\w_{\tau})_{\tau = 0}^{t-1}$ by $\h_t(\bbalpha)$. Let $C_\rho \geq 1$ and $\rho \in (0,1)$ be system related constants. We say that the closed loop system $\bphi$ is $(C_\rho,\rho)$-stable if, for all $\bbalpha$, $(\z_t)_{t \geq 0}$ and $(\w_t)_{t\geq 0}$ triplets, we have 
	\begin{align}
		\tn{\h_t(\bbalpha)-\h_t(0)}\leq C_\rho \rho^t\tn{\bbalpha}.
	\end{align}
\end{definition}
Def.~\ref{def:non-linear stability} is a generalization of the standard notion of stability in the case of LDS. For a stable LDS~($\rho(\As)<1$), as a consequence of Gelfand's formula, there exists $C_\rho \geq 1$ and $\rho \in (\rho(\As),1)$ such that $(C_\rho,\rho)$-stability holds~(see \S\ref{sec:proofs of LDS}). 
\RED{A concrete example of nonlinear stable system is a contractive system where $\bphi$ is $\rho$-Lipschitz function of $\h_t$ for some $\rho<1$. We remark that, our interest in this work is not verifying the stability of a nonlinear system, but using stability of the closed-loop nonlinear system as an ingredient of the learning process. Verifying stability of the nonlinear systems can be very challenging, however, system analysis frameworks such as integral quadratic constraints (IQC)~\cite{megretski1997system} and Sum of Squares \cite{papachristodoulou2005tutorial,prajna2002introducing} may provide informative bounds.}
\begin{assumption}[Stability] \label{assum stability} 
	The closed-loop system $\bphi$ is $(C_\rho,\rho)$-stable for some $\rho<1$. 
\end{assumption}
\red{Assumption 1 implies that the closed-loop system forgets a past state exponentially fast. This is different from the usual notion of ``exponential Lyapunov stability'' which claims the exponential convergence to a point in state space. On the other hand, in the case of $(C_\rho,\rho)$-stability, the trajectories $\h_t(\bbalpha)$ and $\h_t(0)$ do not have to converge, rather their difference $\tn{\h_t(\bbalpha) - \h_t(0)}$ exponentially converges to zero~(assuming $\tn{\bbalpha}$ is bounded).} 
To keep the exposition simple, we will also assume $\h_0=0$ throughout. For data driven guarantees, we will make use of the following independence and boundedness assumptions on excitation and noise.
%
%
\red{
	\begin{assumption}[Boundedness] \label{ass bounded} There exist scalars $B,\sigma >0$, such that $(\z_t)_{t\geq 0}\distas \Dc_z$ and $(\w_t)_{t\geq 0}\distas \Dc_w$ obey $\tn{\bphi(0,\z_t;\btetas)}\leq B\sqrt{n}$ and $\tinf{\w_t} \leq \sigma$ for $0 \leq t \leq T-1$ with probability at least $1-p_0$ over the generation of data.
\end{assumption}}
\subsection{Optimization Machinery}
To concretely show how \red{stability} helps, we define the following loss function, obtained from i.i.d. samples at time $L-1$ and can be used as a proxy for $\E[\hat{\Lc}]$.
\begin{definition}[Auxiliary Loss]\label{def:TPL}
	Suppose $\h_{0} = 0$. Let $(\z_t)_{t\geq 0}\distas \Dc_z$ and $(\w_t)_{t\geq 0}\distas \Dc_w$. The auxiliary loss is defined as the expected loss at timestamp $L-1$, that is, 
	\begin{equation}\label{eqn:TPL}
		\begin{aligned}
			&\Lcp(\bteta) = \E[\Lc(\bteta,(\h_L,\h_{L-1},\z_{L-1}))], \\
			\text{where} \quad &\Lc(\bteta,(\h_L,\h_{L-1},\z_{L-1})) :=  \frac{1}{2}\tn{\h_L - \bphi(\h_{L-1},\z_{L-1};\bteta)}^2.
		\end{aligned}
	\end{equation}
\end{definition}
Our generic system identification results via gradient descent will utilize the one-point convexity hypothesis. This is a special case of Polyak-Łojasiewicz inequality and provides a generalization of strong convexity to nonconvex functions.
\begin{assumption}[One-point convexity~(OPC)  \& smoothness] \label{assum grad dominance} 
	There exist scalars $\beta \geq \alpha>0, r>0$ such that, \red{for all $\bt \in \Bc^d(\btetas,r)$}, the auxiliary loss $\Lcp(\bteta)$ of Definition~\ref{def:TPL} satisfies
	\begin{align}
		&\li\bteta-\btetas,\nabla{\Lcp(\bteta)}\ri\geq\alpha \tn{\bteta-\btetas}^2, \label{assump 1a}\\
		&\tn{\nabla{\Lcp(\bteta)}}\leq \beta\tn{\bteta-\btetas}. \label{assump 1b}
	\end{align}
\end{assumption}
\red{A concrete example of a nonlinear system satisfying OPC is the nonlinear state equation $\h_{t+1} = \phi(\Bhetas \h_{t}) + \z_t + \w_t$, with $\gamma$-increasing activation~(i.e.~$\phi'(x)\geq \gamma>0$ for all $x\in\R$. See Lemma~\ref{lemma:grad dominance nonlinear} for detail). We expect many activations including ReLU to work as well. The main challenge is verifying OPC of the population loss. For ReLU, Lemma 6.1 of~\cite{mohammadreza2019fitting} shows this property for i.i.d. Gaussian features. Extending this to subgaussian features, would yield the ReLU result. 
	There are many other interesting directions for which OPC can be verified, e.g., nonlinear ARX form $\h_{t} = \phi(\A_1\h_{t-1} + \A_2\h_{t-2} +\cdots+\A_{m}\h_{t-m})+\w_{t-1}$ 
	with $\gamma$-increasing activation, where $m \geq 1$ is called the order of the nonlinear ARX form. Specifically, setting $\x_{L-1} := [\h_{L-1}^T~\h_{L-2}^T~\cdots~\h_{L-m}^T]^T \in \R^{mn}$ and $\bTeta_\star := [\A_1~\A_2~\cdots~\A_m] \in \R^{n \times mn}$, we have $\h_{L} = \phi(\Bhetas\x_{L-1})+\w_{L-1}$. To verify the OPC for the nonlinear ARX form, denoting the $k_{\rm th}$ row of $\Bhetas$ by $\bheta_k^{\star \top}$, the auxiliary loss is given by,
	\begin{align*}
		\Lcp(\Bheta) = \sum_{k=1}^n \Lc_{k,\Dc}(\bheta_k) \quad \text{where}\quad \Lc_{k,\Dc}(\bheta_k) := \frac{1}{2}\E[(\h_L[k] -\phi(\bheta_k^\top\x_{L-1}))^2]. 
	\end{align*}
	Verifying the OPC of $\Lc_{k,\Dc}$ is equivalent to showing that the covariance matrix $\bSi[\x_{L-1}]$ is positive-definite. 
	If $\bSi[\x_{L-1}] \succeq \sigma_x^2\Iden_{mn}$ and $\phi$ is $\gamma$-increasing, then using a similar proof strategy of Lemma~\ref{lemma:grad dominance nonlinear}, it is straightforward to show that, for all $1\leq k \leq n$, the auxiliary losses $\Lc_{k,\Dc}(\bheta_k)$ satisfy the following OPC bound,
	\begin{align*}
		\li\bheta_k-\bheta_k^\star,\nabla{\Lc_{k,\Dc}(\bheta_k)}\ri &\geq \gamma^2 \sigma_x^2 \tn{\bheta_k-\bheta_k^\star}^2.
	\end{align*}
	Lastly, if the goal were only to show convergence to a local minima, we believe the one-point convexity assumption might be replaced with a less stronger assumption at the expense of slower learning. Specifically, similar to~\cite{li2017convergence}, we may analyze a two stage convergence of Gradient descent. In the first stage the gradient might point to the wrong direction, however, a potential function $g$ gradually decreases. Then, in the second stage, Gradient descent enters a nice one-point convex region and converges.}

\red{To proceed, if the gradient of $\hat{\Lc}(\bt)$ is close to that of $\Lcp(\bt)$ and Assumption~\ref{assum grad dominance} holds, gradient descent converges to the population minimum up to a statistical error governed by the noise level.
	\begin{theorem}[Informal result] \label{thrm:meta} Consider the state equation~\eqref{general state eqn}. Suppose Assumptions \ref{assum stability} and \ref{assum grad dominance} hold. Let $C,C_0,\xi,\xi_0>0$ be system related constants and $\sigma,\sigma_0>0$ denote the noise levels. Assume for all $\bt\in\Bc^d(\btetas,r)$, $\nabla\hat{\Lc}$ satisfies 
		\begin{align}
			\tn{\nabla \Lch(\bt)-\nabla \Lcp(\bt)}\leq \underbrace{C_0(\sigma_0+\xi_0\tn{\bt-\btetas}) \sqrt{d/N}}_{\text{finite sample error}} + \underbrace{C(\sigma+\xi\tn{\bteta-\btetas})C_\rho\rho^{L-1}}_{\text{single trajectory error}}
		\end{align}
		with high probability. Let $N = \lfloor(T-L)/L\rfloor$, where we pick $L$ via 
		\begin{align}
			L = \big\lceil 1 + \frac{\log((CC\rho/C_0)\sqrt{N/d}(\sigma/\sigma_0 \lor \xi/\xi_0))}{1-\rho} \big\rceil.
		\end{align}
		Suppose $N \gtrsim \xi_0^2 C_0^2d/\alpha^2$. Given $r>0$, set the learning rate $\eta = \alpha/(16\beta^2)$ and pick $\bheta_0 \in \Bc^d(\bhetas,r)$. Assuming $\sigma_0 \lesssim r\xi_0$, with high probability, all gradient descent iterates $\bheta^{(\tau)}$ on $\Lch$ satisfy
		\begin{align}
			\tn{\bt_\tau-\btetas}\leq (1-\frac{\alpha^2}{128\beta^2})^{\tau}\tn{\bt_0-\btetas}+\frac{5C_0\sigma_0}{\alpha}\sqrt{\frac{d}{N}}.
		\end{align}
	\end{theorem}
	Theorem \ref{thrm:meta} will be used in conjunction with uniform convergence of gradients to provide finite sample convergence and estimation guarantees.} For pure linear regression example (eq. \eqref{GLM} with identity $\phi$), it can be verified that this combination achieves the optimal error rate $\sigma \sqrt{d/N}$. Sections~\ref{sec:main results} and \ref{sec:application_to_RNN} accomplish this for more challenging setup of nonlinear systems. \red{In the next two sections, we provide a uniform convergence result for gradient of the empirical loss $\Lch(\bt)$ which will be used in conjunction with Theorem~\ref{thrm:meta} to provide our main result in \S\ref{sec 5}.}

\begin{figure}[t!]
	\centering
	\includegraphics[width=0.93\textwidth,page=6]{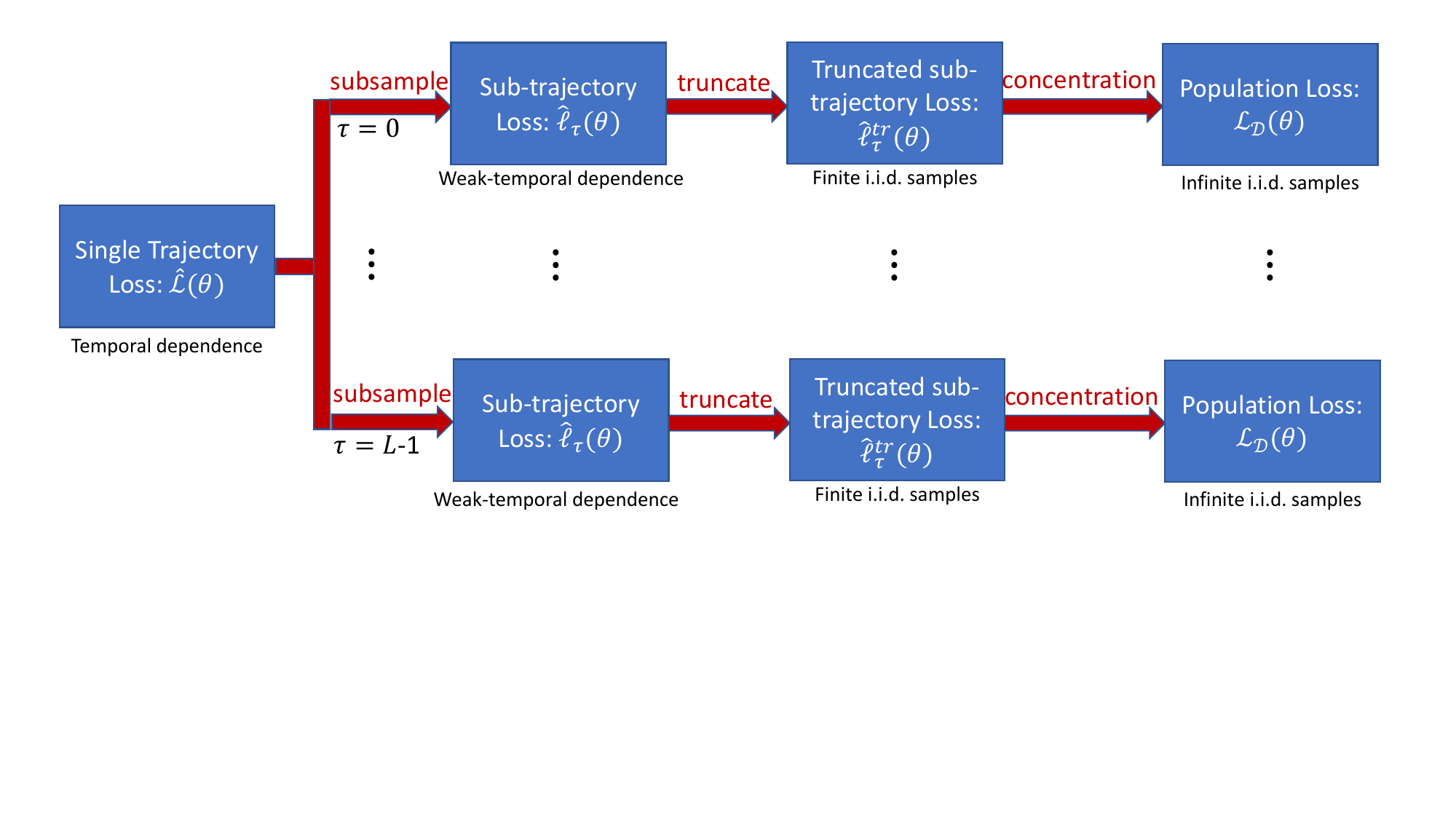}\vspace{-80pt}
	\caption{We learn nonlinear dynamical systems from a single trajectory by minimizing the empirical loss $\Lch(\bteta)$. The idea is to split $\Lch(\bteta)$ as an average of $L$ sub-trajectory losses as $\Lch(\bteta) = \frac{1}{L}\sum_{\tau=0}^{L-1}\hat{\ell}_\tau(\bteta)$, through shifting and sub-sampling. Observing that each sub-trajectory has weakly dependent samples because of stability, we use a mixing time argument to show that $\tn{\nabla\hat{\ell}_\tau(\bteta) - \nabla\ltr_\tau(\bteta)}\lesssim(\sigma+\xi\tn{\bteta-\btetas})C_\rho\rho^{L-1}$, where $\ltr_\tau(\bteta)$ is the loss constructed with finite i.i.d. samples~(\S \ref{sec: truncated results}). Next, we show the uniform convergence of the empirical gradient as $\tn{\nabla \ltr_\tau(\bt)-\nabla \Lcp(\bt)}\lesssim(\sigma_0+\xi_0\tn{\bt-\btetas}) \sqrt{d/N}$, where $\Lcp(\bt) = \E[\ltr_\tau(\bt)]$ is the population loss~(\S \ref{sec:accu SLWG}). Finally, we combine these with the local one-point convexity of the population loss to get our main results~(\S \ref{sec 5}).}  
	\label{fig3}
\end{figure}


\section{Accurate Statistical Learning with Gradient Descent}\label{sec:accu SLWG}
To provide finite sample guarantees, we need to characterize the properties of the empirical loss and its gradients. Towards this goal, this section establishes new gradient based statistical learning guarantees. Let ${\cal{S}}=(\x_i)_{i=1}^N$ be $N$ i.i.d.~samples from a distribution $\Dc$ and $\Lc(\cdot,\x)$ be a loss function that admits a sample $\x$ and outputs the corresponding loss. \red{When learning the nonlinear system~\eqref{general state eqn}, the sample $\x$ corresponds to the variables $(\h_{L},\h_{L-1},\z_{L-1})$ triple and the loss function $\Lc(\bteta,\x)$ is given by~\eqref{eqn:TPL}}. Define the empirical and population losses,
\begin{align}
	\Lch_{\cal{S}}(\bt)=\frac{1}{N}\sum_{i=1}^N \Lc(\bt,\x_i)\quad\text{and}\quad \Lc_{\Dc}(\bt)=\E[\Lc(\bt,\x)]. \label{eqn:LS_and_LD}
\end{align}
Let $\btetas$ denotes the population minimizer which we wish to estimate via gradient descent. Recent works by \cite{mei2018landscape} and \cite{foster2018uniform} provide finite sample learning guarantees via uniform convergence of the empirical gradient over a local ball $\Bc^d(\btetas,r)$. However these works suffer from two drawbacks which we address here. To contrast the results, let us consider the following toy regression problem which is a simplification of our original task \eqref{eqn:minimization problem}.

{\bf{Generalized linear model:}} Suppose labels $y_i$ are generated as, $y_i=\phi(\z_i^\top\btetas)+w_i$ for some activation $\phi:\R\rightarrow \R$ where $\z_i \in \R^d$ is the input, $w_i$ is the noise and $i=1,\dots,N$. Assume $N\gtrsim d$, $\z_i$ is zero-mean subgaussian vector with identity covariance and $w_i$ has variance $\sigma^2$. Consider the quadratic loss
\begin{align}
	\Lch_Q(\bt)=\frac{1}{2N} \sum_{i=1}^N(y_i-\phi(\z_i^\top\bt))^2.\label{GLM}
\end{align}
\begin{itemize}
	\item {\bf{The role of noise:}} Suppose $\phi$ is identity and the problem is purely linear regression. Gradient descent estimator will achieve statistical accuracy $\tn{\hat\bt-\btetas}\lesssim \sigma\sqrt{d/N}$. \red{\cite{mei2018landscape,foster2018uniform} yield the coarser bound $\tn{\hat\bt-\btetas}\lesssim (\sigma+rC)\sqrt{d/N}$ for some scalars $r, C>0$ coming from the uniform convergence of the empirical gradient over a local ball $\Bc(\btetas,r)$.} 
	\item {\bf{Activation $\phi$:}} Both \cite{mei2018landscape,foster2018uniform} can only handle bounded activation $\phi$. \cite{foster2018uniform} uses boundedness to control Rademacher complexity. For \cite{mei2018landscape} this is due to the subgaussian gradient requirement. On the other hand, even for pure linear regression, gradients are subexponential rather than subgaussian (as it involves $\z_i\z_i^\top$).
\end{itemize}
Below we address both of these issues. We restrict our attention to low-dimensional setup, however we expect the results to extend to sparsity/$\ell_1$ constraints in a straightforward fashion by adjusting covering numbers. In a similar spirit to \cite{mei2018landscape}, we study the loss landscape over a local ball $\Bc^d(\btetas,r)$. We first determine the conditions under which empirical and population gradients are close.
\begin{assumption}[Lipschitz gradients] \label{assume lip}There exist numbers $L_\Dc,p_0>0$ such that with probability at least $1-p_0$ over the generation of data, for all pairs $\bt,\bt'\in\Bc^d(\btetas,r)$, \red{the gradients of empirical and population losses in \eqref{eqn:LS_and_LD} satisfy}
	\begin{align}
		\max(\tn{\nabla\Lc_{\Dc}(\bt)-\nabla\Lc_{\Dc}(\bt')},\tn{\nabla\Lch_{\cal{S}}(\bt)-\nabla\Lch_{\cal{S}}(\bt')}) \leq L_\Dc\tn{\bt-\bt'}.
	\end{align}
\end{assumption}
The Lipschitz constant will only appear logarithmically in our bounds, hence, the assumption above is fairly mild. 
\begin{assumption} [Subexponential gradient noise]\label{assume grad} There exist scalars $K,\sigma_0 > 0$ such that, given $\x\sim \Dc$, at any point $\bt$, the subexponential norm of \red{the gradient of single sample loss $\Lc$ in \eqref{eqn:LS_and_LD}} is upper bounded as a function of the noise level $\sigma_0$ and distance to the population minimizer via
	\begin{align}
		\te{\grad{\bt,\x}-\E[\grad{\bt,\x}]}\leq \sigma_0 + K\tn{\bt-\btetas},
	\end{align}
	\red{where the subexponential norm of a random variable $X$ is defined as $\te{X}:= \sup_{k\geq 1} \frac{(\E[|X|^k])^{1/k}}{k}$ and that of a random vector $\x \in \R^n$ is defined as $\te{\x}: \sup_{\vb \in \Sc^{n-1}\te{\vb^\top\x}}$.}
\end{assumption}
This assumption is an improvement over the work of \cite{mei2018landscape} and will help us distinguish the gradient noise due to optimization ($K\tn{\bt-\btetas}$) and due to noise $\sigma_0$ at the population minima. 

\red{As an example, consider the quadratic loss in~\eqref{GLM}. In the case of linear regression~($\phi(x)=x$), it is easy to show that Assumption~\ref{assume lip} holds with $L_\Dc = 2$ and $p_0 = 2\exp(-100d)$, whereas, Assumption~\ref{assume grad} holds with $K=c$ and $\sigma_0 = c_0 \sigma$ for some scalars $c,c_0>0$. To verify Assumptions~\ref{assume lip} and \ref{assume grad} for the linear dynamical systems $\h_{t+1}=\As\h_t+\Bs\z_t+\w_t$, let $\Gb_{L-1}\Gb_{L-1}^\top$ and $\Fb_{L-1}\Fb_{L-1}^\top$ be the finite time controllability Gramians for the control and noise inputs respectively. Define $\gamma_+ := 1 \lor \lmx{\Gb_{L-1}\Gb_{L-1}^\top + \sigma^2 \Fb_{L-1}\Fb_{L-1}^\top}$. In Appendix~\ref{sec:proofs of LDS}, we show that, for linear dynamical systems~(using $\z_t \distas \Nc(0,\Iden_p)$ and $\w_t \distas \Nc(0,\sigma^2\Iden_n)$), Assumption~\ref{assume lip} holds with $L_\Dc = 2\gamma_+$ and $p_0 = 2\exp(-100(n+p))$ as long as $N \gtrsim n+p$, whereas, Assumption~\ref{assume grad} holds with $K = c\gamma_+$ and $\sigma_0 = c\sigma\sqrt{\gamma_+}$. Lastly, in Appendix~\ref{sec:proofs of RNN}, we show that in the case of nonlinear state equations $\h_{t+1}= \phi(\Bhetas\h_t)+\z_t+\w_t$, Assumptions~\ref{assume lip} and \ref{assume grad} hold as long as $\phi$ has bounded first and second derivatives, that is, $|\phi'(x)|, |\phi''(x)| \leq 1$ for all $x \in \R$. Specifically, using $\z_t \distas \Nc(0,\Iden_p)$ and $\w_t \distas \Nc(0,\sigma^2\Iden_n)$, if we bound the state covariance as $\bSi[\h_t] \preceq \beta_+^2\Iden_n$~(see the proof of Lemma~\ref{lemma:subexp grad nonlinear}), then Assumption~\ref{assume lip} holds with $L_\Dc = c((1+\sigma) \beta_+^2n +\tf{\Bhetas} \beta_+^3 n^{3/2} \log^{3/2}(2T))$ and $p_0 = 4T\exp(-100n)$, whereas, Assumption~\ref{assume grad} holds with $K = c\beta_+^2$ and $\sigma_0 = c\sigma\beta_+$.}

The next theorem establishes uniform concentration of the gradient as a function of the noise level and the distance from the population minima. To keep the exposition clean, from here on we set $C_{\log} = \log(3(L_\Dc N/K+1))$.
\begin{theorem}[Uniform gradient convergence] \label{refined thm}Suppose the gradients of $\Lcp$ and $\Lch_{\cal{S}}$ obey Assumptions \ref{assume lip} and \ref{assume grad}. Then, there exists $c_0>0$ such that, with probability at least $1-p_0-\log(\frac{Kr}{\sigma_0})\exp({-100d})$, for all $\bt\in\Bc^d(\btetas,r)$, we have
	\begin{align}
		&\tn{\nabla\Lch_{\cal{S}}(\bt)-\nabla\Lc_\Dc(\bt)} \leq c_0(\sigma_0+K\tn{\bt-\btetas}) C_{\log}\sqrt{\frac{d}{N}}.
	\end{align}
\end{theorem}
\red{\begin{proof}{\bf sketch:}
		Our proof technique uses peeling argument~\cite{geer2000empirical} to split the Euclidean ball $\Bc^d(\btetas,r)$ into $P+1$ sets $\{\Sc_i\}_{i=0}^P$. Given a set $\Sc_i \subset\Bc^d(\btetas,r)$ and the associated radius $r_i$, we pick an $\epsilon_i$-covering of the set $\Sc_i$. We then apply Lemma D.7 of \cite{oymak2018learning}~(by specializing it to unit ball) together with a union bound over the elements of $P+1$ covers, to guarantee uniform convergence of the empirical gradient over the elements of $P+1$ covers. Combining this with Assumption~\ref{assume lip}, we guarantee a uniform convergence of the empirical gradient to its population counterpart over all $\bt \in \Bc^d(\btetas,r)$.
	\end{proof}
	Theorem~\ref{refined thm} provides a refined control over the gradient quality in terms of the distance $\tn{\bt-\btetas}$. The reason why \cite{mei2018landscape,foster2018uniform} are getting coarser dependence on the noise level as compared to ours is their assumption that the gradient of the loss is subgaussian over all $\bt \in \Bc^d(\btetas,r)$ with subgaussian norm bounded by $\sigma + rC$, that is, there is a universal upper bound on the subgaussian norm of the gradient of the loss function over all $\bt \in \Bc^d(\btetas,r)$. On the other hand, we assume that the gradient of the loss is subexponential with subexponential norm bounded by $\sigma_0 + K\tn{\bt-\btetas}$. This enables us to use a peeling argument to split the Euclidean ball $\Bc^d(\btetas,r)$ into $P+1 = \lceil\log(Kr/\sigma_0)\rceil+1$ sets $\{\Sc_i\}_{i=0}^P$ and apply Lemma D.7 of ~\cite{oymak2018learning}, with union bound over the elements of $\epsilon_i$-coverings of the sets $\{\Sc_i\}_{i=0}^P$ to get the uniform convergence of the empirical gradient with high probability. Combining this with Assumption~\ref{assume lip}, we are able get a uniform convergence of empirical gradient result at any $\bt \in \Bc^d(\btetas,r)$ with improved accuracy.}

\red{To show the uniform convergence of the empirical gradient, \cite{mei2018landscape} requires the following assumptions on the gradient and the Hessian of the loss over all $\bt \in \Bc^d(\btetas,r)$: (a) the gradient of the loss is sub-gaussian, (b) the Hessian of the loss, evaluated on a unit vector, is sub-exponential, and (c) the Hessian of the population loss is bounded at one point. Comparing (a) with Assumption~\ref{assume grad}, we observe that Assumption~\ref{assume grad} is milder and is satisfied by a broader class of loss functions as compared to (a). For example, even for pure linear regression, the gradients are subexponential rather than subgaussian (as it involves $\z_i\z_i^T$). On the other hand, our uniform convergence result requires Assumption~\ref{assume lip} which might look restrictive. However, observe that the Lipschitz constant only appears logarithmically in our bounds, hence, Assumption~\ref{assume lip} is fairly mild.}

\red{Going back to the original problem \eqref{eqn:minimization problem}, observe that  Theorem~\ref{refined thm} bounds the impact of finite samples. In the next section, we provide bounds on the impact of learning from a single trajectory. Combining them relates the gradients of the auxiliary loss $\Lc_{\Dc}$ and the finite trajectory loss $\hat{\Lc}$ which will help learning $\btetas$ from finite data obtained from a single trajectory.}
\section{Learning from a Single Trajectory}\label{sec: truncated results}
\red{In this section we bound the impact of dependence in the data obtained from a single trajectory. For this purpose we use perturbation-based techniques to relate the gradients of the single trajectory loss $\Lch$ and the multiple trajectory loss $\Ltr$~(defined below). Before that, we introduce a few more concepts and definitions.
	\begin{definition}[Truncated state vector~\cite{oymak2019stochastic}]\label{def:truncated vector}
		Consider the state equation \eqref{general state eqn}. Suppose $\bphi(0,0;\bteta) = 0$, $\h_0 = 0$. Given $t\geq L>0$, $L$-truncation of $\h_t$ is denoted by $\h_{t,L}$ and is obtained by driving the system with excitations $\z_{\tau}'$ and additive noise $\w_{\tau}'$ until time $t$, where
		\begin{align}
			\z_{\tau}' = \begin{cases} {0} \; \text{if} \; \tau < t-L \\
				\z_{\tau} \; \text{else}
			\end{cases}, \quad \text{and} \quad \quad
			\w_{\tau}' = \begin{cases} {0} \; \text{if} \; \tau < t-L \\
				\w_{\tau} \; \text{else}
			\end{cases}.
		\end{align}
		In words, $L$-truncated state vector $\h_{t,L}$ is obtained by unrolling $\h_{t}$ until time $t-L$ and setting $\h_{t-L}$ to $0$.  
	\end{definition}
	The truncated state vector $\h_{t,L}$ is statistically identical to $\h_L$. Hence, using truncation argument we can obtain i.i.d. samples from a single trajectory which will be used to bound the impact of dependence in the data. At its core our analysis uses a mixing time argument based on contraction and is used by related works \cite{oymak2019stochastic,bahmani2019convex}. Truncated states can be made very close to the original states with sufficiently large truncation length. The difference between truncated and non-truncated state vectors is guaranteed to be bounded as
	\begin{align}
		\tn{\h_t-\h_{t,L}} \leq C_\rho\rho^L\tn{\h_{t-L}}. \label{eqn:truncation impact}
	\end{align}
	This directly follows from Definition~\ref{def:non-linear stability} and asserts that the effect of past states decreases exponentially with truncation length $L$. To tightly capture the effect of truncation, we also bound the Euclidean norm of states $\h_t$ as follows.
	\begin{lemma}[Bounded states]\label{lemma:bounded states}%
		Suppose Assumptions \red{\ref{assum stability} and \ref{ass bounded} hold}. Then, with probability at least $1-p_0$, we have $ \tn{\h_{t}} \leq \beta_+ \sqrt{n}$ for all $0 \leq t \leq T$, where $\beta_+:= C_\rho(\sigma+B)/(1-\rho)$.
	\end{lemma}
	Following this and~\eqref{eqn:truncation impact}, we can obtain weakly dependent sub-trajectories by properly sub-sampling a single trajectory $(\h_t,\z_t)_{t=0}^{T-1}$. For this purpose, we first define a sub-trajectory and its truncation as follows. 
	\begin{definition}[Truncated sub-trajectories~\cite{oymak2019stochastic}]\label{def:sub trajectory}
		Let sampling period $L \geq 1$ be an integer. Set the sub-trajectory length $N = \lfloor \frac{T-L}{L} \rfloor$. We sub-sample the trajectory $(\h_{t},\z_{t})_{t=0}^{T-1}$ at points $\tau+L, \tau+2L,\dots,\tau+N L$ and truncate the states by $L-1$ to get the $\tau_{th}$ truncated sub-trajectory $(\bh^{(i)},\z^{(i)})_{i=1}^{N}$, defined as
		\begin{align}
			(\bh^{(i)},\z^{(i)}):= (\h_{\tau+iL,L-1},\z_{\tau+iL})\quad\text{for}\quad i=1,\dots,N
		\end{align}
		where $0 \leq \tau \leq L-1$ is a fixed offset. 
	\end{definition}
	For notational convenience, we also denote the noise at time $\tau+iL$ by $\w^{(i)}$. The following lemma states that the  $\tau_{th}$ truncated sub-trajectory $(\bh^{(i)},\z^{(i)})_{i=1}^{N}$ has independent samples. 
	\begin{lemma}[Independence]\label{lemma:independence} 
		Suppose $(\z_{t})_{t=0}^\infty\distas \Dc_z$ and $(\w_{t})_{t=0}^\infty \distas \Dc_w$. Then, the $\tau_{th}$ truncated states $(\bh^{(i)})_{i=1}^{N}$ are all independent and are statistically identical to $\h_{L-1}$. Moreover, $(\bh^{(i)})_{i=1}^{N},(\z^{(i)})_{i=1}^{N},(\w^{(i)})_{i=1}^{N}$ are all independent of each other. 
	\end{lemma}
	For the purpose of analysis, we will define the loss restricted to a sub-trajectory and show that each sub-trajectory can have favorable properties that facilitate learning.
	\begin{definition}[Truncated sub-trajectory loss]\label{def:empirical loss}
		We define the truncated loss in terms of truncated~(sub-sampled) triplets $(\by^{(i)},\bh^{(i)},\z^{(i)})_{i=1}^{N} := (\h_{\tau+iL+1,L},\h_{\tau+iL,L-1},\z_{\tau+iL})_{i=1}^{N}$ as
		\begin{align}
			\ltr_\tau(\bteta) := \frac{1}{2N}\sum_{i=1}^{N}\tn{\by^{(i)}-\bphi(\bh^{(i)},\z^{(i)};\bteta)}^2. \label{eqn:ell_tr}	
		\end{align}
		Suppose we have access to $N$ i.i.d. trajectories of \eqref{general state eqn}. From each trajectory, we collect a sample at $t=L-1$ to obtain $\big(\h_L^{(i)},\h_{L-1}^{(i)},\z_{L-1}^{(i)}\big)_{i=1}^N$, where $(\h_L^{(i)},\h_{L-1}^{(i)},\z_{L-1}^{(i)})$ denotes the sample from $i$-th trajectory. Then the truncated loss in~\eqref{eqn:ell_tr} is statistically identical to,
		\begin{align}
			\ltr_\tau(\bteta) \equiv \frac{1}{2N}\sum_{i=1}^{N}\tn{\h_L^{(i)}-\bphi(\h_{L-1}^{(i)},\z_{L-1}^{(i)};\bteta)}^2.
		\end{align}
	\end{definition}
	Observe} that the auxiliary loss $\Lcp(\bteta) = \E[\ltr_\tau(\bteta)]$. Hence, $\ltr_\tau$ is a finite sample approximation of $\Lcp$ and we will use results from Section~\ref{sec:accu SLWG} to bound the Euclidean distance between them.
Before, stating our results on uniform convergence of empirical losses, we want to demonstrate the core idea regarding \red{stability}. For this purpose, we define the truncated loss which is truncated version of the empirical loss~\eqref{eqn:minimization problem}.
\begin{definition}[Truncated loss]\label{def:truncated loss}
	Let $\h_{t+1,L}=\bphi(\h_{t,L-1},\z_t;\btetas)+\w_t$. We define the truncated (empirical) risk as 
	\begin{align}
		\Ltr(\bteta) := \frac{1}{2(T-L)}\sum_{t=L}^{T-1}\tn{\h_{t+1,L}-\bphi(\h_{t,L-1},\z_{t};\bteta)}^2 = \frac{1}{L}\sum_{\tau=0}^{L-1}\ltr_\tau(\bteta).\label{eqn:truncated loss}
	\end{align}
\end{definition}
\red{Let $\cal H$ be the convex hull of all states $\h_t$ and $\cal Z$ be the convex hull of all the inputs $\z_t$ such that Assumptions \ref{assum stability} and \ref{ass bounded} are valid. As a regularity condition, we require the problem to behave nicely over state-excitation pairs $(\h,\z)\subset \cal H\times \cal Z$. Throughout, $\bphi_k$ denotes the scalar function associated to the $k_{\rm th}$ entry of $\bphi$.
} 

\red{The following theorem states that, in the neighborhood of $\btetas$, the empirical risk $\Lch$ behaves like the truncated risk $\Ltr$, when the approximate mixing-time $L$ is chosen sufficiently large.} 
\begin{theorem}[Small impact of truncation]\label{thrm:diff truncated vs actual} 
	Consider the state equation~\eqref{general state eqn}. Suppose Assumptions \ref{assum stability} and \ref{ass bounded} hold. Suppose there exists $r>0$ such that, for all $\bteta \in \Bc^d(\btetas,r)$ and for all $(\h,\z)\subset \cal H\times \cal Z$, we have that $\|\nabla_{\h}\bphi(\h,\z;\bteta)\| \leq B_{\bphi}$, $\tn{\nabla_{\bteta}{\bphi_k(\h,\z;\bteta)}} \leq C_{\bphi}$ and {$\|\nabla_{\h}\nabla_{\bteta}{\bphi_k(\h,\z;\bteta)}\| \leq D_{\bphi}$} for some scalars $B_{\bphi},C_{\bphi},D_{\bphi}>0$ and $1\leq k\leq n$. Let $\beta_+>0$ be as in Lemma~\ref{lemma:bounded states}. Then, with probability at least $1-p_0$, for all $\bteta \in \Bc^d(\btetas,r)$, we have
	\begin{align}
		|\Lch(\bteta) - \Ltr(\bteta)|  &\leq 2n\beta_+C_\rho\rho^{L-1}B_{\bphi}(\sigma + C_{\bphi}\tn{\bteta-\btetas}),\label{eqn:thrm_small_impact_a}  \\
		\tn{\nabla{\Lch(\bteta)}-\nabla\Ltr(\bteta)} &\leq 2n\beta_+ C_\rho\rho^{L-1}D_{\bphi}(\sigma+C_{\bphi}\tn{\bteta-\btetas}). \label{eqn:thrm_small_impact_b}
	\end{align}
\end{theorem}
\red{\begin{proof}{\bf sketch:}
		To prove Theorem~\ref{thrm:diff truncated vs actual}, we use the Mean-value Theorem together with Assumptions~\ref{assum stability} and \ref{ass bounded}. First, using~\eqref{eqn:minimization problem} and \eqref{eqn:truncated loss}, we obtain
		\begin{align}
			|\Lch(\bteta) - \Ltr(\bteta)| &\leq\frac{1}{2} \max_{L\leq t\leq (T-1)}| \tn{\bphi(\h_t,\z_t;\btetas)+\w_t-\bphi(\h_t,\z_t;\bteta)}^2 \nn \\
			&\quad\quad\quad\quad\quad-\tn{\bphi(\h_{t,L-1},\z_t;\btetas)+\w_t-\bphi(\h_{t,L-1},\z_t;\bteta)}^2|. \label{eqn:Lhat-Ltr}
		\end{align}
		Suppose, the maximum is achieved at $(\h,\bh,\z,\w)$~(where $\bh$ is the truncated state). Then, we use the identity $a^2-b^2 = (a+b)(a-b)$ to upper bound the difference $|\Lch(\bteta) - \Ltr(\bteta)|$ as a product of two terms $|a + b|$ and $|a-b|$ with $a:=\tn{\bphi(\h,\z;\btetas)+\w-\bphi(\h,\z;\bteta)}$ and $b:=\tn{\bphi(\bh,\z;\btetas)+\w-\bphi(\bh,\z;\bteta)}$. We upper bound the term $|a+b|$ by bounding each quantity $a$ and $b$ using the Mean-value Theorem together with Assumption~\ref{ass bounded}. Similarly, the term $|a-b|$ is upper bounded by first applying triangle inequality and then using the Mean-value Theorem together with Assumptions~\ref{assum stability} and \ref{ass bounded}~(to bound the difference $\tn{\h - \bh}$). Combining the two bounds gives us the statement~\eqref{eqn:thrm_small_impact_a} of the Theorem. A similar proof technique is used to upper bound the gradient distance $\tn{\nabla{\Lch(\bteta)}-\nabla\Ltr(\bteta)}$.
	\end{proof}
	Combining Theorems~\ref{refined thm} and \ref{thrm:diff truncated vs actual} allows us to upper bound the Euclidean distance between the gradients of the empirical loss $\Lch(\bt)$ and the auxiliary loss $\Lc_{\Dc}(\bt)$ which is the topic of the next section.}

\section{Main Results}\label{sec 5}
\subsection{Non-asymptotic Identification of Nonlinear Systems}\label{sec:main results}
In this section, we provide our main results on statistical and convergence guarantees of gradient descent for learning nonlinear dynamical systems, using finite samples generated from a single trajectory. \red{Before stating our main result on non-asymptotic identification of nonlinear systems, we state a theorem to bound the Euclidean distance between the gradients the empirical loss $\Lch(\bt)$ and the auxiliary loss $\Lc_{\Dc}(\bt)$.}   
\begin{theorem}[Gradient convergence] \label{combined thm}
	Fix $r>0$. Suppose Assumptions \ref{assum stability} and \ref{ass bounded} on the system and Assumptions \ref{assume lip} and \ref{assume grad} on the Auxiliary Loss hold. Also suppose for all $\bteta \in \Bc^d(\btetas,r)$ and $(\h,\z)\subset \cal H\times \cal Z$, we have $\tn{\nabla_{\bteta}{\bphi_k(\h,\z;\bteta)}} \leq C_{\bphi}$ and {$\|\nabla_{\h}\nabla_{\bteta}{\bphi_k(\h,\z;\bteta)}\| \leq D_{\bphi}$} for all $1\leq k\leq n$ for some scalars $C_{\bphi},D_{\bphi}>0$. \red{Define $K_{\bphi} := (2/c_0)\beta_+ D_{\bphi}(\sigma/\sigma_0 \lor C_{\bphi}/K)$. Let $\beta_+>0$ be as in Lemma~\ref{lemma:bounded states}} and $N = \lfloor(T-L)/L\rfloor$, where we pick $L$ via
	\begin{align}
		L = \big\lceil 1 + \frac{\log(C\rho K_{\bphi}n\sqrt{N/d})}{1-\rho} \big\rceil. \label{eqn:pick_L}
	\end{align}
	Then, with probability at least $1-2Lp_0-L\log(\frac{Kr}{\sigma_0})\exp({-100d})$, for all $\bteta \in \Bc^d(\btetas,r)$, we have
	\begin{align}
		\tn{\nabla\Lch(\bt)-\nabla\Lc_\Dc(\bt)} \leq 2 c_0(\sigma_0+K\tn{\bt-\btetas})C_{\log}\sqrt{\frac{d}{N}}.
	\end{align}
\end{theorem}
\red{\begin{proof}{\bf sketch:} Theorem~\ref{combined thm} can be proved by combining the results of Theorems~\ref{refined thm} and \ref{thrm:diff truncated vs actual}. The idea is to split the truncated loss $\Ltr$~(Def.~\ref{def:truncated loss}) as an average of $L$ truncated subtrajectory losses $\ltr_\tau$~(Def.~\ref{def:empirical loss}) as: $\Ltr(\bt) = \frac{1}{L}\sum_{\tau=0}^{L-1}\ltr_\tau(\bteta)$. Recall that $\Lcp(\bt) = \E[\ltr_\tau(\bt)]$. Then, we use Theorem~\ref{refined thm} with a union bound over all $0\leq\tau\leq L-1$ to upper bound $\tn{\nabla\ltr_\tau(\bt) - \nabla\Lcp(\bt)}$ which is used to show the uniform convergence of the truncated loss $\Ltr$ as: $\tn{\nabla\Ltr(\bt) - \nabla\Lcp(\bt)} \leq \frac{1}{L}\sum_{\tau=0}^{L-1}\tn{\nabla\ltr_\tau(\bt) - \nabla\Lcp(\bt)}$. Combining this with Theorem~\ref{thrm:diff truncated vs actual} and picking $L$ via~\eqref{eqn:pick_L}, we get the statement of the theorem.
	\end{proof}
	Observe} that $K_{\bphi}$ depends on the system related constants and the noise level. For example, for a linear dynamical system \eqref{linear state eqn}, we can show that $K_{\bphi} = c\sqrt{n+p}$. Note that, if we choose $N \gtrsim K^2C_{\log}^2d/\alpha^2$ in Theorem~\ref{combined thm}, we get $\tn{\nabla\Lch(\bt)-\nabla\Lc_\Dc(\bt)} \lesssim \sigma_0C_{\log}\sqrt{d/N} + (\alpha/2)\tn{\bt-\btetas}$. 
\red{Combining this result with Assumption~\ref{assum grad dominance} gives our final result on non-asymptotic identification of nonlinear dynamical systems from a single trajectory.}
\begin{theorem}[Non-asymptotic identification]\label{thrm:main thrm}
	Consider the setup of Theorem~\ref{combined thm}. Also suppose the Auxiliary loss satisfies Assumption~\ref{assum grad dominance}. Let $N = \lfloor(T-L)/L\rfloor$, where we pick $L$ as in Theorem~\ref{combined thm}. Suppose $N \gtrsim K^2C_{\log}^2 d/\alpha^2$. 
	Given $r>0$, set learning rate $\eta = \alpha/(16\beta^2)$ and pick $\bheta_0 \in \Bc^d(\bhetas,r)$. Assuming $\sigma_0 \lesssim rK$, with probability at least $1-2Lp_0-L\log(\frac{Kr}{\sigma_0})\exp({-100d})$, all gradient descent iterates $\bteta_\tau$ on $\Lch$ satisfy
	\begin{align}
		\tn{\bteta_\tau - \btetas} & \leq \big(1 - \frac{\alpha^2}{128\beta^2}\big)^{\tau}\tn{\bheta_0 - \bhetas} + \frac{c\sigma_0}{\alpha}C_{\log}\sqrt{\frac{d}{N}}.
	\end{align}
\end{theorem}
\red{\begin{proof}{\bf sketch:} Theorem~\ref{thrm:main thrm} directly follows from combining Theorems~\ref{thrm:meta}~and~\ref{combined thm} and choosing $N \gtrsim K^2C_{\log}^2 d/\alpha^2$. 
	\end{proof}
	Observe} that, Theorem~\ref{thrm:main thrm} requires $\order{d}$ samples to learn the dynamics $\btetas \in \R^d$, \red{hence, our sample complexity captures the correct dependence on the dimension of unknown system dynamics}. Furthermore, it achieves optimal statistical error rate $\sigma\sqrt{d/N}$. Recall that the gradient noise $\sigma_0$ is a function of the process noise $\sigma$, and role of $\sigma$ will be more clear in \S~\ref{sec:applications}. We remark that while this theorem provides strong dependence, the results can be further refined when the number of states $n$ is large since each sample in \eqref{general state eqn} provides $n$ equations. For example, we can accomplish better sample complexity for separable dynamical systems~(see \S\ref{sec:main results sep}) which is the topic of next section.

\subsection{Separable Dynamical Systems}\label{sec:main results sep}
Suppose now that the nonlinear dynamical system is separable, that is, the nonlinear state equation~\eqref{general state eqn} can be split into $n$ state updates via
\begin{align}
	\h_{t+1}[k] =\bphi_k(\h_{t},\z_{t};\bteta_k^\star)+\w_{t}[k], \quad \text{for} \;\; 1\leq k\leq n, \label{separable state eqn}
\end{align} 
where $\h_t[k]$ and $\w_t[k]$ denote the $k_{\rm th}$ entry of $\h_t$ and $\w_t$ respectively while $\bphi_k$ denotes the scalar function associated to the $k_{\rm th}$ entry of $\bphi$. The overall system is given by the concatenation 
$\btetas = [\bteta_1^{\star \top}~\cdots~\bteta_n^{\star \top}]^\top$. For simplicity, let us assume $\bteta_k^{\star} \in \R^{\bad}$, where $\bad = d/n $. \red{In the case of separable dynamical systems, the empirical loss in \eqref{eqn:minimization problem} is alternately given by,}
\begin{align}
	&\hat{\Lc}(\bteta) = \sum_{k=1}^n \Lch_k(\bteta_k)\quad \text{where} \quad \Lch_k(\bteta_k) :=\frac{1}{2(T-L)}\sum_{t=L}^{T-1}(\h_{t+1}[k]-\bphi_k(\h_t,\z_t;\bteta_k))^2. \label{eqn:empirical loss sep}
\end{align}
As before, we \red{aim to learn the system dynamics $\btetas$ via gradient descent.} The gradient of the empirical loss simplifies to $\nabla\Lch(\bteta) = [\nabla\Lch_{1}(\bteta_1)^\top~\cdots~\nabla\Lch_{n}(\bteta_n)^\top]^\top$. From this, we observe that learning $\btetas$ via~\eqref{eqn:minimization problem} is equivalent to learning each of its components $\bteta_k^{\star}$ by solving $n$ separate ERM problems in $\R^\bad$. Denoting $\btetah$ to be the solution of the ERM problem~\eqref{eqn:minimization problem}, we have the following equivalence: $\btetah \equiv [\btetah_1^\top~\cdots~\btetah_n^\top]^\top$, where $\btetah_k \in \R^\bad$ is the solution to the following minimization problem,
\begin{align}
	\btetah_k = \arg \min_{\bteta_k \in \R^\bad}\hat{\Lc}_k(\bteta_k). \label{eqn:minimization problem sep}
\end{align}
Similarly global iterations \eqref{eqn:gradient descent} follows the iterations of the subproblems, that is, the GD iterate~\eqref{eqn:gradient descent} implies $\bteta_k^{(\tau+1)} = \bteta_k^{(\tau)} -\eta \nabla\Lch_k(\bteta_k^{(\tau)})$. \red{Before, stating our main result on learning separable nonlinear dynamical systems, we will show how the Auxiliary loss $\Lcp$ and its finite sample approximation $\Lch_{\Sc}$ can be split into the sum of $n$ losses as follows,
	\begin{equation}\label{eqn:LS_and_LD sep}
		\begin{aligned}
			\Lch_{\cal{S}}(\bt) =  \sum_{k=1}^n \Lch_{k,\cal{S}}(\bt_k) \quad &\text{where} \quad \Lch_{k,\cal{S}}(\bt_k) =\frac{1}{N}\sum_{i=1}^N \Lc_k(\bt_k,\x_i), \\
			\Lc_{\Dc}(\bt)=\sum_{k=1}^n \Lc_{k,\cal{D}}(\bt_k) \quad &\text{where} \quad\Lc_{k,\cal{D}}(\bt_k) =\E[\Lc_k(\bt_k,\x)],
		\end{aligned}
	\end{equation}
	where $\Lc_k(\cdot,\x)$ is a loss function that admits a sample $\x$ and outputs the corresponding loss. When learning~\eqref{separable state eqn}, the sample $\x$ corresponds to the variables $(\h_{L},\h_{L-1},\z_{L-1})$ triple and the loss function $\Lc_k(\bteta,\x)$ is given by
	\begin{align}
		\Lc_k(\bt_k,(\h_L,\h_{L-1},\z_{L-1})) &:= \frac{1}{2} (\h_L[k] - \bphi_k(\h_{L-1},\z_{L-1};\bteta_k))^2.
	\end{align}
	The} following theorem gives refined sample complexity for learning the dynamics of separable nonlinear dynamical systems.
\begin{theorem}[Refined complexity]\label{thrm:main theorem separable}
	Suppose Assumptions \ref{assum stability} and \ref{ass bounded} on the system and Assumptions \ref{assum grad dominance}, \ref{assume lip} and \ref{assume grad} on the Auxiliary Loss~\eqref{eqn:LS_and_LD sep} hold for all $1 \leq k \leq n$. Additionally, suppose the nonlinear dynamical system is separable, that is, the nonlinear state equation follows~\eqref{separable state eqn}. 
	Let $N = \lfloor(T-L)/L\rfloor$, where we pick $L$ via 
	\begin{align}
		L = \big\lceil 1 + \frac{\log(C\rho K_{\bphi}n\sqrt{N/\bad})}{1-\rho} \big\rceil.
	\end{align}
	Suppose $N \gtrsim K^2 C_{\log}^2\bad/\alpha^2$. Given $r>0$, set the learning rate $\eta = \alpha/(16\beta^2)$ and pick $\bheta^{(0)} \in \Bc^d(\bhetas,r)$. Assuming $\sigma_0 \lesssim rK$, with probability at least $1-2Lnp_0-Ln\log(\frac{Kr}{\sigma_0})\exp({-100\bad})$, all gradient descent iterates $\bheta^{(\tau)}$ on $\Lch$ satisfy
	\begin{align}
		\tn{\bteta_{k}^{(\tau)} - \bteta_k^\star} & \leq \big(1 - \frac{\alpha^2}{128\beta^2}\big)^{\tau}\tn{\bheta_{k}^{(0)} - \bheta_k^\star} +\frac{c\sigma_0}{\alpha}C_{\log}\sqrt{\frac{\bad}{N}} \quad \text{for all}\quad 1\leq k\leq n.
	\end{align}
\end{theorem}
\red{\begin{proof}{\bf sketch:} The proof technique for Theorem~\ref{thrm:main theorem separable} is similar to that of Theorem~\ref{thrm:main thrm}. First, using Assumptions~\ref{assume lip} and \ref{assume grad} on the Auxiliary loss~\eqref{eqn:LS_and_LD sep}, we get an upper bound on $\tn{\nabla\Lch_{k,\Sc}(\bt_k) - \nabla\Lc_{k,\cal{D}}(\bt_k)}$ for all $1 \leq k \leq n$. Next, using Assumption~\ref{assum stability} and \ref{ass bounded} on the system, we upper bound $\tn{\nabla{\Lch_k(\bteta_k)}-\nabla\Ltr_k(\bteta_k)}$ for all $1 \leq k \leq n$. Combining these two bounds, we get an upper bound on the gradient distance $\tn{\nabla\Lch_k(\bt_k)-\nabla\Lc_{k,\Dc}(\bt_k)}$ for all $1 \leq k \leq n$. After picking $N$ and $L$ in the same way as we we did in Theorem~\ref{thrm:main thrm}, we use Theorem~\ref{thrm:meta} with Assumption~\ref{assum grad dominance} on the Auxiliary loss~\eqref{eqn:LS_and_LD sep} and the derived bound on $\tn{\nabla\Lch_k(\bt_k)-\nabla\Lc_{k,\Dc}(\bt_k)}$ to get the statement of the theorem.   
	\end{proof}
	Observe that, in the case of separable dynamical systems we require $ \order{\bar{d}}$ samples to learn the dynamics $\btetas \in \R^d$. We achieve refined sample complexity because each sample provides $n$ equations and $\bar{d} = d/n$. Common dynamical systems like linear dynamical systems and nonlinear state equations are very structured and have separable state equations. Hence, applying Theorem~\ref{thrm:main theorem separable} to these systems results in accurate sample complexity and optimal statistical error rates which is the topic of the next section.} 
\section{Applications}\label{sec:applications}
In this section, we apply our results from the previous section to learn two different dynamical systems of the following form,
\begin{align}
	\h_{t+1}=\phi(\As\h_t)+\Bs\z_t+\w_t, \label{eqn:general app}
\end{align}
where $\As \in \R^{n \times n}$, $\Bs \in \R^{n \times p}$ are the unknown system dynamics, $\z_t \distas \Nc(0,\Iden_p)$ and $\w_t \distas \Nc(0,\sigma^2\Iden_n)$. \red{Specifically we learn the dynamics of the following dynamical systems: {\bf{(a)}} Standard linear dynamical systems~($\phi = \Iden_n$); and {\bf{(b)}} Nonlinear state equations 
	\begin{align}
		\h_{t+1} = \phi(\Bhetas \h_{t}) + \z_t + \w_t, \label{nonlinear state eqn}
	\end{align}
	where the nonlinear function $\phi:\R\rightarrow\R$ applies entry-wise on vector inputs. For the clarity of exposition, we focus on stable systems and set the feedback policy $\mtx{\pi}(\h_t)=0$. For linear dynamical systems, this is equivalent to assuming $\rho(\As) <1$. For nonlinear state equation, we assume $(C_\rho,\rho)$-stability holds according to Definition~\ref{def:non-linear stability}.} 
\subsection{Linear Dynamical Systems}\label{sec:application_to_LDS}
To simplify the notation, we define the following concatenated vector/matrix: $\x_t := [\h_t^\top~\z_t^\top]^\top$ and $\Bhetas := [\As~\Bs]$. Letting $\phi = \Iden_n$, the state update~\eqref{eqn:general app} is alternately given by: $\h_{t+1} = \Bhetas \x_t+\w_t$. To proceed, let $\bteta_k^{\star \top}$ denotes the $k_{\rm th}$ row of $\Bhetas$, then $ \Bhetas \equiv [\bteta_1^{\star}~\cdots~\bteta_n^{\star}]^\top$. Observe that the standard linear dynamical system is separable as in \eqref{separable state eqn}. Therefore, given a finite trajectory $(\h_{t},\z_{t})_{t=0}^{T-1}$ of the linear dynamical system~\eqref{eqn:general app}~($\phi = \Iden_n$), we construct the empirical loss as follows,
\begin{align}
\Lch(\Bheta) = \sum_{k=1}^n \Lch_k(\bteta_k) \quad \text{where} \quad \Lch_k(\bteta_k) := \frac{1}{2(T-L)} \sum_{t=L}^{T-1}(\h_{t+1}[k] - \bteta_k^\top\x_t)^2. \label{eqn:Ltrue linear}  
\end{align}
Before stating our main result, we introduce a few more concepts to capture the properties of gradient descent for learning the dynamics $\bteta_k^\star$. \red{Define the matrices,
\begin{align}
	\Gb_t := [\A^{t-1}_\star\B_\star~\A^{t-2}_\star\B_\star~\cdots~\B_\star] \quad \text{and} \quad \Fb_t := [\A^{t-1}_\star~\A^{t-2}_\star~\cdots~\Iden_n]. \label{eqn:Gram matrices}
\end{align}
Then, the matrices $\Gb_t\Gb_t^\top$ and $\Fb_t\Fb_t^\top$ are the finite time controllability Gramians for the control and noise inputs, respectively. 
It is straightforward to see that the covariance matrix of the concatenated vector $\x_t$ satisfies the following bounds~(see \S~\ref{sec:proofs of LDS} for detail)
\begin{equation}\label{def:alpha and beta}
	\begin{aligned}
		(1 \land\lmn{\Gb_t\Gb_t^\top+\sigma^2\Fb_t\Fb_t^\top}) \Iden_{n+p} \preceq \bSi[\x_t] \preceq (1 \lor \lmx{\Gb_t\Gb_t^\top+\sigma^2\Fb_t\Fb_t^\top}\Iden_{n+p}).
	\end{aligned}
\end{equation}
Define, $\gamma_- := 1 \land\lmn{\Gb_{L-1}\Gb_{L-1}^\top+\sigma^2\Fb_{L-1}\Fb_{L-1}^\top}$, $\gamma_+ := 1 \lor \lmx{\Gb_{L-1}\Gb_{L-1}^\top+\sigma^2\Fb_{L-1}\Fb_{L-1}^\top}$ and $\beta_+ = 1 \lor \max_{ 1\leq t\leq T}\lmx{\Gb_t\Gb_t^\top+\sigma^2\Fb_t\Fb_t^\top}$. The following corollary of Theorem~\ref{thrm:main theorem separable} states our main result on the statistical and convergence guarantees of gradient descent for learning the dynamics of linear dynamical systems.}
\begin{corollary}\label{thrm:main thrm LDS}	
\red{Consider the system~\eqref{eqn:general app} with $\phi= \Iden_n$. Suppose $\rho(\As) <1$. Let $C_\rho \geq 1$ and $\rho \in (\rho(\As),1)$ be scalars.} Suppose $\z_t \distas \Nc(0,\Iden_p)$ and $\w_t \distas \Nc(0,\sigma^2\Iden_n)$. Let $\gamma_+\geq \gamma_->0$ be as defined in~\eqref{def:alpha and beta} and set $\kappa = \gamma_+/\gamma_-$. 
Let $N = \lfloor(T-L)/L\rfloor$, where we pick $L$ via
\begin{align}
	L = \big\lceil 1 + \frac{\log(CC_\rho \beta_+ N(n+p)/\gamma_+)}{1-\rho} \big\rceil.
\end{align}
Suppose $N \gtrsim \kappa^2 \log^2(6N+3)(n+p)$. Set the learning rate $\eta = \gamma_-/(16\gamma_+^2)$ and the initialization $\Bheta^{(0)} = 0$. Assuming $\sigma \lesssim \tf{\Bhetas}\sqrt{\gamma_+}$, with probability at least $1-4T\exp(-100n)-Ln\big(4+\log(\frac{\tf{\Bhetas} \sqrt{\gamma_+}}{\sigma})\big)\exp(-100(n+p))$, for all $1\leq k \leq n$, all gradient descent iterates $\Bheta^{(\tau)}$ on $\Lch$ satisfy
\begin{align}
	\tn{\bteta_{k}^{(\tau)} - \bteta_k^\star} &\leq \big(1 - \frac{\gamma_-^2}{128\gamma_+^2}\big)^{\tau}\tn{\bheta_{k}^{(0)} - \bheta_k^\star} + \frac{c\sigma}{\gamma_-} \sqrt{\gamma_+}\log(6N+3)\sqrt{\frac{n+p}{N}}.
\end{align}
\end{corollary}
\red{Observe that Corollary~\ref{thrm:main thrm LDS} requires $\order{n+p}$ samples to learn the dynamics $\As \in \R^{n \times n}$ and $\Bs \in \R^{n \times p}$. The sample complexity captures the correct dependence on the dimension of unknown system dynamics, because each sample provides $n$ equations and there are $n(n+p)$ unknown parameters.} The sample complexity correctly depends on the condition number of the covariance matrix $\bSi[\x_t]$. When the condition number of $\bSi[\x_t]$ is close to $1$, the sample complexity of the problem is lower and vice versa. Lastly, our statistical error rate $\sigma \sqrt{(n+p)/N}$ is also optimal up to a constant. The logarithmic dependence on $\tf{\Bhetas}$ is an artifact of our general framework. We believe it can be possibly removed with a more refined concentration analysis.
\subsection{Nonlinear State Equations}\label{sec:application_to_RNN}
In this section, we apply Theorem~\ref{thrm:main theorem separable} to learn the nonlinear state equation~\eqref{nonlinear state eqn}. Observe that the nonlinear system~\eqref{nonlinear state eqn} is separable because we assume that the nonlinear function $\phi:\R\rightarrow\R$ applies entry-wise on vector inputs. Let $\bheta_k^{\star \top}$ denotes the $k_{\rm th}$ row of $\Bhetas$. Given a finite trajectory $(\h_{t+1},\h_{t})_{t=0}^{T-1}$ of~\eqref{nonlinear state eqn}, we construct the empirical loss as follows,
\begin{align}
\Lch(\Bheta) = \sum_{k=1}^n \Lch_k(\bheta_k) \quad \text{where} \quad \Lch_k(\bheta_k) := \frac{1}{2(T-L)} \sum_{t=L}^{T-1}(\h_{t+1}[k] - \phi(\bheta_k^\top\h_t))^2. \label{eqn:Ltrue nonlinear} 
\end{align}
The following corollary of Theorem~\ref{thrm:main theorem separable} states our main result on the statistical and convergence guarantees of gradient descent for learning the nonlinear system~\eqref{nonlinear state eqn}.
\begin{corollary}\label{thrm:main thrm RNN}	
Suppose the nonlinear system~\eqref{nonlinear state eqn} satisfies $(C_\rho,\rho)$-stability according to Def.~\ref{def:non-linear stability}. Suppose $\phi$ is $\gamma$-increasing (i.e.~$\phi'(x)\geq \gamma>0$ for all $x\in\R$), has bounded first and second derivatives, that is, $|\phi'|,|\phi''| \leq 1$, and $\phi(0)=0$. Suppose $\z_t \distas \Nc(0,\Iden_n)$ and $\w_t \distas \Nc(0,\sigma^2\Iden_n)$. 
Let $N = \lfloor(T-L)/L\rfloor$, where we pick $L$ via
\begin{align}
	L = \big\lceil 1 + \frac{\log(C C_\rho(1+\tf{\Bhetas}C_\rho(1+\sigma)/(1-\rho)) Nn)}{1-\rho} \big\rceil.
\end{align}
Setting $D_{\log} = \log(3(1+\sigma)n + 3C_\rho(1+\sigma)\tf{\Bhetas} n^{3/2}\log^{3/2}(2T)N/(1-\rho)+3)$, suppose $N \gtrsim \frac{C_\rho^4}{\gamma^4(1-\rho)^4}D_{\log}^2n $. Set the learning rate $\eta = \frac{\gamma^2(1-\rho)^4}{32C_\rho^4(1+\sigma)^2n^2}$ and pick the initialization $\Bheta^{(0)}=0$. Assuming $\sigma \lesssim\tf{\Bhetas}$, with probability at least $1-Ln\big(4T+\log(\frac{\tf{\Bhetas}C_\rho(1+\sigma)}{\sigma(1-\rho)})\big)\exp(-100n)$, for all $1\leq k \leq n$, all gradient descent iterates $\Bheta^{(\tau)}$ on $\Lch$ satisfy
\begin{align}
	\tn{\bheta_{k}^{(\tau)} - \bheta_k^\star}  &\leq \big(1 - \frac{\gamma^4(1-\rho)^4}{512C_\rho^4n^2}\big)^{\tau}\tn{\bheta_{k}^{(0)} - \bheta_k^\star} +\frac{c\sigma }{\gamma^2(1-\rho)} C_\rho D_{\log}\sqrt{\frac{n}{N}}.
\end{align}
\end{corollary}
\RED{We believe that the condition of $\gamma$-increasing $\phi$ can be relaxed and we expect many \red{nonlinear activations including ReLU to work}. The main challenge is verifying one-point convexity of the population loss \red{when $\phi$ is ReLU}. Lemma 6.1 of~\cite{mohammadreza2019fitting} shows this property for i.i.d. Gaussian features. Extending this to subgaussian features, would yield the ReLU result.} Theorem~\ref{thrm:main thrm RNN} requires $\order{n}$ samples to learn the dynamics $\Bhetas \in \R^{n\times n}$ since each sample gives $n$ equations. The sample complexity depends on the condition number of the covariance matrix $\bSi[\h_t]$, which can be shown to be bounded by $C_\rho^2/(1-\rho)^2$~(see Section~\ref{sec:proofs of RNN}). Lastly, similar to the linear case, our statistical error rate $\sqrt{n/N}$ is also optimal up to a constant.

\red{\begin{remark}[Probability of success]
	\emph{For our main results, instead of achieving $1-\delta$ probability of success with variable $\delta$, we are content with achieving $1-K_{\log}\exp(-cd)$ probability of success for an absolute constant $c>0$, where $K_{\log}$ is a fixed constant which depends either logarithmically or linearly on the values of $n,L,T,N,\sigma_0,K$ etc. However, there is a trade-off~(by changing the variable $t$ or $\tau$) between the probability of success and the concentration bounds in Lemma~\ref{subexp lem}. The probability of success and the error bounds in Theorems~\ref{combined thm},~\ref{thrm:main thrm} and \ref{thrm:main theorem separable} are coming from an application of Lemma~\ref{subexp lem}. We simply apply this lemma using a fixed choice of $t = c \sqrt{d}$ so that we obtain the error bound $\tilde{\mathcal{O}}(\sigma_0\sqrt{d/N})$ with probability at least $1 -K_{\log}\exp({-100d})$.	
		Similarly, Lemma~\ref{lemma:subgaussian vector bound} shows a trade-off~(by changing the variable $m$) between the probability of success and the Euclidean norm bound. We use Theorem~\ref{thrm:main theorem separable} along-with Lemma~\ref{lemma:subgaussian vector bound} to obtain the Corollaries~\ref{thrm:main thrm LDS} and \ref{thrm:main thrm RNN}. Specifically, we use Lemma~\ref{lemma:subgaussian vector bound} with a fixed choice of $m = n+p$ or $m=n$ to obtain the advertised error bounds with probability at least $1-K_{\log}\exp({-100n})$.}
\end{remark}}

\section{Numerical Experiments}\label{sec:numerical_exp}
\begin{center}
	{
		\setlength\arrayrulewidth{0.8pt}
		\resizebox{\columnwidth}{!}{	
			\begin{tabular}{ |c|c|c|c|c|c|c| }
				\hline
				Leakage & $\|\As\|$& $\|\Aps\|$ & $\rho(\As)$ &$\rho(\Aps)$  & $\sup_{\tn{\x}=1}{\tn{\phi(\As\x)}}$ & $\sup_{\tn{\x}=1}{\tn{\phi(\Aps\x)}}$ \\ 
				\hline
				$\lambda =0.00 $ & 2.07 & 1.85 & 1.12 & 0.65 & 1.79 & 1.56 \\  
				\hline
				$\lambda =0.50 $ & 2.07 & 1.85 & 1.12 & 0.65 & 1.84 & 1.60 \\ 
				\hline
				$\lambda =0.80 $ & 2.07 & 1.85 & 1.12 & 0.65 & 1.92 & 1.70 \\  
				\hline
				$\lambda =1.00 $ & 2.07 & 1.85 & 1.12 & 0.65 & 2.07 & 1.85 \\ 
				\hline
			\end{tabular}
		}
	}
	\captionof{table}{\small{This table lists the core properties of the (random) state matrix in our experiments. The values are averaged over $1000$ random trials. For linear systems, the state matrix $\As$ is unstable however the closed-loop matrix $\Aps$ is stable. We also list the nonlinear spectral norms (i.e.~$\sup_{\tn{\x}=1}\tn{\phi(\As\x)}$) associated with $\As$ and $\Aps$, as a function of different leakage levels of leaky-ReLUs, which are all larger than $1$. Despite this, experiments show nonlinear systems are stable with $\Aps$ (some even with $\As$). This indicates that Definition \ref{def:non-linear stability} is indeed applicable to a broad range of systems. }}	\label{table:only one}
\end{center}
\begin{figure*}
	\centering
	\begin{subfigure}[t]{0.30\textwidth}
		\includegraphics[width=\linewidth]{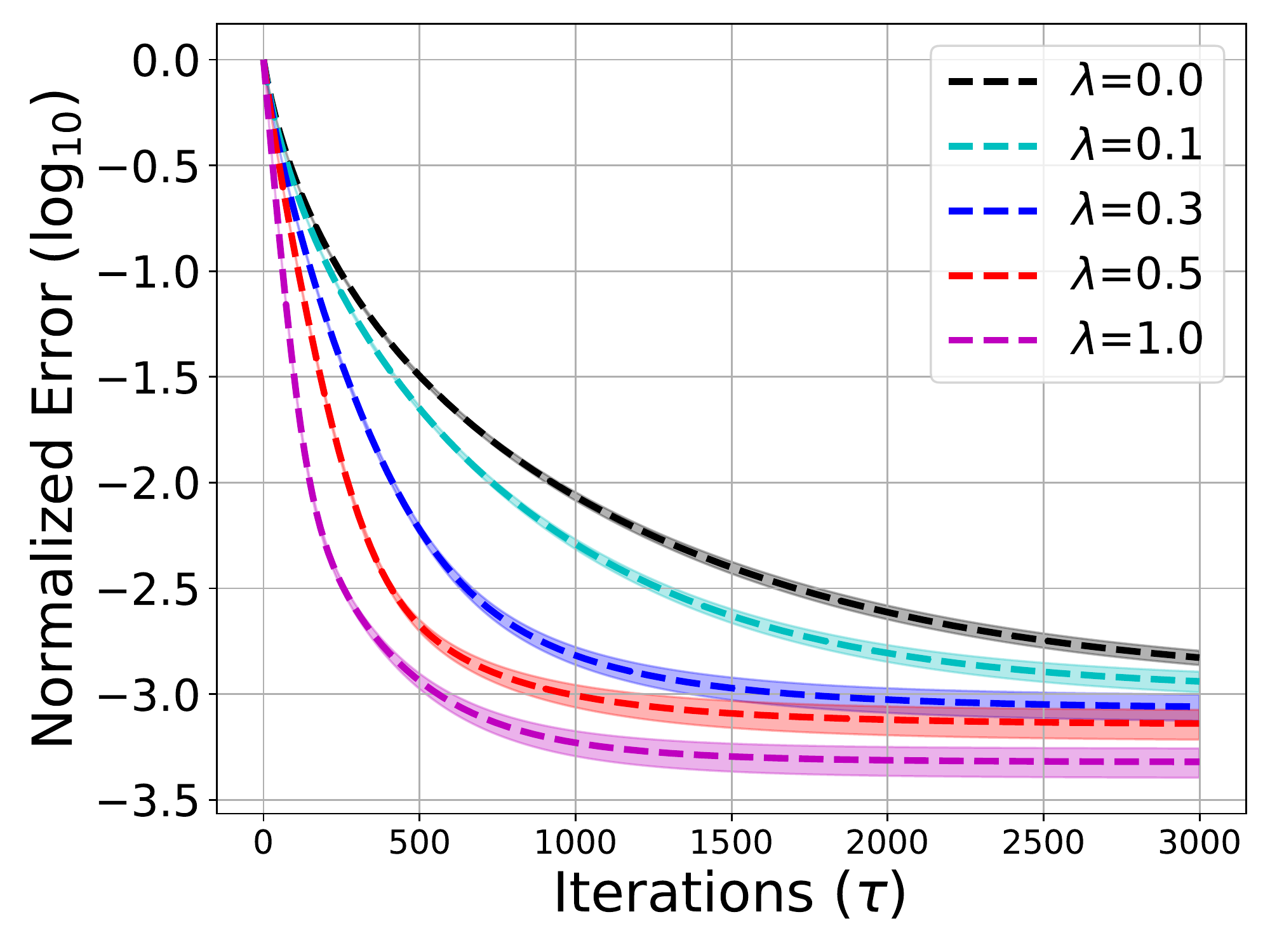}\vspace{-6pt}
		\caption{Nonlinearity}\label{fig1a}
	\end{subfigure}
	~
	\begin{subfigure}[t]{0.30\textwidth}
		\includegraphics[width=\linewidth]{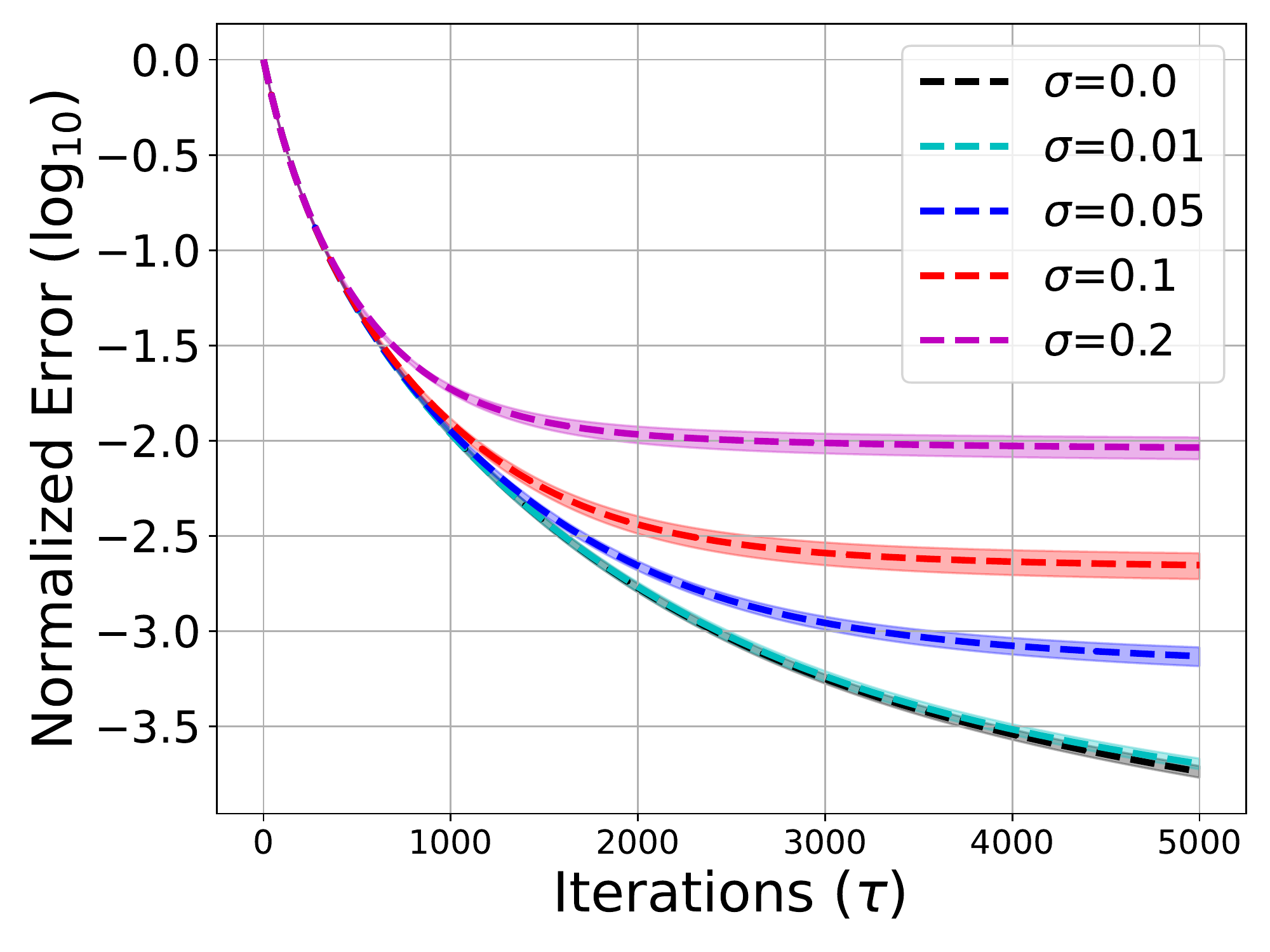}\vspace{-6pt}
		\caption{Noise level}\label{fig1b}
	\end{subfigure}
	~
	\begin{subfigure}[t]{0.30\textwidth}
		\includegraphics[width=\linewidth]{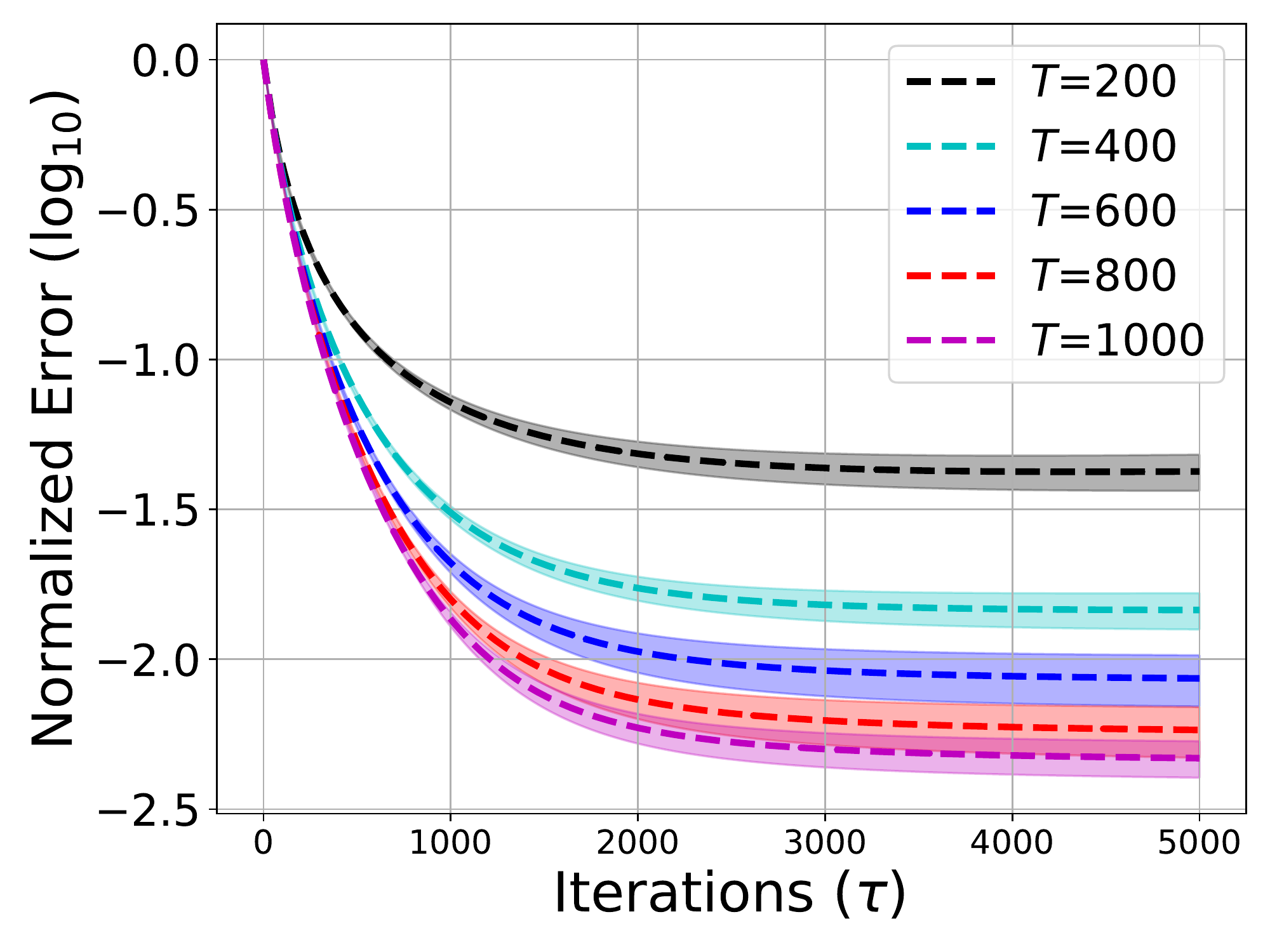}\vspace{-6pt}
		\caption{Trajectory length}\label{fig1c}
	\end{subfigure}
	\caption{We run gradient descent to learn nonlinear dynamical system governed by state equation $\h_{t+1} = \phi(\A\h_t + \B\ub_t)+\w_t$. We study the effect of nonlinearity, noise variance and trajectory length on the convergence of gradient descent. The empirical results verify what is predicted by our theory.}
	\label{fig1}
\end{figure*}
For our experiments, we choose unstable nonlinear dynamical systems~($\rho(\A) >1$) governed by nonlinear state equation $\h_{t+1} = \phi(\A\h_t + \B\ub_t)+\w_t$ with state dimension $n = 80$ and input dimension $p=50$. $\A$ is generated with $\Nc(0,1)$ entries and scaled to have its largest $10$ eigenvalues greater than $1$. $\B$ is generated with i.i.d.~$\Nc(0,1/n)$ entries. For nonlinearity, we use either softplus~($\phi(x) = \ln(1+e^x)$) or leaky-ReLU~($\max(x,\la x)$, with leakage $0\leq \lambda\leq 1$) activations. We run gradient descent with fixed learning rate $\eta = 0.1/T$, where $T$ denotes the trajectory length. \RED{We choose a noisy stabilizing policy $\mtx{K}$ for the linear system (ignoring $\phi$) and set $\ub_t = -\mtx{K}\h_t + \z_t $. Here $\mtx{K}$ is obtained by solving a discrete-time Riccati equation~(by setting rewards $\mtx{Q},\mtx{R}$ to identity) and adding random Gaussian noise with zero mean and variance 0.001 to each entry of the Riccati solution. We want to emphasize that any stabilizing policy will work here. For some nonlinear activations, as shown in Figure~\ref{fig2}, one can learn the system dynamics using a policy which is unstable for the linear system but remains stable for the nonlinear system. Lastly, $\z_t~\distas \Nc(0,\Iden_p)$ and $\w_t \distas \Nc(0,\sigma^2\Iden_n)$.}

We plot the normalized estimation error of $\A$ and $\B$ given by the formula $\tf{\A-\hat{\A}}^2/\tf{\A}^2$ (same for $\B$). Each experiment is repeated $20$ times and we plot the mean and one standard deviation. To verify our theoretical results, we study the effect of the following on the convergence of gradient descent for learning the system dynamics.

\noindent $\bullet$ {\bf{Nonlinearity:}} This experiment studies the effect of nonlinearity on the convergence of gradient descent for learning nonlinear dynamical system with leaky-ReLU activation. We run gradient descent over different values of $\lambda$~(leakage). The trajectory length is set to $T=2000$ and the noise variance is set to $\sigma^2=0.01$. In Figure \ref{fig1a}, we plot the normalized estimation error of $\A$ over different values of $\lambda$. We observe that, decreasing nonlinearity leads to faster convergence of gradient descent. 

\noindent $\bullet$ {\bf{Noise level:}} This experiment studies the effect of noise variance on the convergence of gradient descent for learning nonlinear dynamical system with softplus activation. The trajectory length is set to $T=2000$. In Figure \ref{fig1b}, we plot the normalized estimation error of $\A$ over different values of noise variance. We observe that, the gradient descent linearly converges to the ground truth plus some residual which is proportional to the noise variance as predicted by our theory.

\noindent $\bullet$ {\bf{Trajectory length:}} This experiment studies the effect of trajectory length on the statistical accuracy of learning system dynamics via gradient descent. We use softplus activation and the noise variance is set to $\sigma^2=0.01$. In Figure \ref{fig1c}, we plot the normalized estimation error of $\A$ over different values of $T$. We observe that, by increasing the trajectory length~(number of samples), the estimation gets better, verifying our theoretical results.

We remark that, we get similar plots for the input matrix $\B$. Lastly, Figure~\ref{fig2} is generated by evolving the state through $100$ timesteps and recording the Euclidean norm of $\h_t$ at each timestep. This is repeated $500$ times with $\rho(\A) > 1$ and using leaky-ReLU activations. In Figure~\ref{fig2}, we plot the mean and one standard deviation of the Euclidean norm of the states $\h_t$ over different values of $\lambda$~(leakage). The states are bounded when we use leaky-ReLU with $\la\leq 0.5$ even when the corresponding LDS is unstable. \red{This shows that the nonlinearity can help the states converge to a point in state space. However, this is not always true. For example, when $\mtx{A} = 2 \Iden$ and $\h_0$ has all entries positive. Then, using leaky-ReLU will not help the trajectory to converge.}

\begin{figure}[t!]
	\centering
	\includegraphics[width=0.4\textwidth]{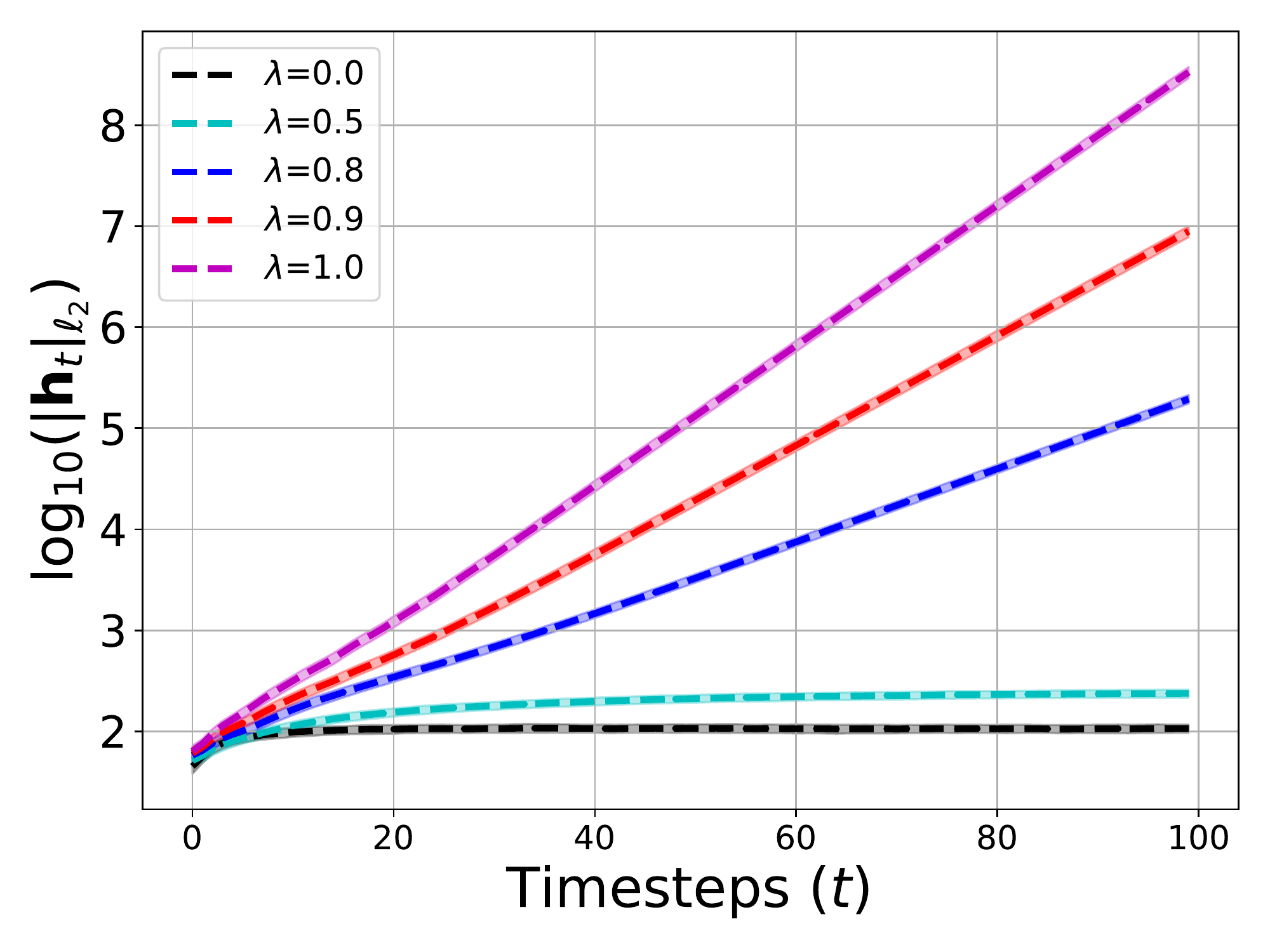}\vspace{-5pt}
	\vspace{-6pt}\caption{For a properly chosen random unstable system the state vectors diverge for LDS while they stay bounded for leaky ReLU systems with small leakage.
	}\label{fig2}
\end{figure}


\section{Related Work}\label{sec:related}
Nonlinear dynamical systems relate to the literature in control theory, reinforcement learning, and recurrent neural networks. We study nonlinear dynamical systems from optimization and learning perspective rather than control. While such problems are known to be challenging (especially under nonlinearity), there is a growing interest in understanding system identification and associated optimal control problems (e.g.~LQR) in a non-asymptotic and data-dependent fashion \cite{recht2019tour}. \red{More recently \cite{faradonbeh2017regret,faradonbeh2018finite,dean2018regret,simchowitz2018learning,simchowitz2019learning,hardt2018gradient,oymak2018non,fattahi2019learning,hazan2017learning,hazan2018spectral,sarkar2018fast, sarkar2019data,sarkar2019finite,sarkar2019near,tsiamis2019finite,tsiamis2020sample,wagenmaker2020active} explore linear system identification in great depth.} Similar to these, we also seek to provide non-asymptotic statistical guarantees for learning nonlinear dynamical systems. \cite{allen2018convergence} provides preliminary guarantees for recurrent networks (RNN) and \cite{miller2018stable} shows the role of stability in RNNs. There is also a substantial amount of work on model-free approaches~\cite{dann2015sample,zou2019finite,fazel2018global,krauth2019finite,malik2019derivative} which avoid learning the dynamics and find the optimal control input by directly optimizing over policy space.

\red{In a different line of work, \cite{singh2020learning} proposed a learning framework for trajectory planning from learned dynamics. They propose a regularizer of dynamics that promotes stabilizability of the learned model, which allows tracking reference trajectories based on estimated dynamics. Also, \cite{khosravi2020convex} and \cite{khosravi2020nonlinear} developed learning methods that exploit other control-theoretic priors. Nonetheless, none of these works characterize the sample complexity of the problem.}

Closer to our work, \cite{oymak2019stochastic,bahmani2019convex} study theoretical properties of nonlinear state equations with a goal towards understanding recurrent networks and nonlinear systems. While some high-level ideas, such as mixing-time arguments, are shared, our results \red{(a) apply to a broader class of nonlinear systems (e.g.~mild assumptions on nonlinearity), (b) utilize a variation of the spectral radius for nonlinear systems\footnote{Rather than enforcing contraction (i.e. 1-Lipschitzness)-based stability which corresponds to using spectral norm rather than spectral radius.}, (c) account for process noise, and (d) develop new statistical guarantees for the uniform convergence of the gradient of the empirical loss.} \RED{The concurrent work of~\cite{foster2020learning} provides related results for the recovery of generalized linear dynamical systems~($\h_{t+1} = \phi(\Bhetas\h_t)+\w_t$) using complementary techniques. \cite{foster2020learning} uses martingale arguments and analyze GLMtron algorithm of~\cite{Kakade_NIPS2011}, while we use mixing time arguments and analyze gradient descent.} \red{Lastly, \cite{mania2020active} proposes an active learning approach for non-asymptotic identification of nonlinear dynamical systems whose state transitions depend linearly on a known feature embedding of state-action pairs. Their adaptive method successfully achieves non-asymptotic identification by repeatedly applying trajectory planning, trajectory tracking, and re-estimation of the system.}

\red{ A very preliminary version of this work has been appeared in a workshop paper~\cite{sattar2019simple} where we provide preliminary guarantees for the identification nonlinear dynamical systems. In contrast to this work, \cite{sattar2019simple} does not provide sample complexity and statistical error bounds and learns a simple noiseless system by assuming the one-point convexity of the empirical loss~(with i.i.d. samples). On the other hand, this work provides new guarantees for non-asymptotic identification of nonlinear dynamical systems under process noise. It develops new statistical guarantees for the uniform convergence of the gradients of the empirical loss and applies the developed framework to learn nonlinear state equations $\h_{t+1} = \phi(\Bhetas\h_t)+\z_t+\w_t$. Lastly, it also provides the necessary technical framework and the associated proofs.}

Perhaps the most established technique in the statistics literature for dealing with non-independent, time-series data is the  use of mixing-time arguments~\cite{yu1994rates}. In  the  machine learning literature, mixing-time arguments have been used to develop generalization bounds~\cite{mohri2008stability,mohri2009rademacher,kuznetsov2017generalization,mcdonald2017nonparametric} which are analogous to the classical generalization bounds for i.i.d.~data. We utilize mixing-time for nonlinear stabilizable systems to connect our temporally-dependent problem to standard supervised learning task with a focus on establishing statistical guarantees for gradient descent. 

Finite sample convergence of the gradients of the empirical loss (to the population gradient) is studied by \cite{mei2018landscape,foster2018uniform}. These guarantees are not sufficient for our analysis as they only apply to problems with bounded nonlinearities and do not accurately capture the noise dependence. We address this by establishing stronger uniform convergence guarantees for empirical gradients and translate our bounds to the system identification via mixing-time/stability arguments.


\section{Conclusions}\label{sec:conclusion}
We proposed a general approach for learning nonlinear dynamical systems by utilizing stabizability and mixing-time arguments. We showed that, under reasonable assumptions, one can learn the dynamics of a nonlinear stabilized systems from a single finite trajectory. Our general approach can treat important dynamical systems, such as LDS and the setups of \cite{oymak2019stochastic,bahmani2019convex,foster2020learning} as special cases. We provided both sample size and estimation error guarantees on LDS and certain nonlinear state equations. Finally, the numerical experiments verify our theoretical findings on statistical and computational efficiency of gradient descent for learning nonlinear systems.

There are many interesting future avenues. One direction is exploring alternative approaches to mixing-time arguments. Martingale based arguments have the potential to provide tighter statistical guarantees and mitigate dependence on the spectral radius \cite{simchowitz2018learning}. Another important direction is learning better control policies by optimizing the policy function $\mtx{\pi}$ in a data driven fashion. This topic attracted significant attention for linear systems \cite{dean2018regret,recht2019tour} and led to strong regret guarantees \cite{mania2019certainty,cohen2019learning} however nonlinearity presents significant challenges. Our framework is more suitable for model based approaches (as it learns system dynamics $\btetas$) however model-free guarantees would be similarly intriguing.

\vspace{-5pt}\section*{Acknowledgements}\vspace{-5pt} This work is partially supported by NSF CNS-1932254.
{
\bibliographystyle{plain}
\bibliography{Bibfiles}
}
\clearpage
\appendix
\section{Proofs of the Main Results}\label{sec: proofs main}
\red{In this section, we present the proofs of our main results.}
\subsection{Proof of Theorem \ref{thrm:meta}}
\red{Before we begin the proof, we state a lemma to show the linear convergence of Gradient descent for minimizing an empirical loss $\Lch$ when the population loss $\Lcp$ satisfies Assumption~\ref{assum grad dominance} and the Euclidean distance between the gradients of the two losses is upper bounded as follows: $\tn{\nabla \Lch(\bt)-\nabla \Lcp(\bt)}\leq \nu+(\alpha/2)\tn{\bt-\btetas}$.}
\begin{theorem}[OPC convergence] \label{one point thm} Suppose Assumption \ref{assum grad dominance} holds. Assume for all $\bt\in\Bc^d(\btetas,r)$, $\nabla\hat{\Lc}$ satisfies 
$
\tn{\nabla \Lch(\bt)-\nabla \Lcp(\bt)}\leq \nu+(\alpha/2)\tn{\bt-\btetas}
$ and $r\geq 5\nu/\alpha$. Set learning rate $\eta=\alpha/(16\beta^2)$ and pick $\bteta_0\in\Bc^d(\btetas,r)$. All gradient descent iterates $\bt_\tau$ on ${\hat{\Lc}}$ satisfy
\begin{align}
	\tn{\bt_\tau-\btetas}\leq \big(1-\frac{\alpha^2}{128\beta^2}\big)^{\tau}\tn{\bt_0-\btetas}+\frac{5\nu}{\alpha}.
\end{align}
\end{theorem}
\begin{proof} Set $\bdel_\tau=\bt_\tau-\bt_\star$. At a given iteration $\tau$ we have that $\bdel_{\tau+1}=\bdel_\tau-\eta\nabla \Lch(\bt_\tau)$ which implies
\begin{align}
	\tn{\bdel_{\tau+1}}^2&=\tn{\bdel_\tau}^2-2\eta\li\bdel_\tau,\nabla \Lch(\bt_\tau)\ri+\eta^2\tn{\nabla \Lch(\bt_\tau)}^2.
\end{align}
Using Assumptions~\ref{assum grad dominance} and $\tn{\nabla \Lch(\bt)-\nabla \Lcp(\bt)}\leq \nu+(\alpha/2)\tn{\bt-\bt_\star}$, we have that
\begin{align}
	\li\bdel_\tau,\nabla \Lch(\bt_\tau)\ri&\geq \li\bdel_\tau,\nabla \Lc_\Dc(\bt_\tau)\ri-  | \li\bdel_\tau,\nabla \Lch(\bt_\tau)-\nabla \Lc_\Dc(\bt_\tau)\ri|, \nn \\
	&\geq \alpha\tn{\bdel_\tau}^2-(\nu+(\alpha/2)\tn{\bdel_\tau})\tn{\bdel_\tau}\geq (\alpha/2)\tn{\bdel_\tau}^2-\nu\tn{\bdel_\tau}. 
\end{align}
Similarly,
\begin{align}
	\tn{\nabla \Lch(\bt_\tau)}\leq \tn{\nabla \Lcp(\bt_\tau)} + \tn{\nabla \Lch(\bt_\tau)-\nabla \Lcp(\bt_\tau)} \leq (3/2)\beta\tn{\bdel_\tau}+\nu.
\end{align}
Suppose $\tn{\bdel_\tau}\geq 4\nu/\alpha$. Then, $(\alpha/2)\tn{\bdel_\tau}^2-\nu\tn{\bdel_\tau} \geq (\alpha/4)\tn{\bdel_\tau}^2$ and $(3/2)\beta\tn{\bdel_\tau}+\nu\leq 2\beta\tn{\bdel_\tau}$. Hence, using the learning rate $\eta=\frac{\alpha}{16\beta^2}$, we obtain
\[
\tn{\bdel_{\tau+1}}^2\leq \tn{\bdel_\tau}^2(1-\eta\alpha/2+4\eta^2\beta^2)\leq (1-\frac{\alpha^2}{64\beta^2})\tn{\bdel_\tau}^2.
\]
Now, imagine the scenario $\tn{\bdel_\tau}\leq 4\nu/\alpha$. We would like to prove that $\bdel_{\tau+1}$ satisfies a similar bound namely $\tn{\bdel_{\tau+1}}\leq 5\nu/\alpha$. This is shown as follows.
\begin{align}
	\tn{\bdel_{\tau+1}}^2&\leq \tn{\bdel_\tau}^2(1-\eta\alpha +(9/4)\eta^2\beta^2)+2\eta\nu\tn{\bdel_\tau}+\eta^2(3\nu\beta\tn{\bdel_\tau}+\nu^2), \nn \\
	&\leq (1-\frac{3\alpha^2}{64\beta^2})\tn{\bdel_\tau}^2 + \frac{\alpha}{8\beta^2}\nu\tn{\bdel_\tau} + \frac{\alpha^2}{256\beta^4}(3\nu \beta\tn{\bdel_\tau}+\nu^2), \nn \\
	& \leq (\frac{16}{\alpha^2}+\frac{1}{2\beta^2}+\frac{3\alpha}{64\beta^3}+\frac{\alpha^2}{256\beta^4})\nu^2\leq \frac{25}{\alpha^2} \nu^2, \nn
\end{align}
which implies $\tn{\bdel_{\tau+1}}\leq 5\nu/\alpha$. To get the final result observe that during initial iterations, as long as $\tn{\bdel_\tau}\geq 4\nu/\alpha$, we have
\[
\tn{\bdel_\tau}^2\leq  (1-\frac{\alpha^2}{64\beta^2})^\tau\tn{\bdel_0}^2\implies \tn{\bdel_\tau}\leq  (1-\frac{\alpha^2}{128\beta^2})^\tau\tn{\bdel_0}.
\]
After the first instance $\tn{\bdel_\tau}< 4\nu/\alpha$, iterations will never violate $\tn{\bdel_\tau}\leq 5\nu/\alpha$. The reason is
\begin{itemize}
	\item If $\tn{\bdel_\tau}< 4\nu/\alpha$: we can only go up to $5\nu/\alpha$ and $\bdel_{\tau+1}\leq  5\nu/\alpha$.
	\item If $4\nu/\alpha\leq \tn{\bdel_\tau}\leq 5\nu/\alpha$: we have to go down hence $\bdel_{\tau+1}\leq  5\nu/\alpha$.
\end{itemize}
\end{proof}
\red{In order to use Theorem~\ref{one point thm}, observe that the gradient difference $\tn{\nabla \Lch(\bt)-\nabla \Lcp(\bt)}$ can be simplified by picking a sufficiently large $L$ as follows,
\begin{align}
	\tn{\nabla \Lch(\bt)-\nabla \Lcp(\bt)}&\leq C_0(\sigma_0+\xi_0\tn{\bt-\btetas}) \sqrt{d/N} + C(\sigma+\xi\tn{\bteta-\btetas})C_\rho\rho^{L-1}, \nn \\
	&\leq 2C_0(\sigma_0+\xi_0\tn{\bt-\btetas}) \sqrt{d/N},
\end{align}
where we get the last inequality by choosing $L$ via,
\begin{align}
	\rho^{L-1} &\leq C_0(\sigma_0/\sigma \land \xi_0/\xi)\sqrt{d/N}/(CC_\rho), \nn \\ 
	\iff\quad\;\;\;   L &\geq 1 + \frac{\log((CC_\rho/C_0)\sqrt{N/d}(\sigma/\sigma_0 \lor \xi/\xi_0))}{\log(\rho^{-1})}, \nn \\
	\Longleftarrow \quad\quad    L &= \big\lceil 1 + \frac{\log((CC\rho/C_0)\sqrt{N/d}(\sigma/\sigma_0 \lor \xi/\xi_0))}{1-\rho} \big\rceil. \label{eqn: choice of L0} 
\end{align}
To proceed, if we pick $N \gtrsim \xi_0^2C_0^2 d/\alpha^2$, we obtain the following upper bound on the Euclidean distance between the gradient of $\Lch$ and $\Lcp$,
\begin{align}
	\tn{\nabla \Lch(\bt)-\nabla \Lcp(\bt)}&\leq (\alpha/2)\tn{\bt-\btetas} + C_0\sigma_0\sqrt{d/N}. 	
\end{align}
We are now ready to use Theorem~\ref{one point thm} to show the linear convergence Gradient descent algorithm for learning the unknown dynamical system~\eqref{general state eqn} by minimizing the empirical loss $\Lch$. Specifically, using Theorem~\ref{one point thm} with $\nu = C_0\sigma_0\sqrt{d/N}$, we get the statement of the Theorem. This completes the proof.}
\subsection{Proof of Theorem \ref{refined thm}}
\red{Before we begin our proof, we state a lemma to bound the Euclidean norm of a sum of i.i.d. subexponential random vectors. The following lemma is a restatement of Lemma D.7 of \cite{oymak2018learning} (by specializing it to unit ball) and it follows from an application of generic chaining tools.}
\begin{lemma} \label{subexp lem}Let $C>0$ be a universal constant. Suppose $N\geq d$. Let $(\vb_i)_{i=1}^N \in \R^d$ be i.i.d.~vectors obeying $\bmu=\E[\vb_i]$ and subexponential norm $\te{\vb_i-\bmu}\leq K$. With probability at least $1-2\exp(-c\min(t\sqrt{N},t^2))$, we have that
\begin{align}
	\tn{\frac{1}{N}\sum_{i=1}^n \vb_i-\bmu}\leq CK\frac{\sqrt{d}+t}{\sqrt{N}}.
\end{align}
Alternatively, setting $t=\tau\sqrt{d}$ for $\tau\geq 1$, with probability at least $1-2\exp(-c\tau d)$, we have
\begin{align}
	\tn{\frac{1}{N}\sum_{i=1}^N \vb_i-\bmu}\leq CK(\tau+1)\sqrt{d/N}.
\end{align}
\end{lemma}
Throughout the proof of Theorem~\ref{refined thm}. we pick the constraint set $\Cc=\Bc^d(\btetas, r)$, however, these ideas are general and would apply to any set with small covering numbers (such as sparsity, $\ell_1$, rank constraints).
\begin{proof}{\bf{of uniform convergence with covering argument:}} We will use a peeling argument \cite{geer2000empirical}. Split the ball $\Bc^d(\btetas, r)$ into $P+1=\lceil\log(Kr/\sigma_0)\rceil+1$ sets via following arguments,
\begin{align*}
	\Bc^d(\btetas, r)=\cup_{i=0}^P \Sc_i\quad\text{where}\quad\Sc_i=\begin{cases}\Bc^d(\btetas,\sigma_0/K)&\text{if}\quad i=0, \\ \Bc^d(\btetas,\min(r,\e^i\sigma_0/K))-\Bc^d(\btetas,\e^{i-1}\sigma_0/K) &\text{else}. \end{cases}
\end{align*}
By Assumption \ref{assume lip}, with probability at least $1-p_0$, $\nabla \Lch_\Sc(\bt),~\nabla \Lc_{\Dc}(\bt)$ are $L_\Dc$-Lipschitz. Given a set $\Sc_i$ and the associated radius $r_i=\min(r,\e^i\sigma_0/K)$, pick an $\eps_i\leq r_i\leq r$ covering $\Nc_i$ of the set $\Sc_i\subset \Bc^d(\btetas, r_i)$ such that $\log |\Nc_i|\leq d\log (3r_i/\eps_i)$. Observe that over $\Sc_i$, by construction, we have
\begin{align}
	\max(\sigma_0/K,\tn{\bt-\bt_\star})\leq r_i\leq \max(\sigma_0/K,\e\tn{\bt-\bt_\star})\label{const eq}.
\end{align}
Applying Lemma \ref{subexp lem} together with a union bound over the $P+1$ covers and elements of the covers, we guarantee the following: Within all covers $\Nc_i$, gradient vector at all points $\bt\in \Nc_i$ satisfies
\begin{align}
	\tn{\nabla\Lch_\Sc(\bt)-\nabla\Lc_\Dc(\bt)}\lesssim (\sigma_0+Kr_i)\log(3r_i/\eps_i)\sqrt{d/N},
\end{align}
with probability at least $1-\sum_{i=0}^P\exp(-100d\log(3r_i/\eps_i))$. Given both events hold with probability at least $1-p_0-\sum_{i=0}^P\exp(-100d\log(3r_i/\eps_i))$, for any $\bt\in\Sc_i$, pick $\bt'\in \Nc_i$ so that $\tn{\bt-\bt'}\leq \eps$. This yields
\begin{align}
	&\tn{\nabla\Lch_\Sc(\bt)-\nabla\Lc_\Dc(\bt)} \nn \\
	&\quad\quad\quad\quad \leq \tn{\nabla\Lc_\Dc(\bt)-\nabla \Lc_\Dc(\bt')}+ \tn{\nabla\Lch_\Sc(\bt)-\nabla \Lch_\Sc(\bt')}+\tn{\nabla \Lc_\Dc(\bt')-\nabla \Lch_\Sc(\bt')},\nn\\
	&\quad\quad\quad\quad\lesssim \eps_i L_\Dc +(\sigma_0+Kr_i)\log(3r_i/\eps_i)\sqrt{d/N}.
\end{align}
Setting $\eps_i=\min(1, \frac{K}{L_\Dc}\sqrt{d/N})r_i$ for $0\leq i\leq P$, for any $\bt\in \Sc_i$ (and thus for any $\bt\in\Bc^d(\btetas, r)$), we have
\begin{align}
	\tn{\nabla\Lch_\Sc(\bt)-\nabla\Lc_\Dc(\bt)} &\lesssim (\sigma_0+Kr_i)\log(3(1+L_\Dc N/K))\sqrt{d/N}, \nn \\
	&\lesssim (\sigma_0+K\tn{\bt-\bt_\star})\log(3(1+L_\Dc N/K))\sqrt{d/N},
\end{align}
where we used \eqref{const eq} to get the last inequality. Finally, observing that $\log(3r_i/\eps_i)\geq 1$, the probability bound simplifies to
\begin{align}
	1-p_0-\sum_{i=0}^P\exp(-100d\log(3r_i/\eps_i))\geq 1-p_0-\log(\frac{Kr}{\sigma_0})\exp({-100d}).
\end{align}
This completes the proof.
\end{proof}
\subsection{Proof of Lemma~\ref{lemma:bounded states}}
\begin{proof}
Suppose $\h_0=0$. We claim that $\tn{\h_t}\leq \beta_+\sqrt{n}(1-\rho^t)$ with probability at least $1-p_0$, where $\beta_+:=C_\rho(\sigma+B)/(1-\rho)$. Note that, using the bounds on $\z_t,\w_t$, the state vector $\h_1$ satisfies the following bound and obeys the induction
\begin{align}
	\tn{\h_1}\leq B\sqrt{n}+\sigma\sqrt{n}\leq C_\rho\sqrt{n}(B+\sigma)=\beta_+\sqrt{n}(1-\rho^1).
\end{align}
Suppose the bound holds until $t-1$, where $ t\leq T$, and let us apply induction. First observe that $\tn{\h_{t,L}}$ obeys the same upper bound as $\tn{\h_{L}}$ by construction. Recalling~\eqref{eqn:truncation impact}, we get the following by induction
\begin{align}
	\tn{\h_t-\h_{t,t-1}} \leq C_\rho\rho^{t-1}\tn{\h_{1}} \implies \tn{\h_t}&\leq C_\rho\rho^{t-1}\tn{\h_{1}}+\tn{\h_{t,t-1}}, \nn \\
	\tn{\h_t} &\leqsym{a} C_\rho\rho^{t-1}\tn{\h_{1}} + \beta_+\sqrt{n}(1-\rho^{t-1}), \nn\\
	&\leqsym{b} \sqrt{n}(C_\rho\rho^{t-1}(\sigma+B)+ \beta_+(1-\rho^{t-1})),\nn \\
	&\leq  \beta_+\sqrt{n}(1-\rho^{t}),\label{eqn:states bound recursion}
\end{align}
where, we get (a) from the induction hypothesis and (b) from the bound on $\h_1$. This bound also implies $\tn{\h_t}\leq \beta_+\sqrt{n}$ with probability at least $1-p_0$, for all $0 \leq t \leq T$, and completes the proof.
\end{proof}

\subsection{Proof of Lemma~\ref{lemma:independence}}
\begin{proof}
By construction $\bh^{(i)}$ only depends on the vectors $\{\z_t,\w_t\}_{t = \tau+(i-1)L+1}^{\tau+iL-1}$. Note that the dependence ranges $[\tau+(i-1)L+1,\tau+iL-1]$ are disjoint intervals for each $i's$. Hence, $\{\bh^{(i)}\}_{i=1}^{N}$ are all independent of each other. To show the independence of $\{\bh^{(i)}\}_{i=1}^{N}$ and $\{\z^{(i)}\}_{i=1}^{N}$, observe that the inputs $\z^{(i)} = \z_{\tau+iL}$ have timestamps $\tau+iL$; which is not covered by $[\tau+(i-1)L+1,\tau+iL-1]$ - the dependence ranges of $\{\bh^{(i)}\}_{i=1}^{N}$. Identical argument shows the independence of $\{\bh^{(i)}\}_{i=1}^{N}$ and $\{\w^{(i)}\}_{i=1}^{N}$. Lastly, $\{\z^{(i)}\}_{i=1}^{N}$ and $\{\w^{(i)}\}_{i=1}^{N}$ are independent of each other by definition. Hence, $\{\bh^{(i)}\}_{i=1}^{N},\{\z^{(i)}\}_{i=1}^{N},\{\w^{(i)}\}_{i=1}^{N}$ are all independent of each other. This completes the proof.
\end{proof}

\subsection{Proof of Theorem~\ref{thrm:diff truncated vs actual}}
\begin{proof} 
Our proof consists of two parts. The first part bounds the Euclidean distance between the truncated and non-truncated losses while the second part bounds the Euclidean distance between their gradients.

\noindent $\bullet$ {\bf{Convergence of loss:}}
To start, recall $\Lch(\bteta)$ and $\Ltr(\bteta)$ from~\eqref{eqn:minimization problem} and \eqref{eqn:truncated loss} respectively. The distance between them can be bounded as follows.
\begin{align}
	&|\Lch(\bteta) - \Ltr(\bteta)| \nn \\ 
	&\quad\quad\quad = |\frac{1}{2(T-L)}\sum_{t=L}^{T-1}\tn{\h_{t+1}-\bphi(\h_t,\z_t;\bteta)}^2 -\frac{1}{2(T-L)}\sum_{t=L}^{T-1}\tn{\h_{t+1,L}-\bphi(\h_{t,L-1},\z_t;\bteta)}^2|, \nn\\
	&\quad\quad\quad \leq \frac{1}{2(T-L)}\sum_{t=L}^{T-1}|\tn{\h_{t+1}-\bphi(\h_t,\z_t;\bteta)}^2-\tn{\h_{t+1,L}-\bphi(\h_{t,L-1},\z_t;\bteta)}^2|, \nn\\
	&\quad\quad\quad \leq\frac{1}{2} \max_{L\leq t\leq (T-1)}| \tn{\h_{t+1}-\bphi(\h_t,\z_t;\bteta)}^2-\tn{\h_{t+1,L}-\bphi(\h_{t,L-1},\z_t;\bteta)}^2|, \nn\\
	&\quad\quad\quad\leq\frac{1}{2}\tn{\bphi(\h,\z;\btetas)+\w-\bphi(\h,\z;\bteta)}^2-\tn{\bphi(\bh,\z;\btetas)+ \w-\bphi(\bh,\z;\bteta)}^2|, \nn\\
	&\quad\quad\quad\leq \frac{1}{2} (|\tn{\bphi(\h,\z;\btetas)+\w-\bphi(\h,\z;\bteta)}-\tn{\bphi(\bh,\z;\btetas)+\w-\bphi(\bh,\z;\bteta)}|) \nn\\
	&\quad\quad\quad ~~~~~(|\tn{\bphi(\h,\z;\btetas)+\w-\bphi(\h,\z;\bteta)}+\tn{\bphi(\bh,\z;\btetas)+\w-\bphi(\bh,\z;\bteta)}|),
\end{align}
where, $(\h,\bh, \z,\w)$ corresponds to the maximum index~($\bh$ be the truncated state) and we used the identity $a^2-b^2 = (a+b)(a-b)$. Denote the $k_{\rm th}$ element of $\bphi(\h,\z;\bteta)$ by $\bphi_k(\h,\z;\bteta)$ and that of $\w$ by $w_k$ for $1\leq k\leq n$. To proceed, using Mean-value Theorem, with probability at least $1-p_0$, we have
\begin{align}
	|\bphi_k(\h,\z;\btetas)-\bphi_k(\h,\z;\bteta)+w_k| &\leq \sigma+ \sup_{\btetat \in [\bteta,\btetas]}\tn{\nabla_{\bteta}\bphi_k(\h,\z;\btetat)}\tn{\bteta-\btetas}, \nn \\
	& \leq \sigma+C_{\bphi}\tn{\bteta-\btetas}\quad \text{for all} \quad 1 \leq k \leq n,\label{perturb beta} \\
	\implies \tn{\bphi(\h,\z;\btetas)+\w-\bphi(\h,\z;\bteta)} &\leq \sqrt{n}\max_{1\leq k\leq n}|\bphi_k(\h,\z;\btetas)-\bphi_k(\h,\z;\bteta)+w_k|,\nn\\
	& \leq \sqrt{n}(\sigma+C_{\bphi}\tn{\bteta-\btetas}).
\end{align}
This further implies that, with probability at least $1-p_0$, we have 
\begin{align}
	&\frac{1}{2}|\tn{\bphi(\h,\z;\btetas)+\w-\bphi(\h,\z;\bteta)}+\tn{\bphi(\bh,\z;\btetas)+\w-\bphi(\bh,\z;\bteta)}| \nn \\
	&\quad\quad\quad\quad\quad\quad\quad\quad\quad\quad\quad\quad\quad\quad\quad\quad\quad\quad\quad\quad\quad\quad\quad\quad \leq \sqrt{n}(\sigma+C_{\bphi}\tn{\bteta-\btetas}).\label{sum bound}
\end{align}
To conclude, applying the triangle inequality and using the Mean-value Theorem, the difference term $\Delta:=|\tn{\bphi(\h,\z;\btetas)+\w-\bphi(\h,\z;\bteta)}-\tn{\bphi(\bh,\z;\btetas)+\w-\bphi(\bh,\z;\bteta)}|$ is bounded as follows,
\begin{align}
	\Delta &\leq \tn{\bphi(\h,\z;\btetas)-\bphi(\h,\z;\bteta)- \bphi(\bh,\z;\btetas)+\bphi(\bh,\z;\bteta)}, \nn \\
	&\leq \tn{\bphi(\h,\z;\bteta)-\bphi(\bh,\z;\bteta)}+ \tn{\bphi(\h,\z;\btetas)-\bphi(\bh,\z;\btetas)}, \nn \\
	&\leq\sup_{\tilde{\h} \in [\h,\bh]}\|\nabla_{\h}\bphi(\tilde{\h},\z;\bteta)\|\tn{\h-\bh} +\sup_{\tilde{\h} \in [\h,\bh]} \|\nabla_{\h}\bphi(\tilde{\h},\z;\btetas)\|\tn{\h-\bh}, \nn \\
	&\leqsym{a} B_{\bphi}C_\rho\rho^{L-1}\beta_+\sqrt{n}+B_{\bphi}C_\rho\rho^{L-1}\beta_+\sqrt{n}, \nn \\
	& =2B_{\bphi}C_\rho\rho^{L-1}\beta_+\sqrt{n},
\end{align}
with probability at least $1-p_0$, where we get (a) from~\eqref{eqn:truncation impact} and the initial assumption that $\|\nabla_{\h}\bphi(\h,\z;\bteta)\| \leq B_{\bphi}$. Multiplying this bound with \eqref{sum bound} yields the advertised bound on the loss difference.

\noindent $\bullet$ {\bf{Convergence of gradients:}}	
Next, we take the gradients of $\Lch(\bteta)$ and $\Ltr(\bteta)$ to bound Euclidean distance between them. We begin with
\begin{align}
	\tn{\nabla{\Lch(\bteta)}-\nabla\Ltr(\bteta)}&\leq \frac{1}{T-L}\sum_{t=L}^{T-1}\|\nabla_{\bteta}{\bphi(\h_t,\z_t;\bteta)}^\top(\bphi(\h_t,\z_t;\bteta)-\h_{t+1})\nn \\
	&\quad\quad\quad\quad\quad-\nabla_{\bteta}{\bphi(\h_{t,L-1},\z_t;\bteta)}^\top(\bphi(\h_{t,L-1},\z_t;\bteta)-\h_{t+1,L})\|_{\ell_2}, \nn \\
	& \leq \max_{L\leq t\leq (T-1)} \|\nabla_{\bteta}{\bphi(\h_t,\z_t;\bteta)}^\top(\bphi(\h_t,\z_t;\bteta)-\h_{t+1}) \nn\\ &\quad\quad\quad\quad\quad-\nabla_{\bteta}{\bphi(\h_{t,L-1},\z_t;\bteta)}^\top(\bphi(\h_{t,L-1},\z_t;\bteta)-\h_{t+1,L})\|_{\ell_2}, \nn \\
	& \leq \|\nabla_{\bteta}{\bphi(\h,\z;\bteta)}^\top(\bphi(\h,\z;\bteta)-\bphi(\h,\z;\btetas)-\w)\nn\\ &\quad\quad\quad\quad\quad-\nabla_{\bteta}{\bphi(\bh,\z;\bteta)}^\top(\bphi(\bh,\z;\bteta)-\bphi(\bh,\z;\btetas)-\w)\|_{\ell_2}, \nn \\
	&\leq \sqrt{n}\Lambda ,\label{eqn:grad diff first}
\end{align}
where $(\h,\bh, \z,\w)$ corresponds to the maximum index~($\bh$ be the truncated state) and we define $\Lambda$ to be the entry-wise maximum
\begin{align}
	\Lambda&:= \max_{1\leq k\leq n}\|(\bphi_k(\h,\z;\bteta)-\bphi_k(\h,\z;\btetas)-w_k)\nabla_{\bteta}{\bphi_k(\h,\z;\bteta)} \nn \\
	&\quad\quad\quad-(\bphi_k(\bh,\z;\bteta)-\bphi_k(\bh,\z;\btetas)-w_k)\nabla_{\bteta}{\bphi_k(\bh,\z;\bteta)}\|_{\ell_2},
\end{align}
where $\bphi_k(\h,\z;\bteta)$ denotes the $k_{\rm th}$ element of $\bphi(\h,\z;\bteta)$. Without losing generality, suppose $k$ is the coordinate achieving maximum value and attaining $\Lambda$. Note that $\Lambda=\alpha(\h)-\alpha(\bar{\h})$ for some function $\alpha$, hence, using Mean-value Theorem as previously, we bound $\Lambda\leq \sup_{\htt \in [\h,\bh]} \|\nabla_{\h}\alpha(\htt)\|\tn{\h-\bar{\h}}$ as follows,
\begin{align}
	\Lambda  & \leq \sup_{\htt \in [\h,\bh]}\|(\bphi_k(\htt,\z;\bteta)-\bphi_k(\htt,\z;\btetas)-w_k)\nabla_{\h}\nabla_{\bteta}{\bphi_k(\htt,\z;\bteta)} \nn \\
	&\quad\quad\quad\quad+\nabla_{\bteta}{\bphi_k(\htt,\z;\bteta)}(\nabla_{\h}\bphi_k(\htt,\z;\bteta)^\top-\nabla_{\h}\bphi_k(\htt,\z;\btetas)^\top)\|\tn{\h -\bh}, \nn \\
	& \leq \sup_{\htt \in [\h,\bh]} \big[|\bphi_k(\htt,\z;\bteta)-\bphi_k(\htt,\z;\btetas)-w_k|\|\nabla_{\h}\nabla_{\bteta}{\bphi_k(\htt,\z;\bteta)}\| \nn \\
	& \quad \quad \quad\quad +\tn{\nabla_{\bteta}{\bphi_k(\htt,\z;\bteta)}}\tn{\nabla_{\h}\bphi_k(\htt,\z;\bteta)-\nabla_{\h}\bphi_k(\htt,\z;\btetas)} \big] \tn{\h -\bh}, \nn \\
	&\leqsym{a}  \sup_{\htt \in [\h,\bh]} \big[D_{\bphi}|\bphi_k(\htt,\z;\bteta)-\bphi_k(\htt,\z;\btetas)-w_k|\nn \\ &\quad\quad\quad\quad+C_{\bphi}\tn{\nabla_{\h}\bphi_k(\htt,\z;\bteta)-\nabla_{\h}\bphi_k(\htt,\z;\btetas)}\big]\tn{\h -\bh}, \label{eqn:grad diff second}
\end{align} 
where we get (a) from the initial assumptions $\tn{\nabla_{\bteta}{\bphi_k(\h,\z;\bteta)}} \leq C_{\bphi}$ and $\|\nabla_{\h}\nabla_{\bteta}{\bphi_k(\h,\z;\bteta)}\| \leq D_{\bphi}$. To proceed, again using Mean-value Theorem, we obtain
\begin{align} \sup_{\htt \in [\h,\bh]}\tn{\nabla_{\h}\bphi_k(\htt,\z;\bteta)-\nabla_{\h}\bphi_k(\htt,\z;\btetas)} &\leq \sup_{\substack{\htt \in [\h,\bh] \\ \btetat \in [\bteta,\btetas] }}\|\nabla_{\bteta}\nabla_{\h}\bphi_k(\htt,\z;\btetat)\|\tn{\bteta-\btetas}, \nn \\
	&\leq D_{\bphi}\tn{\bteta-\btetas}. \label{eqn:grad diff put2}
\end{align}
Finally, plugging the bounds from \eqref{perturb beta} and \eqref{eqn:grad diff put2} into \eqref{eqn:grad diff second}, with probability at least $1-p_0$, we have
\begin{align}
	\tn{\nabla{\Lch(\bteta)}-\nabla\Ltr(\bteta)} &\leq \sqrt{n}\Lambda,\nn \\ & \leq \sqrt{n}(D_{\bphi}(\sigma+C_{\bphi}\tn{\bteta-\btetas})+C_{\bphi}D_{\bphi}\tn{\bteta-\btetas})\tn{\h-\bar{\h}}, \nn \\
	&\leq 2n\beta_+ C_\rho\rho^{L-1}D_{\bphi}(\sigma+C_{\bphi}\tn{\bteta-\btetas}),
\end{align}
This completes the proof.
\end{proof}


\subsection{Proof of Theorem~\ref{combined thm}}
\begin{proof}
\red{Theorem~\ref{combined thm} is a direct consequence of combining the results from Sections~\ref{sec:accu SLWG} and \ref{sec: truncated results}.}
To begin our proof, consider the truncated sub-trajectory loss $\ltr_\tau$ from Definition~\ref{def:empirical loss} which also implies that $\Lcp(\bteta) = \E[\ltr_\tau(\bteta)]$. Hence, $\ltr_\tau$ it is a finite sample approximation of the Auxiliary loss $\Lcp$. 
To proceed, using Theorem~\ref{refined thm} with Assumptions~\ref{assume lip} and \ref{assume grad} on the Auxiliary loss $\Lcp$ and its finite sample approximation $\ltr_\tau$, with probability at least $1-Lp_0-L\log(\frac{Kr}{\sigma_0})\exp({-100d})$, for all $\bt\in\Bc^d(\btetas,r)$, we have
\begin{align}
\tn{\nabla\ltr_\tau(\bt)-\nabla\Lc_\Dc(\bt)}\leq c_0(\sigma_0+K\tn{\bt-\bt_\star})\log(3(L_\Dc N/K+1))\sqrt{d/N},
\end{align}
for all $0 \leq \tau \leq L-1$, where we get the advertised probability by union bounding over all $0 \leq \tau \leq L-1$. Next, observe that the truncated loss $\Ltr$ can be split into (average of) $L$~sub-trajectory losses via $ \Ltr(\bteta) = \frac{1}{L}\sum_{\tau=0}^{L-1}\ltr_\tau(\bteta)$. This implies that, with probability at least $1-Lp_0-L\log(\frac{Kr}{\sigma_0})\exp({-100d})$, for all $\bt\in\Bc^d(\btetas,r)$, we have
\begin{align}
\tn{\nabla\Ltr(\bteta) - \nabla\Lcp(\bteta)} &\leq \frac{1}{L}\sum_{\tau=0}^{L-1} \tn{\nabla\ltr_\tau(\bteta)- \nabla\Lcp(\bteta)}, \nn \\ 
&\leq \max_{0\leq \tau \leq(L-1)} \tn{\nabla\ltr_\tau(\bteta)- \nabla\Lcp(\bteta)}, \nn\\
&\leq c_0(\sigma_0+K\tn{\bt-\bt_\star})\log(3(L_\Dc N/K+1))\sqrt{d/N}.
\end{align}
Combining this with Theorem~\ref{thrm:diff truncated vs actual}, with the advertised probability, for all $\bt\in\Bc^d(\btetas,r)$, we have
\begin{align}
&\tn{\Lch(\bteta) - \Lcp(\bteta) } \nn \\
&\quad\quad \leq \tn{\Ltr(\bteta) - \Lcp(\bteta) } + \tn{\Lch(\bteta) - \Ltr(\bteta) }, \nn \\
&\quad\quad \leq c_0(\sigma_0+K\tn{\bt-\bt_\star})\log(3(L_\Dc N/K+1))\sqrt{d/N} + 2n\beta_+ C_\rho\rho^{L-1}D_{\bphi}(\sigma+C_{\bphi}\tn{\bteta-\btetas}). \nn
\end{align}
To simplify the result further, we pick $L$ to be large enough so that the second term in the above inequality becomes smaller than or equal to the first one. This is possible when
\begin{align}
&2n\beta_+ C_\rho\rho^{L-1}D_{\bphi} \leq c_0(\sigma_0/\sigma \land K/C_{\bphi})\log(3(L_\Dc N/K+1))\sqrt{d/N}, \nn \\ 
\iff \quad &\rho^{L-1} \leq (\sigma_0/\sigma \land K/C_{\bphi})\frac{c_0\log(3(L_\Dc N/K+1))\sqrt{d/N}}{2n\beta_+ C_\rho D_{\bphi}}, \nn \\
\iff \quad & L \geq 1 + \big[\log\big(\frac{2 n\beta_+ C_\rho D_{\bphi}\sqrt{N/d}}{c_0\log(3(L_\Dc N/K+1))} \big) + \log(\sigma/\sigma_0 \lor C_{\bphi}/K)\big]/\log(\rho^{-1}), \nn \\
\Longleftarrow \quad\;  & L = \big\lceil 1 + \frac{\log((2/c_0) n\beta_+ C_\rho D_{\bphi}\sqrt{N/d}(\sigma/\sigma_0 \lor C_{\bphi}/K))}{1 - \rho} \big\rceil. \label{eqn: choice of L} 
\end{align}
Hence, picking $L$ via \eqref{eqn: choice of L}, with probability at least $1-2Lp_0-L\log(\frac{Kr}{\sigma_0})\exp({-100d})$, for all $\bteta \in \Bc^d(\btetas,r)$, we have
\begin{align}
\tn{\nabla\Lch(\bt)-\nabla\Lc_\Dc(\bt)}\leq 2c_0(\sigma_0+K\tn{\bt-\bt_\star})\log(3(L_\Dc N/K+1))\sqrt{d/N}.
\end{align}
This completes the proof.
\end{proof}
\subsection{Proof of Theorem~\ref{thrm:main thrm}}
\begin{proof} \red{The proof of Theorem~\ref{thrm:main thrm} readily follows from combining our gradient convergence result with Theorem~\ref{one point thm}. We begin by picking $N \geq 16c_0^2 K^2 \log^2(3(L_\Dc N/K+1))d/\alpha^2$ in Theorem~\ref{combined thm} to obtain
\begin{align}
	\tn{\nabla{\Lch(\bheta)}-\nabla{\Lcp(\bheta)}} 
	\leq (\alpha/2)\tn{\bteta-\btetas} +  2c_0\sigma_0\log(3(L_\Dc N/K+1))\sqrt{d/N}, \label{eqn:grad diff proof}
\end{align}
with probability at least $1-2Lp_0-L\log(\frac{Kr}{\sigma_0})\exp({-100d})$ for all $\bheta \in \Bc^d(\bhetas, r)$. We then use Theorem~\ref{one point thm} with $\nu = 2c_0\sigma_0\log(3(L_\Dc N/K+1))\sqrt{d/N}$ and set $c = 10c_0$ to get the statement of the theorem. Lastly, observe that by choosing $N \geq 16c_0^2K^2 \log^2(3(L_\Dc N/K+1))d/\alpha^2$, the statistical error rate of our non-asymptotic identification can be upper bounded as follows,}
\begin{align}
\frac{5\nu}{\alpha} = \frac{10c_0\sigma_0}{\alpha}\log(3(L_\Dc N/K+1))\sqrt{d/N} \lesssim \sigma_0/K.
\end{align}
Therefore, to ensure that Theorem \ref{one point thm} is applicable, we assume that the noise is small enough, so that $\sigma_0 \lesssim r K$. This completes the proof.
\end{proof}
\subsection{Proof of Theorem~\ref{thrm:main theorem separable}}
\begin{proof}
\red{Our proof strategy is similar to that of Theorem~\ref{thrm:main thrm}, that is, we first show the gradient convergence result for each component $\Lch_k$ of the empirical loss $\Lch$. We then use Theorem~\ref{one point thm} to learn the dynamics of separable dynamical systems using finite samples obtained from a single trajectory.}

\noindent $\bullet$ \red{{\bf{Uniform gradient convergence:}}
In the case of separable dynamical systems, Assumption~\ref{assume lip} states that, there exist numbers $L_\Dc,p_0>0$ such that with probability at least $1-p_0$ over the generation of data, for all pairs $\bteta,\bteta' \in \Bc^d(\btetas,r)$, the gradients of empirical and population losses in \eqref{eqn:LS_and_LD sep} satisfy 
\begin{align}
	\max(\tn{\nabla\Lc_{k,\Dc}(\bt_k)-\nabla\Lc_{k,\Dc}(\bt_k')},\tn{\nabla\Lch_{k,\cal{S}}(\bt_k)-\nabla\Lch_{k,\cal{S}}(\bt_k')})\leq L_\Dc\tn{\bt_k-\bt_k'}, \label{proof of (a)}
\end{align}
for all $1 \leq k \leq n$. Similarly, Assumption~\ref{assume grad} states that, there exist scalars $K,\sigma_0 > 0$ such that, given $\x \sim \Dc$, at any point $\bt$, the subexponential norm of the gradient is upper bounded as a function of the noise level $\sigma_0$ and distance to the population minimizer via
\begin{align}
	\te{\nabla\Lc_k(\bt_k,\x)-\E[\nabla\Lc_k(\bt_k,\x)}
	\leq \sigma_0 + K\tn{\bt_k-\bteta_k^{\star}} \quad \text{for all} \quad 1 \leq k \leq n. \label{proof of (b)}
\end{align}
To proceed, using Theorem~\ref{refined thm} with Assumptions~\ref{assume lip} and \ref{assume grad} replaced by \eqref{proof of (a)} and \eqref{proof of (b)} respectively, with probability at least $1-np_0-n\log(\frac{Kr}{\sigma_0})\exp({-100 \bad})$, for all $\bt \in\Bc^d(\btetas,r)$ and $1 \leq k \leq n$, we have
\begin{align}
	\tn{\nabla\Lch_{k,\cal{S}}(\bt_k)-\nabla\Lc_{k,\Dc}(\bt_k)}\leq c_0(\sigma_0+K\tn{\bt_k-\bteta_k^{\star}})\log(3(L_\Dc N/K+1))\sqrt{\bad/N}. \label{uniform conv using (a) and (b)}
\end{align}}

\noindent $\bullet$ \red{{\bf{Small impact of truncation:}}
Next, we relate the gradients of the single trajectory loss $\Lch_k$ in~\eqref{eqn:empirical loss sep} and the multiple trajectory loss $\Ltr_k$~(defined below). Similar to~\eqref{eqn:empirical loss sep}, the truncated loss for separable dynamical systems is alternately given by 
\begin{align}
	\Ltr(\bteta) = \sum_{k=1}^n \Ltr_k(\bteta_k),
	\text{where}\;\Ltr_k(\bteta_k) 
	:= \frac{1}{2(T-L)}\sum_{t=L}^{T-1}(\h_{t+1,L}[k]-\bphi_k(\h_{t,L-1},\z_{t};\bteta_k))^2,\label{eqn:truncated loss sep}
\end{align}
where $\h_{t,L}[k]$ denotes the $k_{\rm th}$ element of the truncated vector $\h_{t,L}$. We remark that Assumptions~\ref{assum stability} and \ref{ass bounded} are same for both non-separable and separable dynamical systems. Therefore, repeating the same proof strategy of Theorem~\ref{thrm:diff truncated vs actual}, with $\Ltr$ and $\Lch$ replaced by $\Ltr_k$ and $\Lch_k$ respectively, with probability at least $1-np_0$, for all $\bteta \in \Bc^d(\btetas, r)$ and $1 \leq k \leq n$, we have 
\begin{align}
	\tn{\nabla{\Lch_k(\bteta_k)}-\nabla\Ltr_k(\bteta_k)} &\leq 2n\beta_+ C_\rho\rho^{L-1}D_{\bphi}(\sigma+C_{\bphi}\tn{\bteta_k-\bteta_k^{\star}}).\label{truncation impact sep}
\end{align}}

\noindent $\bullet$ \red{{\bf{Combined result:}} Next, we combine \eqref{uniform conv using (a) and (b)} and \eqref{truncation impact sep} to obtain a uniform convergence result for the gradient of the empirical loss $\Lch_k$. Observe that, similar to $\Ltr$, the truncated loss $\Ltr_k$ can also be split into $L$~truncated sub-trajectory losses~(see the proof of Theorem~\ref{combined thm}). Each of these truncated sub-trajectory loss is statistically identical to $\Lch_{k,\Sc}$. Therefore, using a similar line of reasoning as we did in the proof of Theorem~\ref{combined thm}, with probability at least $1-Lnp_0-Ln\log(\frac{Kr}{\sigma_0})\exp({-100 \bad})$, for all $\bt \in \Bc^d(\btetas,r)$ and $1 \leq k \leq n$, we have
\begin{align}
	\tn{\nabla\Ltr_k(\bt_k)-\nabla\Lc_{k,\Dc}(\bt_k)}\leq c_0(\sigma_0+K\tn{\bt_k-\bteta_k^{\star}})\log(3(L_\Dc N/K+1))\sqrt{\bad/N}. \label{eqn:L_trk vs L_Dk}
\end{align}
Combining} this with \eqref{truncation impact sep}, with probability at least $1-Lnp_0-Ln\log(\frac{Kr}{\sigma_0})\exp({-100 \bad})$, for all $\bt \in\Bc^d(\btetas,r)$ and $1 \leq k \leq n$, we have
\begin{align}
&\tn{\nabla\Lch_k(\bteta_k) - \nabla\Lc_{k,\Dc}(\bteta_k)}\nn \\
&\quad\leq \tn{\nabla\Ltr_k(\bteta_k) - \nabla\Lc_{k,\Dc}(\bteta_k)} + \tn{\nabla\Lch_k(\bteta_k) - \nabla\Ltr_k(\bteta_k)},  \nn \\ 
&\quad \leq c_0 (\sigma_0+K\tn{\bt_k-\bteta_k^{\star}})\log(3(L_\Dc N/K+1))\sqrt{\bad/N} + 2n\beta_+ C_\rho\rho^{L-1}D_{\bphi}(\sigma+C_{\bphi}\tn{\bteta_k-\bteta_k^{\star}}). \nn	\label{eqn:result only1}
\end{align}
To simplify the result further, we pick $L$ to be large enough so that the second term in the above inequality becomes smaller than or equal to the first one. This is possible when
\begin{align}
L = \big\lceil 1 + \frac{\log((2/c_0)n\beta_+ C_\rho D_{\bphi}\sqrt{N/\bad}(\sigma/\sigma_0 \lor C_{\bphi}/K))}{1-\rho} \big\rceil.
\end{align}
Hence, picking $L$ as above, with probability at least $1-2Lnp_0-Ln\log(\frac{Kr}{\sigma_0})\exp({-100\bad})$, for all $\bt \in \Bc^d(\btetas,r)$ and $1 \leq k \leq n$, we have
\begin{align}
\tn{\nabla\Lch_k(\bt_k)-\nabla\Lc_{k,\Dc}(\bt_k)} &\leq  2c_0(\sigma_0+K\tn{\bt_k-\bteta_k^{\star}})\log(3(L_\Dc N/K+1))\sqrt{\bad/N}, \nn \\
&\leqsym{a}(\alpha/2)\tn{\bt_k-\bteta_k^{\star}} + 2c_0\sigma_0\log(3(L_\Dc N/K+1))\sqrt{\bad/N}, \label{uniform conv sep final}
\end{align}
where we get (a) by choosing $N \geq 16c_0^2K^2 \log^2(3(L_\Dc N/K+1))\bad/\alpha^2$.

\noindent $\bullet$ \red{{\bf{One-point convexity \& smoothness:}} Lastly, Assumption~\ref{assum grad dominance} on the Auxiliary loss $\Lc_{k,\Dc}$ states that, there exist scalars $\beta \geq \alpha>0$ such that, for all $\bt \in \Bc^d(\btetas,r)$ and $1 \leq k \leq n$, the auxiliary loss $\Lc_{k,\Dc}(\bt_k)$ of \eqref{eqn:LS_and_LD sep} satisfies
\begin{align}
	\li\bteta_k-\bteta_k^{\star},\nabla{\Lc_{k,\Dc}(\bteta_k)}\ri&\geq\alpha \tn{\bteta_k-\bteta_k^{\star}}^2, \label{OPC sep} \\
	\tn{\nabla{\Lc_{k,\Dc}(\bteta_k)}} &\leq \beta\tn{\bteta_k-\bteta_k^{\star}}. \label{smooth sep}
\end{align}
\noindent $\bullet$ {\bf{Finalizing the proof:}}
We are now ready to use Theorem~\ref{one point thm} with gradient concentration bound given by \eqref{uniform conv sep final} and the OPC/smoothness Assumptions given by~\eqref{OPC sep} and \eqref{smooth sep}. Specifically, we use Theorem~\ref{one point thm} with $\nu = 2c_0 \sigma_0 \log(3(L_\Dc N/K+1))\sqrt{\bad/N}$, the OPC assumption \eqref{OPC sep} and the smoothness assumption \eqref{smooth sep} to get the statement of the theorem. This completes the proof.}
\end{proof}

\section{Proof of Corollaries~\ref{thrm:main thrm LDS} and \ref{thrm:main thrm RNN}}\label{sec: appenix A}
\subsection{Application to Linear Dynamical Systems}\label{sec:proofs of LDS}
\subsubsection{Verification of Assumption~\ref{assum stability}}
The following lemma states that a linear dynamical system satisfies $(C_\rho,\rho)$-stability if the spectral radius $\rho(\As) <1$.
\begin{lemma}[$(C_\rho,\rho)$-stability]\label{lemma:truncation impact linear}
	Fix excitations $(\z_t)_{t=0}^\infty$ and noise $(\w_t)_{t=0}^\infty$. Denote the state sequence~\eqref{eqn:general app}~($\phi = \Iden_n$) resulting from initial state $\h_0=\bbalpha$, $(\z_{\tau})_{\tau = 0}^t$ and $(\w_{\tau})_{\tau = 0}^t$ by $\h_t(\bbalpha)$. Suppose $\rho(\As) <1$. Then, there exists $C_\rho \geq 1$ and $\rho \in (\rho(\As),1)$ such that
	$
	\tn{\h_t(\bbalpha)-\h_t(0)}\leq C_\rho \rho^t\tn{\bbalpha}.
	$
\end{lemma}
\begin{proof}
	To begin, consider the difference,
	\begin{align}
		\h_t(\bbalpha) - \h_t(0) &= \As\h_{t-1}(\bbalpha) +\Bs\z_{t-1} - \As\h_{t-1}(0) -\Bs\z_{t-1} = \As(\h_{t-1}(\bbalpha)-\h_{t-1}(0)). \nn
	\end{align}
	Repeating this recursion till $t = 0$ and taking the norm, we get
	\begin{align}
		\tn{\h_t(\bbalpha) - \h_t(0)} = \tn{\A^t_\star(\bbalpha-0)} \leq  \|\A^t_\star\|\tn{\bbalpha}.
	\end{align}
	Given $\rho(\As)<1$, as a consequence of Gelfand's formula, there exists $C_\rho \geq 1$ and $\rho \in (\rho(\As),1)$ such that, $\|\A^t_\star\| \leq C_\rho \rho^t$, for all $t \geq 0$. Hence, $\tn{\h_t(\bbalpha) - \h_t(0)} \leq C_\rho \rho^t \tn{\bbalpha}$. This completes the proof.
\end{proof}

\subsubsection{Verification of Assumption~\ref{ass bounded}}
To show that the states of a stable linear dynamical system are bounded with high probability, we state a standard Lemma from~\cite{oymak2019stochastic} that bounds the Euclidean norm of a subgaussian vector.
\begin{lemma}\label{lemma:subgaussian vector bound}
	Let $\ab \in \R^n$ be a zero-mean subgaussian random vector with $\tsub{\ab} \leq L$. Then for any $m \geq n$, there exists $C>0$ such that
	\begin{align}
		\P(\tn{\ab} \leq CL \sqrt{m}) \geq 1-2\exp(-100m).
	\end{align}
\end{lemma}
To apply Lemma~\ref{lemma:subgaussian vector bound}, we require the subgaussian norm of the state vector $\h_t$ and the concatenated vector $\x_t$. We will do that by first bounding the corresponding covariance matrices as follows.
\begin{theorem}[Covariance bounds]\label{thrm:covariance of h_t bounds} Consider the LDS in ~\eqref{eqn:general app}~with $\phi = \Iden_n$. Suppose $\z_t \distas \Nc(0,\Iden_p)$ and $\w_t \distas \Nc(0,\sigma^2\Iden_n)$. Let $\Gb_t$ and $\Fb_t$ be as in~\eqref{eqn:Gram matrices}. Then, the covariance matrix of the vectors $\h_t$ and $\x_t = [\h_t^\top~\z_t^\top]^\top$ satisfies 
	\begin{align}
		\lmn{\Gb_t\Gb_t^\top+\sigma^2\Fb_t\Fb_t^\top} \Iden_n \preceq  &\bSi[\h_t]  \preceq  \lmx{\Gb_t\Gb_t^\top+\sigma^2\Fb_t\Fb_t^\top}\Iden_n, \\
		(1 \land \lmn{\Gb_t\Gb_t^\top+\sigma^2\Fb_t\Fb_t^\top}) \Iden_{n+p} \preceq  &\bSi[\x_t]  \preceq  (1 \lor \lmx{\Gb_t\Gb_t^\top+\sigma^2\Fb_t\Fb_t^\top})\Iden_{n+p},	 
	\end{align}
\end{theorem}
\begin{proof}
	We first expand the state vector $\h_t$ as a sum of two independent components $\g_t$ and $\omg_t$ as follows,
	\begin{align}
		\h_{t} = \underbrace{\sum_{i=0}^{t-1}\A^{t-1-i}_\star\Bs \z_i}_{\g_t} + \underbrace{\sum_{i=0}^{t-1}\A^{t-1-i}_\star \w_i}_{\omg_t}. \label{eqn:h_t split}
	\end{align} 
	Observe that, $\g_t$ denotes the state evolution due to control input and $\omg_t$ denotes the state evolution due to noise. Furthermore, $\g_t$ and $\omg_t$ are both independent and zero-mean. Therefore, we have 
	\begin{align}
		\bSi[\h_t] &= \bSi[\g_t+\omg_t] = \bSi[\g_t]+\bSi[\omg_t] 
		= \E[\g_t\g_t^\top] + \E[\omg_t \omg_t^\top] \nn \\
		&= \sum_{i=0}^{t-1}\sum_{j=0}^{t-1}(\A^i_\star)\Bs\E[\z_i\z_j^\top]\B^\top_\star(\A^j_\star)^\top+\sum_{i=0}^{t-1}\sum_{j=0}^{t-1}(\A^i_\star)\E[\w_i\w_j^\top](\A^j_\star)^\top \nn \\
		&\eqsym{a} \sum_{i=0}^{t-1}(\A^i_\star)\Bs\B^\top_\star(\A^i_\star)^\top+\sigma^2\sum_{i=0}^{t-1}(\A^i_\star)(\A^i_\star)^\top, 
	\end{align}
	where we get (a) from the fact that $\E[\z_i\z_j^\top] = \Iden_p$ and $\E[\w_i\w_j^\top] = \sigma^2\Iden_n$ when $i=j$, and zero otherwise. To proceed, let $\Gb_t := [\A^{t-1}_\star\B_\star~\A^{t-2}_\star\B_\star~\cdots~\B_\star]$ and $\Fb_t := [\A^{t-1}_\star~\A^{t-2}_\star~\cdots~\Iden_n]$. Observing $\Gb_t\Gb_t^\top = \sum_{i=0}^{t-1}(\A^i_\star)\Bs\B^\top_\star(\A^i_\star)^\top$ and $\Fb_t\Fb_t^\top = \sum_{i=0}^{t-1}(\A^i_\star)(\A^i_\star)^\top$, we obtain the following bounds on the covariance matrix of the state vector $\h_t$ and the concatenated vector $\x_t = [\h_t^\top~\z_t^\top]^\top$.
	\begin{align}
		\lmn{\Gb_t\Gb_t^\top+\sigma^2\Fb_t\Fb_t^\top} \Iden_n \preceq  &\bSi[\h_t]  \preceq  \lmx{\Gb_t\Gb_t^\top+\sigma^2\Fb_t\Fb_t^\top}\Iden_n, \\
		(1 \land \lmn{\Gb_t\Gb_t^\top+\sigma^2\Fb_t\Fb_t^\top}) \Iden_{n+p} \preceq  &\bSi[\x_t]  \preceq  (1 \lor \lmx{\Gb_t\Gb_t^\top+\sigma^2\Fb_t\Fb_t^\top})\Iden_{n+p},	 
	\end{align}
	where to get the second relation, we use the fact that $\bSi[\z_t] = \Iden_p$. This completes the proof.
\end{proof}
\red{Once we bound the covariance matrices, using standard bounds on the subgaussian norm of a random vector, we find that $\tsub{\h_t} \lesssim \sqrt{\bSi[\h_t]} \leq \sqrt{\lmx{\Gb_t\Gb_t^\top+\sigma^2\Fb_t\Fb_t^\top}}$ and $\tsub{\x_t} \lesssim \sqrt{\bSi[\x_t]} \leq 1 \lor
	\sqrt{ \lmx{\Gb_t\Gb_t^\top+\sigma^2\Fb_t\Fb_t^\top}}$. Combining these with Lemma~\ref{lemma:subgaussian vector bound}, we find that, with probability at least $1-4T\exp(-100n)$, for all $1 \leq t \leq T$, we have  $\tn{\h_{t}} \leq c\sqrt{\beta_+n}$ and $\tn{\x_{t}} \leq c_0 \sqrt{\beta_+(n+p)}$, where we set $\beta_+ = 1 \lor \max_{1 \leq t\leq T}\lmx{\Gb_t\Gb_t^\top+\sigma^2\Fb_t\Fb_t^\top}$. This verifies Lemma~\ref{lemma:bounded states} and consequently Assumption~\ref{ass bounded}.}
\subsubsection{Verification of Assumption~\ref{assum grad dominance}}
Recall that, we define the following concatenated vector/matrix for linear dynamical systems: $\x_t = [\h_t^\top~\z_t^\top]^\top$ and $\Bhetas = [\As~\Bs]$. Let $\bteta_k^{\star \top}$ denotes the $k_{\rm th}$ row of $\Bhetas$. Then, the auxiliary loss for linear dynamical system is defined as follows,
\begin{align}
	\Lcp(\Bheta) = \sum_{k=1}^n \Lc_{k,\Dc}(\bteta_k), \quad \text{where} \quad \Lc_{k,\Dc}(\bteta_k) := \frac{1}{2}\E[(\h_L[k] -\bteta_k^\top\x_{L-1})^2]. \label{eqn:Lpop linear}
\end{align} 
Using the derived bounds on the covariance matrix, it is straightforward to show that the auxiliary loss satisfies the following one-point convexity and smoothness conditions.
\begin{lemma}[One-point convexity \& smoothness] \label{lemma:grad dominance linear} 
	Consider the setup of Theorem~\ref{thrm:covariance of h_t bounds} and the auxiliary loss given by~\eqref{eqn:Lpop linear}. Define $\bGma_t := \Gb_t\Gb_t^\top+\sigma^2\Fb_t\Fb_t^\top$. Let $\gamma_- := 1 \land \lmn{\bGma_{L-1}}$ and $\gamma_+ := 1 \lor \lmx{\bGma_{L-1}}$. For all $1 \leq k \leq n$, the gradient $\nabla{\Lc_{k,\Dc}(\bteta_k)}$ satisfies,
	\begin{align*}
		\li\bteta_k-\bteta_k^{\star},\nabla{\Lc_{k,\Dc}(\bteta_k)}\ri &\geq \gamma_- \tn{\bteta_k-\bteta_k^{\star}}^2,\\
		\tn{\nabla{\Lc_{k,\Dc}(\bteta_k)}} &\leq \gamma_+ \tn{\bteta_k-\bteta_k^{\star}}.
	\end{align*}
\end{lemma}
\begin{proof}
	To begin, we take the gradient of the auxiliary loss $\Lc_{k,\Dc}$ \eqref{eqn:Lpop linear} to get $\nabla{\Lc_{k,\Dc}(\bteta_k)} = \E[\x_{L-1}\x_{L-1}^\top(\bteta_k-\bteta_k^{\star}) - \x_{L-1}\w_{L-1}[k]]$. Note that, $\E[\x_{L-1}\w_{L-1}[k]] = 0$ for linear dynamical systems because $\w_{L-1}$ and $\x_{L-1}$ are independent and we have $\E[\w_{L-1}] = \E[\x_{L-1}] = 0$. Therefore, using Theorem~\ref{thrm:covariance of h_t bounds} with $t=L-1$, we get the following one point convexity bound,
	\begin{align}
		\li\bteta_k-\bteta_k^{\star},\nabla{\Lc_{k,\Dc}(\bteta_k)}\ri & =  \li\bteta_k-\bteta_k^{\star},\E[\x_{L-1}\x_{L-1}^\top](\bteta_k-\bteta_k^{\star})\ri, \nn \\
		& \geq \gamma_-\tn{\bteta_k-\bteta_k^{\star}}^2.
	\end{align}
	Similarly, we also have
	\begin{align}
		\tn{\nabla{\Lc_{k,\Dc}(\bteta_k)}} \leq \|\E[\x_{L-1}\x_{L-1}^\top]\|\tn{\bteta_k-\bteta_k^{\star}}  
		\leq  \gamma_+ \tn{\bteta_k-\bteta_k^{\star}}.
	\end{align} 
	This completes the proof.
\end{proof}
\subsubsection{Verification of Assumption~\ref{assume lip}}
Let $\Sc := (\h_L^{(i)},\h_{L-1}^{(i)},\z_{L-1}^{(i)})_{i=1}^N $ be $N$ i.i.d. copies of $(\h_L, \h_{L-1},\z_{L-1})$ generated from $N$ i.i.d. trajectories of the system \eqref{eqn:general app}~with $\phi = \Iden_n$. Let $\x_{L-1}^{(i)} := [\h_{L-1}^{(i)\top}~\z_{L-1}^{(i)\top}]^\top$ and $\Bheta:=[\A~\B]$ be the concatenated vector/matrix. Then, the finite sample approximation of the auxiliary loss $\Lc_{\Dc}$ is given by
\begin{align}
	\Lch_\Sc(\Bheta) = \sum_{k=1}^n \Lch_{k,\Sc}(\bteta_k), \quad \text{where} \quad \Lch_{k,\Sc}(\bteta_k) := \frac{1}{2N} \sum_{i=1}^N (\h_L^{(i)}[k] - \bteta_k^\top\x_{L-1}^{(i)})^2. \label{eqn:finite Lpop Linear}
\end{align}
The following lemma states that both $\nabla\Lc_{k,\Dc}$ and $\nabla\Lch_{k,\Sc}$ are Lipschitz with high probability. 
\begin{lemma}[Lipschitz gradient]\label{lemma:lipschitz grad linear}
	Consider the same setup of Theorem~\ref{thrm:covariance of h_t bounds}. Consider the auxiliary loss $\Lc_{k,\Dc}$ and its finite sample approximation $\Lch_{k,\cal{S}}$ from \eqref{eqn:Lpop linear} and \eqref{eqn:finite Lpop Linear} respectively. Let $\gamma_+ >0$ be as in Lemma~\ref{lemma:grad dominance linear}. For $N \gtrsim n+p$, with probability at least $1-2\exp(-100(n+p))$, for all pairs $\Bheta,\Bheta'$ and for all $1 \leq k \leq n $, we have
	\begin{align}
		\max(\tn{\nabla\Lc_{k,\Dc}(\bt_k)-\nabla\Lc_{k,\Dc}(\bt_k')},\tn{\nabla\Lch_{k,\cal{S}}(\bt_k)-\nabla\Lch_{k,\cal{S}}(\bt_k')})\leq 2 \gamma_+\tn{\bt_k-\bt_k'}.
	\end{align}
\end{lemma}
\begin{proof}
	To begin, recall the auxiliary loss from \eqref{eqn:Lpop linear}. We have that
	\begin{align}
		\tn{\nabla{\Lc_{k,\Dc}(\bteta_k)} - \nabla{\Lc_{k,\Dc}(\bteta_k')}} &= \tn{\E[\x_{L-1}\x_{L-1}^\top](\bteta_k-\bteta_k^\star)-\E[\x_{L-1}\x_{L-1}^\top](\bteta_k'-\bteta_k^\star)}, \nn \\
		&\leq \|\E[\x_{L-1}\x_{L-1}^\top]\|\tn{\bteta_k-\bteta_k'}, \nn \\
		&\leq \gamma_+ \tn{\bteta_{k}-\bteta_{k}'}. \label{eqn:lips LD}
	\end{align}
	To obtain a similar result for the finite sample loss $\Lch_{k,\cal{S}}$, we use Corollary 5.50 from~\cite{vershynin2010introduction} which bounds the concentration of empirical covariance around its population when the sample size is sufficiently large. Specifically, applying this corollary on the empirical covariance of $\x_{L-1}^{(i)}$ with $t=10,\eps=1$ shows that, for $N \gtrsim n+p$, with probability at least $1-2\exp(-100(n+p))$, we have
	\begin{align}
		\|\frac{1}{N}\sum_{i=1}^N\x_{L-1}^{(i)}(\x_{L-1}^{(i)})^\top-\E[\x_{L-1}\x_{L-1}^\top]\|\leq \gamma_+.
	\end{align}
	Thus, the gradient $\nabla\Lch_{k,\Sc}(\bteta_k)$ also satisfies the Lipschitz property, that is, for $N \gtrsim n+p$, with probability at least $1-2\exp(-100(n+p))$, we have
	\begin{align}
		&\tn{\nabla{\Lch_{k,\Sc}(\bteta_k)}-\nabla{\Lch_{k,\Sc}(\bteta_k')}} \nn \\
		&\quad\quad\quad\quad\quad \leq \tn{\frac{1}{N}\sum_{i=1}^N\x_{L-1}^{(i)}(\x_{L-1}^{(i)})^\top(\bteta_k-\bteta_k^\star)-\frac{1}{N}\sum_{i=1}^N\x_{L-1}^{(i)}(\x_{L-1}^{(i)})^\top(\bteta_k'-\bteta_k^\star)}, \nn \\
		& \quad\quad\quad\quad\quad\leq \|\frac{1}{N}\sum_{i=1}^N\x_{L-1}^{(i)}(\x_{L-1}^{(i)})^\top\|\tn{\bteta_k-\bteta_k'}, \nn \\
		&\quad\quad\quad\quad\quad \leq\big[\|\E[\x_{L-1}\x_{L-1}^\top]\| + \|\frac{1}{N}\sum_{i=1}^N\x_{L-1}^{(i)}(\x_{L-1}^{(i)})^\top-\E[\x_{L-1}\x_{L-1}^\top]\|\big]\tn{\bteta_k-\bteta_k'}, \nn \\
		&\quad\quad\quad\quad\quad\leq 2\gamma_+\tn{\bteta_k-\bteta_k'},
	\end{align}
	for all $1 \leq k \leq n$. Combining the two results, we get the statement of the lemma. This completes the proof.
\end{proof}
\subsubsection{Verification of Assumption~\ref{assume grad}}
Given a single sample $(\h_L, \h_{L-1},\z_{L-1})$ from the trajectory of a linear dynamical system, setting $\x_{L-1} = [\h_{L-1}^\top~\z_{L-1}^\top]^\top$, the single sample loss is given by,
\begin{align}
	&\Lc(\Bheta,(\h_L,\x_{L-1})) = \sum_{k=1}^n \Lc_k(\bteta_k,(\h_L[k],\x_{L-1})), \nn \\
	\text{where} \quad &\Lc_k(\bteta_k,(\h_L[k],\x_{L-1})):= \frac{1}{2}(\h_L[k] -\bteta_k^\top\x_{L-1})^2. \label{eqn:signle sample loss linear}
\end{align}
The following lemma shows that the gradient of the above loss is subexponential.
\begin{lemma}[Subexponential gradient]\label{lemma:subexp grad linear}
	Consider the same setup of Theorem~\ref{thrm:covariance of h_t bounds}. Let $\Lc_k(\bteta_k,(\h_L[k],\x_{L-1}))$ be as defined in~\eqref{eqn:signle sample loss linear} and $\gamma_+>0$ be as in lemma~\ref{lemma:grad dominance linear}. Then, at any point $\Bheta$, for all $1 \leq k \leq n$, we have 
	\[
	\te{\nabla\Lc_k(\bteta_k,(\h_L[k],\x_{L-1})) - \E[\nabla\Lc_k(\bteta_k,(\h_L[k],\x_{L-1}))]} \lesssim \gamma_+ \tn{\bteta_k - \bteta_k^\star} + \sigma \sqrt{\gamma_+}.
	\]
\end{lemma}
\begin{proof}
	Using standard bounds on the subgaussian norm of a random vector, we find that $\tsub{\x_{L-1}} \lesssim \sqrt{\bSi[\x_{L-1}]} \leq \sqrt{\gamma_+}$, where $\gamma_+>0$ is as defined in Lemma~\ref{lemma:grad dominance linear}. Combining this with $\tsub{\w_{L-1}[k]} \leq \sigma$, we get the following subexponential norm bound,
	\begin{align*}
		&\te{\nabla\Lc_k(\bteta_k,(\h_L[k],\x_{L-1})) - \E[\nabla\Lc_k(\bteta_k,(\h_L[k],\x_{L-1}))]} \\
		&\quad\quad\quad\quad\quad\quad\quad\quad\quad\quad~~
		= \te{(\x_{L-1}\x_{L-1}^\top - \E[\x_{L-1}\x_{L-1}^\top])(\bteta_k - \bteta_k^\star) - \x_{L-1}\w_{L-1}[k]}, \\
		&\quad\quad\quad\quad\quad\quad\quad\quad\quad\quad~~ \leq \te{(\x_{L-1}\x_{L-1}^\top - \E[\x_{L-1}\x_{L-1}^\top])(\bteta_k - \bteta_k^\star)} + \te{\x_{L-1}\w_{L-1}[k]}, \nn \\
		&\quad\quad\quad\quad\quad\quad\quad\quad\quad\quad~~ \lesssim \gamma_+ \tn{\bteta_k - \bteta_k^\star} + \sigma\sqrt{\gamma_+},
	\end{align*}
	where we get the last inequality from the fact that, the product of two subgaussian random variables results in a subexponential random variable with its subexponential norm bounded by the product of the two subgaussian norms.
\end{proof}

\subsubsection{Proof of Corollary~\ref{thrm:main thrm LDS}}
\begin{proof}
	Our proof strategy is based on verifying Assumptions~\ref{assum stability}, \ref{ass bounded}, \ref{assum grad dominance}, \ref{assume lip} and \ref{assume grad} for a stable linear dynamical system and then applying Theorem~\ref{thrm:main theorem separable}. Since, we already verified all the assumptions, we are ready to use Theorem~\ref{thrm:main theorem separable}. Before that, we find the values of the system related constants to be used in Theorem~\ref{thrm:main theorem separable} as follows. 
	
	\begin{remark}
		Consider the same setup of Theorem~\ref{thrm:covariance of h_t bounds}. For a stable linear dynamical system, with probability at least $1-4T\exp(-100n)$, for all $1 \leq t \leq T$, the scalars $C_{\bphi},D_{\bphi}$ take the following values:
		\begin{align}
			&\tn{\nabla_{\bteta_k}(\bteta_k^\top\x_t)} = \tn{\x_t} \leq c_0\sqrt{\beta_+(n+p)} = C_{\bphi}, \\
			&\|\nabla_{\x_t}\nabla_{\bteta_k}(\bteta_k^\top\x_t)\| = \|\Iden_{n+p}\| \leq 1 =  D_{\bphi},
		\end{align}
		where $\beta_+ = 1 \lor \max_{ 1\leq t\leq T}\lmx{\Gb_t\Gb_t^\top+\sigma^2\Fb_t\Fb_t^\top}$. Furthermore, the Lipschitz constant and the gradient noise coefficients take the following values:
		$L_{\Dc} = 2\gamma_+$, $K = c\gamma_+ $ and $\sigma_0 = c\sigma \sqrt{\gamma_+}$. Lastly, we also have $p_0 = 2\exp(-100(n+p))$. 
	\end{remark}
	Using these values, we get the following sample complexity bound for learning linear dynamical system via gradient descent,
	\begin{align}
		N \gtrsim \kappa^2 \log^2(3(2\gamma_+)N/\gamma_++3)(n+p) \Leftrightarrow  N \gtrsim \kappa^2 \log^2(6N+3)(n+p), \label{eqn:sample complex linear}
	\end{align}
	where $\kappa = \gamma_+/\gamma_-$ is an upper bound on the condition number of the covariance matrix $\bSi[\x_t]$. Similarly, the approximate mixing time for the linear dynamical system is given by,
	\begin{align}
		& L \geq 1 + \big[ \log(c_0(n+p)\sqrt{\beta_+}C_\rho\sqrt{N/(n+p)}) + \log(c/\sqrt{\gamma_+} \lor c\sqrt{\beta_+(n+p)}/\gamma_+)\big]/\log(\rho^{-1}) \nn \\
		&\Longleftarrow \quad L = \big\lceil 1 + \frac{\log(CC_\rho \beta_+ N(n+p)/\gamma_+)}{1-\rho} \big\rceil, \label{eqn:L complexity}
	\end{align}
	where, $C>0$ is a constant. Finally, given the trajectory length $T \gtrsim L(N+1)$, where $N$ and $L$ are given by~\eqref{eqn:sample complex linear} and \eqref{eqn:L complexity} respectively, starting from $\Bheta^{(0)} = 0$ and using learning rate $\eta = \gamma_-/(16\gamma_+^2)$~(in Theorem~\ref{thrm:main theorem separable}), with probability at least $1-4T\exp(-100n)-Ln\big(4+\log(\frac{\tf{\Bhetas} \sqrt{\gamma_+}}{\sigma})\big)\exp(-100(n+p))$ for all $1 \leq k \leq n$, all gradient descent iterates $\Bheta^{(\tau)}$ on $\Lch$ satisfy
	\begin{align}
		\tn{\bteta_{k}^{(\tau)} - \bteta_k^\star} & \leq (1 - \frac{\gamma_-^2}{128\gamma_+^2})^{\tau}\tn{\bheta_{k}^{(0)} - \bheta_k^\star} + \frac{5c}{\gamma_-} \sigma\sqrt{\gamma_+}\log(6N+3)\sqrt{\frac{n+p}{N}}. \label{eqn:GD convergence linear}
	\end{align}
	We remark that, choosing $N \gtrsim \kappa^2 \log^2(6N+3)(n+p)$, the residual term in \eqref{eqn:GD convergence linear} can be bounded as follows,
	\[
	\frac{5c}{\gamma_-} \sigma\sqrt{\gamma_+}\log(6N+3)\sqrt{\frac{n+p}{N}} \lesssim \sigma/\sqrt{\gamma_+}.
	\]
	Therefore, to ensure that Theorem~\ref{thrm:main theorem separable} is applicable, we assume that the noise is small enough, so that $\sigma \lesssim \sqrt{\gamma_+}\tf{\Bhetas}$~(we choose $\Bheta^{(0)} = 0$ and $r = \tf{\Bhetas}$). This completes the proof. 
\end{proof}

\subsection{Application to Nonlinear State Equations}\label{sec:proofs of RNN}	
\begin{lemma}\label{lemma:conc bound for yahya}
	Let $X$ be a non-negative random variable upper bounded by another random variable $Y$. Fix an integer $k>0$. Fix a constant $C>1+ k\log 3$ and suppose for some $B>0$ we have that $\Pro(Y\geq B(1+t))\leq \exp(-Ct^2)$ for all $t>0$. Then, the following bound holds,
	\[
	\E[X^k]\leq (2^k+2)B^k.
	\]
\end{lemma}

\begin{proof} Split the real line into regions $\Rc_i=\{x\bgl Bi\leq x\leq B(i+1)\}$. Observe that $\Pro(Y\in \Rc_0)+\Pro(Y\in \Rc_1)\leq 1$ and $\Pro(Y\in \Rc_{i+1})\leq \exp(-Ci^2)$ for $i\geq 1$. Then,
	\begin{align}
		\E[Y^k] &\leq \sum_{i=0}^\infty (B(i+1))^k \Pro(Y\in\Rc_i),\nn\\
		&\leq (2^k+1)B^k+\sum_{i=1}^\infty (i+2)^kB^k \exp(-Ci^2)\nn.
	\end{align}
	Next, we pick $C>0$ sufficiently large to satisfy $\exp(-Ci^2)(i+2)^k\leq \exp(-i^2)\leq \exp(-i)$. This can be guaranteed by picking $C$ to satisfy, for all $i$
	\begin{align}
		\exp((C-1)i^2)\geq (i+2)^{k}&\iff (C-1)i^2\geq k\log (i+2),\nn\\
		&\iff C\geq1+ \sup_{i\geq 1}\frac{k\log (i+2)}{i^2},\nn\\
		&\iff C\geq1+ k\log 3\nn.
	\end{align}
	Following this, we obtain $\sum_{i=1}^\infty (i+2)^kB^k \exp(-Ci^2)\leq B^k$. Thus, we find $\E[Y^k] \leq (2^k+2)B^k$.
\end{proof}

\subsubsection{Verification of Assumption~\ref{ass bounded}}
\begin{lemma}[Bounded states]\label{lemma:bounded states expectation}
	Suppose, the nonlinear system~\eqref{nonlinear state eqn} is $(C_\rho,\rho)$-stable and $\phi(0) = 0$. Suppose, $\z_t \distas \Nc(0,\Iden_n)$, $\w_t \distas \Nc(0,\sigma^2\Iden_n)$ and let $\beta_+:=C_\rho(1+\sigma)/(1-\rho)$. Then, starting from $\h_0 = 0$, for all $0 \leq t \leq T$, we have: 
	
	\noindent {\bf{(a)}} $\P(\tn{\h_t} \leq c\beta_+\sqrt{n}) \geq 1 - 4T\exp(-100n)$. 
	
	\noindent {\bf{(b)}} $\E[\tn{\h_t}^2] \leq  \beta_+^2n$.
	
	\noindent {\bf{(c)}} $\E[\tn{\h_t}^3] \leq C\beta_+^3(\log(2T)n)^{3/2}$.
\end{lemma}
\begin{proof}
	
	\noindent {\bf{(a)}} Given $\tsub{\z_t} \leq 1$ and $\tsub{\w_t} \leq \sigma$, we use Lemma~\ref{lemma:subgaussian vector bound} to obtain $\P(\tn{\z_t} \lesssim \sqrt{n}) \geq 1-2T\exp(-100n)$ and $\P(\tn{\w_t} \lesssim \sigma\sqrt{n}) \geq 1-2T\exp(-100n)$ for all $0\leq t \leq T-1$. Using these results along-with $(C_\rho,\rho)$-stability in Lemma~\ref{lemma:bounded states}, we get the desired bound on the Euclidean norm of the state vector $\h_t$. 
	
	\noindent {\bf{(b)}} Recall that $\h_0=0$. We claim that $\E[\tn{\h_t}^2] \leq \beta_+^2n(1-\rho^t)^2$, where $\beta_+:=C_\rho(1+\sigma)/(1-\rho)$. Note that, using standard results on the distribution of squared Euclidean norm of a Gaussian vector, we have $\E[\tn{\z_t}^2] = n$ and $\E[\tn{\w_t}^2] = \sigma^2 n$, which implies $\E[\tn{\z_t}] \leq \sqrt{n}$ and $\E[\tn{\w_t}] \leq \sigma \sqrt{n}$. Using this results, we show that $\h_1$ satisfies the following bound and obeys the induction
	\[
	\E[\tn{\h_1}^2] = \E[\tn{\phi(0) + \z_t +  \w_t}^2] \leq (1+\sigma^2)n \leq C_\rho^2(1+\sigma)^2n=\beta_+^2 n(1-\rho^1)^2.
	\]
	This implies $\E[\tn{\h_1}] \leq \beta_+ \sqrt{n}(1-\rho^1)$ as well. Suppose the bound holds until $t-1$, that is, $\E[\tn{\h_{t-1}}^2] \leq \beta_+^2 n(1-\rho^{t-1})^2$ ~(which also means $\E[\tn{\h_{t-1}}] \leq \beta_+ \sqrt{n} (1-\rho^{t-1})$). We now apply the induction as follows: First observe that $\E[\tn{\h_{t,L}}]$ obeys the same upper bound as $\E[\tn{\h_{L}}]$ by construction. To proceed, recalling~\eqref{eqn:truncation impact}, we get the following by induction
	\begin{align}
		\tn{\h_t-\h_{t,t-1}} &\leq C_\rho\rho^{t-1}\tn{\h_{1}}\nn \\ 
		\implies  \quad\;\; \tn{\h_t}&\leq C_\rho\rho^{t-1}\tn{\h_{1}}+\tn{\h_{t,t-1}}, \nn \\
		\implies \quad\;\; \tn{\h_t}^2&\leq (C_\rho\rho^{t-1}\tn{\h_{1}}+\tn{\h_{t,t-1}})^2, \nn \\
		\implies \E[\tn{\h_t}^2]&\leq C_\rho^2\rho^{2(t-1)} \E[\tn{\h_{1}}^2]+\E[\tn{\h_{t-1}}^2] + 2C_\rho \rho^{t-1}\E[\tn{\h_{1}}]\E[\tn{\h_{t-1}}], \nn \\
		&\leqsym{a} C_\rho^2\rho^{2(t-1)}(1+\sigma)^2n + \beta_+^2n(1-\rho^{t-1})^2 + 2nC_\rho \rho^{t-1}(1+\sigma)\beta_+(1-\rho^{t-1}),  \nn \\
		&\leqsym{b} \beta_+^2n(\rho^{2(t-1)}(1-\rho^1)^2+ (1-\rho^{t-1})^2 + 2 \rho^{t-1}(1-\rho^{t-1})(1-\rho^1)),\nn \\
		&=  \beta_+^2 n[\rho^{2t-2}(1+\rho^2-2\rho) + 1 + \rho^{2t-2} - 2\rho^{t-1} + (2\rho^{t-1}-2 \rho^{2t-2})(1-\rho)], \nn \\
		& = \beta_+^2 n (1 + \rho^{2t} - 2\rho^t), \nn \\
		& = \beta_+^2 n(1-\rho^t)^2, \label{eqn:states bound recursion non 2}
	\end{align}
	where we get (a) from the induction hypothesis and (b) from the bound on $\h_1$. This bound also implies $\E[\tn{\h_t}^2] \leq \beta_+^2n$ and completes the proof.
	
	\noindent {\bf{(c)}} Recall that, we have $\tsub{\z_t} \leq 1$, $\tsub{\w_t} \leq \sigma$, $\E[\tn{\z_t}] \leq \sqrt{n}$ and $\E[\tn{\w_t}] \leq \sigma \sqrt{n}$. Combining these bounds with standard concentration inequalities of a Guassian random vector, we have
	\begin{align}
		&\P(\tn{\z_t} \geq \E[\tn{\z_t}] + t) \leq \exp(-t^2/2) \quad \text{and} \quad \P(\tn{\w_t} \geq \E[\tn{\w_t}] + t) \leq \exp(-t^2/(2\sigma^2)), \nn \\
		&\implies \P(\tn{\z_t} \geq \sqrt{2c n}(1+t)) \leq \exp(-c n t^2), \label{eqn:u_t bound}\\
		&\text{and}\quad \P(\tn{\w_t} \geq \sigma\sqrt{2cn}(1+t)) \leq \exp(-cnt^2). \label{eqn:w_t bound}
	\end{align}
	To proceed, let $X = \tn{\h_t}$ and $Y=\sum_{\tau=0}^{t-1}C_\rho\rho^\tau(\tn{\z_t}+\tn{\w_t})$ and note that $X \leq Y$. Now, using \eqref{eqn:u_t bound}, \eqref{eqn:w_t bound} and union bounding over all $0\leq t \leq T-1$, we get the following high probability upper bound on $Y$, that is,
	\begin{align}
		&\P(Y \geq \sum_{\tau=0}^{t-1}C_\rho\rho^\tau\sqrt{2cn}(1+\sigma)(1+t)) \leq 2T\exp(-cnt^2), \nn \\
		\implies \quad &\P(Y \geq C_\rho\sqrt{10n\log(2T)}(1+t)(1+\sigma)/(1-\rho)) \leq \exp(-5nt^2), \nn
	\end{align}
	where we choose $c = 5\log(2T)$ to get the final concentration bound of $Y$. Finally using this bound in Lemma~\ref{lemma:conc bound for yahya}, we get
	\begin{align}
		\E[\tn{\h_t}^3] \leq 32\beta_+^3(\log(2T)n)^{3/2},  
	\end{align}
	where $\beta_+ = C_\rho(1+\sigma)/(1-\rho)$, as defined earlier. This completes the proof.
\end{proof}

\subsubsection{Verification of Assumption~\ref{assum grad dominance}}
\begin{theorem}\label{thrm:cov h_t bounds nonlinear}
	Suppose the nonlinear system~\eqref{nonlinear state eqn} satisfies $(C_\rho,\rho)$-stability. Suppose $\z_t \distas \Nc(0,\Iden_n)$ and $\w_t \distas \Nc(0,\sigma^2\Iden_n)$. Let $\beta_+$ be as in Lemma~\ref{lemma:bounded states expectation}. Then, the matrix $\E[\h_t\h_t^\top]$ satisfies
	\begin{align}
		(1+\sigma^2) \Iden_n \preceq \E[\h_t\h_t^\top]  \preceq \beta_+^2n \Iden_n. 
	\end{align}
\end{theorem}
\begin{proof}
	We first upper bound the matrix $\E[\h_t\h_t^\top]$ by bounding its largest singular value as follows,
	\begin{align}
		\E[\h_t\h_t^\top] \preceq \E[\|\h_t\h_t^\top\|]\Iden_{n} \preceq \E[\tn{\h_t}^2] \Iden_{n} \preceq \beta_+^2 n \Iden_{n}, \label{eqn:upper bound nonlinear}
	\end{align}
	where we get the last inequality by applying Lemma~\ref{lemma:bounded states expectation}. To get a lower bound, note that $\bSi[\h_t] = \E[\h_t\h_t^\top] - \E[\h_t]\E[\h_t]^\top$. Since, all of these matrices are positive semi-definite, we get the following lower bound,
	\begin{align}
		\E[\h_t\h_t^\top] \succeq \bSi[\h_t] = \bSi[\phi(\Bhetas\h_{t-1})+\z_t +\w_t] \succeq \bSi[\z_t+\w_t] = (1+\sigma^2)\Iden_{n}. \label{eqn:lower bound nonlinear}
	\end{align}
	Combining the two bounds gives us the statement of the lemma. This completes the proof.
\end{proof}
To verify Assumption~\ref{assum grad dominance} for the nonlinear system~\eqref{nonlinear state eqn}, denoting the $k_{\rm th}$ row of $\Bheta$ by $\bheta_k^{ \top}$, the auxiliary loss for the nonlinear system~\eqref{nonlinear state eqn} is given by,
\begin{align}
	\Lcp(\Bheta) = \sum_{k=1}^n \Lc_{k,\Dc}(\bheta_k) \quad \text{where}\quad \Lc_{k,\Dc}(\bheta_k) := \frac{1}{2}\E[(\h_L[k] -\phi(\bheta_k^\top\h_{L-1})-\z_{L-1}[k])^2]. \label{eqn:Lpop nonlinear}
\end{align} 
Using the derived bounds on the matrix $\E[\h_t\h_t^\top]$, it is straightforward to show that the auxiliary loss satisfies the following one-point convexity and smoothness conditions.
\begin{lemma}[One-point convexity \& smoothness] \label{lemma:grad dominance nonlinear} 
	Consider the setup of Theorem~\ref{thrm:cov h_t bounds nonlinear} and the auxiliary loss given by~\eqref{eqn:Lpop nonlinear}. Suppose, $\phi$ is $\gamma$-increasing~(i.e.~$\phi'(x)\geq \gamma>0$ for all $x\in\R$) and $1$-Lipschitz. Let $\beta_+$ be as in Lemma~\ref{lemma:bounded states expectation}. Then, for all $1\leq k \leq n$, the gradients $\nabla{\Lc_{k,\Dc}(\bheta_k)}$ satisfy,
	\begin{align*}
		\li\bheta_k-\bheta_k^\star,\nabla{\Lc_{k,\Dc}(\bheta_k)}\ri &\geq \gamma^2 (1+\sigma^2) \tn{\bheta_k-\bheta_k^\star}^2,\\
		\tn{\nabla{\Lc_{k,\Dc}(\bheta_k)}} &\leq \beta_+^2n\tn{\bheta_k-\bheta_k^\star}.
	\end{align*}
\end{lemma}
\begin{proof}
	Given two distinct scalars $a,b$ we define $\phi'(a,b) := \frac{\phi(a)-\phi(b)}{a-b}$. Observe that $0<\gamma \leq \phi'(a,b) \leq 1 $ because of the assumption that $\phi$ is $1$-Lipschitz and $\gamma$-increasing. Now, recalling the auxiliary loss $\Lc_{k,\Dc}$ from \eqref{eqn:Lpop nonlinear}, we have
	\begin{align}
		\nabla{\Lc_{k,\Dc}(\bheta_k)} &= \E[(\phi(\bheta_k^\top\h_{L-1})-\phi(\bheta_k^{\star \top}\h_{L-1})-\w_{L-1}[k])\phi'(\bheta_k^\top\h_{L-1})\h_{L-1}],  \nn\\
		&=\E[ \phi'(\bheta_k^\top\h_{L-1},\bheta_k^{\star \top}\h_{L-1}) \phi'(\bheta_k^\top\h_{L-1})(\bheta_k^\top\h_{L-1}-\bheta_k^{\star \top}\h_{L-1})\h_{L-1}] \nn \\ &-\E[\w_{L-1}[k]\phi'(\bheta_k^\top\h_{L-1})\h_{L-1}], \nn\\
		& = \E[ \phi'(\bheta_k^\top\h_{L-1},\bheta_k^{\star \top}\h_{L-1}) \phi'(\bheta_k^\top\h_{L-1})\h_{L-1}\h_{L-1}^\top(\bheta_k -\bheta_k^\star)], \label{eqn:nonlinear grad}
	\end{align}
	where $\E[\w_{L-1}[k]\phi'(\bheta_k^\top\h_{L-1})\h_{L-1}] = 0$ because $\h_{L-1}$ and $\w_{L-1}$ are independent and we have $\E[\w_{L-1}] = 0$. Next, using $\gamma$-increasing property of $\phi$, we get the following one-point convexity bound,
	\begin{align}
		\li \bheta_k -\bheta_k^\star, \nabla{\Lc_{k,\Dc}(\bheta_k)}\ri 
		& =  \li  \bheta_k -\bheta_k^\star,\E[ \phi'(\bheta_k^\top\h_{L-1},\bheta_k^{\star \top}\h_{L-1}) \phi'(\bheta_k^\top\h_{L-1})\h_{L-1}\h_{L-1}^\top(\bheta_k -\bheta_k^\star)]\ri, \nn \\
		&\geq \gamma^2 \li  \bheta_k -\bheta_k^\star,\E[\h_{L-1}\h_{L-1}^\top](\bheta_k -\bheta_k^\star) \ri, \nn \\
		&\geq \gamma^2 (1+\sigma^2) \tn{\bheta_k -\bheta_k^\star}^2. \label{eqn:OPC RNN}
	\end{align}
	Similarly, using $1$-Lipschitzness of $\phi$, we get the following smoothness bound,
	\begin{align}
		\tn{\nabla{\Lc_{k,\Dc}(\bheta_k)}} &= \tn{\E[ \phi'(\bheta_k^\top\h_{L-1},\bheta_k^{\star \top}\h_{L-1}) \phi'(\bheta_k^\top\h_{L-1})\h_{L-1}\h_{L-1}^\top(\bheta_k -\bheta_k^\star)]}, \nn  \\
		& \leq \E[ \|\phi'(\bheta_k^\top\h_{L-1},\bheta_k^{\star \top}\h_{L-1}) \phi'(\bheta_k^\top\h_{L-1})\h_{L-1}\h_{L-1}^\top\|]\tn{\bheta_k -\bheta_k^\star}, \nn \\
		& \leq \E[\|\h_{L-1}\h_{L-1}^\top\|]\tn{\bheta_k -\bheta_k^\star}. \nn \\
		& \leq \beta_+^2n \tn{\bheta_k -\bheta_k^\star}, \label{eqn:smoothness RNN}
	\end{align} 
	where $\beta_+$ is as defined in Lemma~\ref{lemma:bounded states expectation}. This completes the proof.
\end{proof}
\subsubsection{Verification of Assumption~\ref{assume lip}}
Let $\Sc = (\h_L^{(i)},\h_{L-1}^{(i)}, \z_{L-1}^{(i)})_{i=1}^N $ be $N$ i.i.d. copies of $(\h_L, \h_{L-1}, \z_{L-1})$ generated from $N$ i.i.d. trajectories of the system~\eqref{nonlinear state eqn}. Then, the finite sample approximation of the auxiliary loss $\Lcp$ is given by,
\begin{align}
	\Lch_\Sc(\Bheta) = \sum_{k=1}^n \Lch_{k,\Sc}(\bheta_k) \;\;\text{where}\;\; \Lch_{k,\Sc}(\bheta_k) := \frac{1}{2N} \sum_{i=1}^N (\h_L^{(i)}[k] - \phi(\bheta_k^\top\h_{L-1}^{(i)})-\z_{L-1}^{(i)}[k])^2. \label{eqn:finite Lpop nonLinear}
\end{align}
The following lemma states that both $\nabla\Lc_{k,\Dc}$ and $\nabla\Lch_{k,\Sc}$ are Lipschitz with high probability. 
\begin{lemma}[Lipschitz gradient]\label{lemma:lipschitz grad nonlinear}
	Consider the same setup of Theorem~\ref{thrm:cov h_t bounds nonlinear}. Consider the auxiliary loss $\Lc_{k,\Dc}$ and its finite sample approximation $\Lch_{k,\cal{S}}$ from \eqref{eqn:Lpop nonlinear} and \eqref{eqn:finite Lpop nonLinear} respectively. Suppose, $\phi$ has bounded first and second derivatives, that is, $|\phi'|,|\phi''| \leq 1$. Let $\beta_+$ be as in Lemma~\ref{lemma:bounded states expectation}. Then, with probability at least $1-4T\exp(-100n)$, for all pairs $\Bheta,\Bheta' \in \Bc^{n\times n}(\Bhetas,r)$ and for $1 \leq k \leq n$, we have
	\begin{align*}
		&\max(\tn{\nabla\Lc_{k,\Dc}(\bheta_k)-\nabla\Lc_{k,\Dc}(\bheta_k')},\tn{\nabla\Lch_{k,\cal{S}}(\bheta_k)-\nabla\Lch_{k,\cal{S}}(\bheta_k')}) \\
		&\quad\quad\quad\quad\quad\quad\quad\quad\quad\quad\quad\quad\quad\quad \lesssim ((1+\sigma)\beta_+^2n + r\beta_+^3n^{3/2}\log^{3/2}(2T))\tn{\bheta_k-\bheta_k'}.	
	\end{align*}
\end{lemma}
\begin{proof}
	To begin recall that, $\nabla{\Lc_{k,\Dc}(\bheta_k)} = \E[(\phi(\bheta_k^\top\h_{L-1})-\phi(\bheta_k^{\star \top}\h_{L-1}))\phi'(\bheta_k^\top\h_{L-1})\h_{L-1}]$. To bound the Lipschitz constant of the gradient $\nabla{\Lc_{k,\Dc}(\bheta_k)}$, we will upper bound the spectral norm of the Hessian as follows,
	\begin{align}
		\|\nabla^2{\Lc_{k,\Dc}(\bheta_k)}\| &= \|\E[(\phi(\bheta_k^\top\h_{L-1})-\phi(\bheta_k^{\star \top}\h_{L-1}))\phi''(\bheta_k^\top\h_{L-1})\h_{L-1}\h_{L-1}^\top] \nn \\
		&\quad + \E[\phi'(\bheta_k^\top\h_{L-1})\phi'(\bheta_k^\top\h_{L-1}) \h_{L-1}\h_{L-1}^\top]
		\|, \nn \\
		& \leq \E[\|\phi'(\bheta_k^\top\h_{L-1},\bheta_k^{\star \top}\h_{L-1})(\bheta_k^\top\h_{L-1}-\bheta_k^{\star \top}\h_{L-1})\phi''(\bheta_k^\top\h_{L-1})\h_{L-1}\h_{L-1}^\top\|] \nn \\
		&\quad + \E[\|\phi'(\bheta_k^\top\h_{L-1})\phi'(\bheta_k^\top\h_{L-1}) \h_{L-1}\h_{L-1}^\top\|], \nn \\
		& \leq \E[\|(\bheta_k^\top\h_{L-1}-\bheta_k^{\star \top}\h_{L-1})\h_{L-1}\h_{L-1}^\top\|]
		+ \E[\|\h_{L-1}\h_{L-1}^\top\|], \nn \\
		&\leq \tn{\bheta_k - \bheta_k^\star} \E[\tn{\h_{L-1}}^3] + \E[\tn{\h_{L-1}}^2], \nn \\
		&\lesssim  \beta_+^3(\log(2T)n)^{3/2}\tn{\bheta_k - \bheta_k^\star} + \beta_+^2n, \label{eqn:Lip constant non 1}
	\end{align}
	where we get the last inequality by applying Lemma~\ref{lemma:bounded states expectation}. Similarly, to bound the Lipschitz constant of the empirical gradient \[\nabla{\Lch_{k,\Sc}(\bheta_k)} = 1/N \sum_{i=1}^N(\phi(\bheta_k^\top\h_{L-1}^{(i)})-\phi(\bheta_k^{\star \top}\h_{L-1}^{(i)}) - \w_{L-1}^{(i)}[k])\phi'(\bheta_k^\top\h_{L-1}^{(i)})\h_{L-1}^{(i)},\] we bound the spectral norm of the Hessian of the empirical loss $\Lch_{k,\cal{S}}$ as follows,
	\begin{align}
		\|\nabla^2{\Lch_{k,\Sc}(\bheta_k)}\| \leq  &\frac{1}{N} \sum_{i=1}^N\|(\phi(\bheta_k^\top\h_{L-1}^{(i)})-\phi(\bheta_k^{\star \top}\h_{L-1}^{(i)}) - \w_{L-1}^{(i)}[k])\phi''(\bheta_k^\top\h_{L-1}^{(i)})\h_{L-1}^{(i)}(\h_{L-1}^{(i)})^\top\| \nn \\
		&+\frac{1}{N} \sum_{i=1}^N \|\phi'(\bheta_k^\top\h_{L-1}^{(i)}) \phi'(\bheta_k^\top\h_{L-1}^{(i)}) \h_{L-1}^{(i)}(\h_{L-1}^{(i)})^\top\|, \nn \\
		\leqsym{a} &\frac{1}{N} \sum_{i=1}^N[\|(\bheta_k^\top\h_{L-1}^{(i)}-\bheta_k^{\star \top}\h_{L-1}^{(i)}) \h_{L-1}^{(i)}(\h_{L-1}^{(i)})^\top\| + (1+ |\w_{L-1}^{(i)}[k]|)\|\h_{L-1}^{(i)}(\h_{L-1}^{(i)})^\top\|], \nn \\
		\leq&\frac{1}{N} \sum_{i=1}^N[\tn{\bheta_k-\bheta_k^\star}\tn{\h_{L-1}^{(i)}}^3 + (1+ |\w_{L-1}^{(i)}[k]|)\tn{\h_{L-1}^{(i)}}^2], \nn \\
		\lesssim & \beta_+^3 n^{3/2}\tn{\bheta_k-\bheta_k^\star} + (1+\sigma)\beta_+^2 n, \label{eqn:Lip constant non 2}
	\end{align}
	with probability at least $1-4T\exp(-100n)$, where we get (a) by using a similar argument as we used in the case of auxiliary loss while the last inequality comes from Lemma~\ref{lemma:bounded states expectation}. Combining the two bounds, gives us the statement of the lemma. This completes the proof.
\end{proof}
\subsubsection{Verification of Assumption~\ref{assume grad}}
Given a single sample $(\h_L, \h_{L-1}, \z_{L-1})$ from the trajectory of the nonlinear system~\eqref{nonlinear state eqn}, the single sample loss is given by,
\begin{align}
	&\Lc(\Bheta,(\h_L,\h_{L-1}, \z_{L-1})) = \sum_{k=1}^n \Lc_k(\bheta_k,(\h_L[k],\h_{L-1},\z_{L-1}[k])), \nn \\
	\text{where} \quad &\Lc_k(\bheta_k,(\h_L[k],\h_{L-1},\z_{L-1}[k])):= \frac{1}{2}(\h_L[k] -\phi(\bheta_k^\top\h_{L-1})-\z_{L-1}[k])^2. \label{eqn:single sample loss nonlinear}
\end{align}
Before stating a lemma on bounding the subexponential norm of the gradient of the single sample loss~\eqref{eqn:single sample loss nonlinear}, we will state an intermediate lemma to prove the Lipschitzness of the state vector. 
\begin{lemma}[Lipschitzness of the state vector]\label{lemma:lips of ht}
	Suppose the nonlinear system~\eqref{nonlinear state eqn} is  $(C_\rho,\rho)$-stable, $\z_t \distas \Nc(0,\Iden_n)$ and $\w_t \distas \Nc(0,\sigma^2\Iden_n)$. Let $\vb_t := [\z_t^\top~1/\sigma\w_t^\top]^\top$ and $\h_0 = 0$. Fixing all $\{\vb_i\}_{i \neq \tau}$~(i.e., all except $\vb_\tau$), $\h_{t+1}$ is $C_\rho\rho^{t-\tau}(1+\sigma^2)^{1/2}$ Lipschitz function of $\vb_\tau$ for $0 \leq \tau \leq t$.
\end{lemma}
\begin{proof}
	To begin, observe that $\h_{t+1}$ is deterministic function of the sequence $\{\vb_\tau\}_{\tau=0}^t$. Fixing all $\{\vb_i\}_{i \neq \tau}$, we denote $\h_{t+1}$ as a function of $\vb_\tau$ by $\h_{t+1}(\vb_\tau)$. Given a pair of vectors $(\vb_\tau,\vh_\tau)$, using $(C_\rho,\rho)$-stability of the nonlinear system~\eqref{nonlinear state eqn}, for any $t \geq \tau$, we have
	\begin{align}
		\tn{\h_{t+1}(\vb_\tau) - \h_{t+1}(\vh_\tau)} &\leq C_\rho\rho^{t-\tau}\tn{\h_{\tau+1}(\vb_\tau) - \h_{\tau+1}(\vh_\tau)}, \nn \\
		&\leq C_\rho\rho^{t-\tau}\tn{\phi(\Bhetas\h_\tau) + \z_\tau + \w_{\tau} - \phi(\Bhetas\h_\tau) - \hat{\z}_\tau - \wh_\tau}, \nn \\
		&\leq C_\rho\rho^{t-\tau}(\tn{\z_\tau  - \hat{\z}_\tau} + \sigma\tn{1/\sigma\w_{\tau} - 1/\sigma\wh_\tau}), \nn \\
		& \leqsym{a} C_\rho\rho^{t-\tau}(1+\sigma^2)^{1/2} (\tn{\z_{\tau}-\hat{\z}_\tau}^2 + 1/\sigma^2\tn{\w_\tau  - \hat{\w}_\tau}^2)^{1/2} ,\nn \\
		& \leq C_\rho\rho^{t-\tau}(1+\sigma^2)^{1/2} \tn{\vb_\tau -\vh_\tau},
		\label{eqn:cov ht bound RNN_1}
	\end{align} 
	where we get (a) by using Cauchy-Schwarz inequality. This implies $\h_{t+1}$ is $C_\rho\rho^{t-\tau}(1+\sigma^2)^{1/2}$ Lipschitz function of $\vb_\tau$ for $0 \leq \tau \leq t$ and completes the proof.
\end{proof}
We are now ready to state a lemma to bound the subexponential norm of the gradient of the single sample loss~\eqref{eqn:single sample loss nonlinear}.
\begin{lemma}[Subexponential gradient]\label{lemma:subexp grad nonlinear}
	Consider the same setup of Lemma~\ref{lemma:lips of ht}. Let $\Lc_k(\bheta_k,(\h_L[k],\h_{L-1},\z_{L-1}[k])$ be as in \eqref{eqn:single sample loss nonlinear} and $\beta_+ := C_\rho(1+\sigma)/(1-\rho)$. Suppose $|\phi'(x)| \leq 1$ for all $x \in \R$. Then, at any point $\Bheta$, for all $1 \leq k \leq n$, we have 
	\begin{align*}
		&\te{\nabla\Lc_k(\bheta_k,(\h_L[k],\h_{L-1},\z_{L-1}[k])) - \E[\nabla\Lc_k(\bheta_k,(\h_L[k],\h_{L-1},\z_{L-1}[k]))]} \nn \\
		&\quad\quad\quad\quad\quad\quad\quad\quad\quad\quad\quad\quad\quad\quad\quad\quad\quad\quad\quad\quad\quad\quad\quad\quad\quad\quad\quad\quad\quad \lesssim \beta_+^2 \tn{\bheta_k - \bheta_k^\star} + \sigma \beta_+.	
	\end{align*}
\end{lemma}
\begin{proof}
	We first bound the subgaussian norm of the state vector $\h_t$ following~\cite{oymak2019stochastic} as follows: Setting $\vb_t = [\z_t^\top~1/\sigma\w_t^\top]^\top$, define the vectors $\q_t :=
	[\vb_0^\top~\cdots~\vb_{t-1}^\top]^\top \in \R^{2nt}$ and $\hq_t :=
	[\vh_0^\top~\cdots~\vh_{t-1}^\top]^\top \in \R^{2nt}$. Observe that $\h_{t}$ is a deterministic function of $\q_t$, that is, $\h_t = f(\q_t)$ for some function $f$. To bound the Lipschitz constant of $f$, for all~(deterministic) vector pairs $\q_t$ and $\hq_t$, we find the scalar $L_f$ satisfying
	\begin{align}
		\tn{f(\q_t)-f(\hq_t)} \leq L_f\tn{\q_t-\hq_t}.
	\end{align}
	For this purpose, we define the vectors $\{\bb_i\}_{i=0}^t$ as follows:
	$
	\bb_i = [\vh_0^\top~\cdots~\vh_{i-1}^\top~\vb_i^\top~\cdots~\vb_{t-1}]^\top.
	$
	Observing that $\bb_0 = \q_t$ and $\bb_t = \hq_t$, we write the telescopic sum,
	\begin{align}
		\tn{f(\q_t)-f(\hq_t)} \leq \sum_{i=0}^{t-1} \tn{f(\bb_{i+1})-f(\bb_i)}.
	\end{align}
	Observe that $f(\bb_{i+1})$ and $f(\bb_i)$ differs only in $\vb_i,\vh_i$ terms in the argument. Hence, viewing $\h_t$ as a function of $\w_i$ and using the result of Lemma~\ref{lemma:lips of ht}, we have
	\begin{align}
		\tn{f(\q_t)-f(\hq_t)} &\leq \sum_{i=0}^{t-1} C_\rho\rho^{t-1-i}(1+\sigma^2)^{1/2} \tn{\vb_i -\vh_i} ,\nn \\
		&\leqsym{a} C_\rho (1+\sigma^2)^{1/2} \big(\sum_{i=0}^{t-1} \rho^{2(t-1-i)}\big)^{1/2} \underbrace{\big(\sum_{i=0}^{t-1} \tn{\vb_i -\vh_i}^2\big)^{1/2}}_{\tn{\q_t - \hq_t}}, \nn \\
		&\leqsym{b}  \frac{C_\rho(1+\sigma^2)^{1/2}}{(1-\rho^2)^{1/2}}\tn{\q_t - \hq_t},  \label{eqn:lipschitzness of ht RNN}
	\end{align}
	where we get (a) by applying the Cauchy-Schwarz inequality and (b) follows from $\rho < 1$. Setting $\beta_K = C_\rho(1+\sigma^2)^{1/2}/(1-\rho^2)^{1/2}$, we found that $\h_t$ is $\beta_K$-Lipschitz function of $\q_t$. Since $\vb_t \distas \Nc(0,\Iden_{2n})$, the vector $\q_t \distas \Nc(0,\Iden_{2nt})$. Since, $\h_t$ is $\beta_K$-Lipschitz function of $\q_t$, for any fixed unit length vector $\ab$, $\ab^\top\h_t$ is still $\beta_K$-Lipschitz function of $\q_t$. This implies $\tsub{\h_t - \E[\h_t]} \lesssim \beta_K$. Secondly, $\beta_K$-Lipschitz function of a Gaussian vector obeys the variance inequality $\var{\ab^\top\h_t} \leq \beta_K^2$ (page 49 of~\cite{ledoux2001concentration}), which implies the covariance bound $\bSi[\h_t] \preceq \beta_K^2\Iden_{n}$. Combining these results with $\tsub{\w_t[k]} \leq \sigma$, we get the following subexponential norm bound,
	\begin{align}
		&\te{\nabla\Lc_k(\bheta_k,(\h_L[k],\h_{L-1},\z_{L-1}[k])) - \E[\nabla\Lc_k(\bheta_k,(\h_L[k],\h_{L-1},\z_{L-1}[k]))]} \nn\\
		& \quad\quad\quad\quad\quad\quad\leq \|\phi'(\bheta_k^\top\h_{L-1},\bheta_k^{\star \top}\h_{L-1}) \phi'(\bheta_k^\top\h_{L-1})\h_{L-1}\h_{L-1}^\top(\bheta_k -\bheta_k^\star) \nn\\ &\quad\quad\quad\quad\quad\quad-\E[\phi'(\bheta_k^\top\h_{L-1},\bheta_k^{\star \top}\h_{L-1}) \phi'(\bheta_k^\top\h_{L-1})\h_{L-1}\h_{L-1}^\top(\bheta_k -\bheta_k^\star)]\|_{\psi_{1}} \nn\\
		& \quad\quad\quad\quad\quad\quad+ \|\phi'(\bheta_k^\top\h_{L-1})\w_{L-1}[k]\h_{L-1}\|_{\psi_{1}}, \nn\\
		& \quad\quad\quad\quad\quad\quad\lesssim \beta_K^2 \tn{\bheta_k - \bheta_k^\star} + \sigma \beta_K, \nn \\
		& \quad\quad\quad\quad\quad\quad \lesssim \beta_+^2 \tn{\bheta_k - \bheta_k^\star} + \sigma \beta_+,
	\end{align}
	where we get the last two inequalities from the fact that the product of a bounded function~($\phi$ is $1$-Lipschitz because $|\phi'(x)| \leq 1$ for all $x \in \R$) with a subgaussian/subexponential random vector is still a subgaussian/subexponential random vector. This completes the proof.  
\end{proof}

\subsubsection{Proof of Corollary~\ref{thrm:main thrm RNN}}
\begin{proof}
	We have verified Assumptions~\ref{ass bounded}, \ref{assum grad dominance}, \ref{assume lip} and \ref{assume grad} for the nonlinear system~\ref{nonlinear state eqn}. Hence, we are ready to use Theorem~\ref{thrm:main theorem separable} to learn the dynamics $\Bhetas$ of the nonlinear system~\eqref{nonlinear state eqn} . Before that, we find the values of the system related constants to be used in Theorem~\ref{thrm:main theorem separable} as follows. 
	\begin{remark}
		Consider the same setup of Lemma~\ref{lemma:lips of ht}. Let $ \beta_+ \geq \beta_K >0$ be as defined in Lemmas~\ref{lemma:bounded states expectation} and \ref{lemma:subexp grad nonlinear} respectively. Then, with probability at least $1-4T\exp(-100n)$, for all $1 \leq t \leq T$, $\Bheta \in \Bc^{n \times n}(\Bhetas,r)$ and $1 \leq k \leq n$, the scalars $C_{\phi},D_{\phi}$ take the following values.
		\begin{align*}
			&\tn{\nabla_{\bheta_k}\phi(\bheta_k^\top\h_t)} = \tn{\phi'(\bheta_k^\top\h_t)\h_t} \leq \tn{\h_t} \lesssim \beta_+ \sqrt{n} = C_{\phi}, \\
			&\|\nabla_{\h_t}\nabla_{\bheta_k}\phi(\bheta_k^\top\h_t)\| =  \|\phi'(\bheta_k^\top\h_t)\Iden_n + \phi''(\bheta_k^\top\h_t)\h_t\bheta_{k}^\top\| \lesssim 1 + \beta_+\sqrt{n}\tn{\bheta_k} \lesssim 1+\tf{\Bhetas}\beta_+\sqrt{n} = D_{\phi} 
		\end{align*}
		where without loss of generality we choose $\Bheta^{(0)} = 0$ and $r = \tf{\Bhetas}$. Furthermore, the Lipschitz constant and the gradient noise coefficients take the following values:
		$L_{\Dc} =  c((1+\sigma)\beta_+^2n + \tf{\Bhetas}\beta_+^3n^{3/2}\log^{3/2}(2T))$, $K = c\beta_+^2$ and $\sigma_0 = c\sigma \beta_+$. Lastly, we also have $p_0 = 4T \exp(-100n)$.
	\end{remark}
	Using these values, we get the following sample complexity bound for learning nonlinear system~\eqref{nonlinear state eqn} via gradient descent,
	\begin{align}
		N &\gtrsim \frac{\beta_+^4}{\gamma^4(1+\sigma^2)^2}\log^2(3((1+\sigma)\beta_+^2n + \tf{\Bhetas}\beta_+^3 n^{3/2}\log^{3/2}(2T))N/\beta_+^2+3)n, \nn \\
		\implies \quad N &\gtrsim \frac{C_\rho^4}{\gamma^4(1-\rho)^4} \log^2(3(1+\sigma)n + 3\tf{\Bhetas}\beta_+ n^{3/2}\log^{3/2}(2T)N+3)n, \label{eqn:sample complex nonlinear}
	\end{align}
	where $\frac{\beta_+^2}{1+\sigma^2} \leq \frac{C_\rho^2(1+\sigma)^2/(1-\rho)^2}{(1+\sigma)^2/2} =\frac{2C_\rho^2}{(1-\rho)^2}$ is an upper bound on the condition number of the covariance matrix $\bSi[\h_t]$. Similarly, the approximate mixing time of the nonlinear system~\eqref{nonlinear state eqn} is given by,
	\begin{align}
		& L \geq 1 + \big[ \log(c_0 C_\rho \beta_+ (1+\tf{\Bhetas}\beta_+\sqrt{n})n\sqrt{N/n}) + \log(c/\beta_+ \lor c\sqrt{n}/\beta_+)\big]/\log(\rho^{-1}), \nn \\
		&\Longleftarrow \quad L = \big\lceil 1 + \frac{\log(C C_\rho(1+\tf{\Bhetas}\beta_+) Nn)}{1-\rho} \big\rceil, \label{eqn:L complexity non}
	\end{align}
	where $C>0$ is a constant. Finally, given the trajectory length $T \gtrsim L(N+1)$, where $N$ and $L$ are as given by~\eqref{eqn:sample complex nonlinear} and \eqref{eqn:L complexity non} respectively, starting from $\Bheta^{(0)} = 0$ and using the learning rate  $\eta = \frac{\gamma^2(1+\sigma^2)}{16\beta_+^4n^2} \geq \frac{\gamma^2(1-\rho)^4}{32C_\rho^4(1+\sigma)^2n^2}$, with probability at least $1-Ln\big(4T+\log(\frac{\tf{\Bhetas} C_\rho(1+\sigma)}{\sigma(1-\rho)})\big)\exp(-100n)$ for all $1 \leq k \leq n$, all gradient descent iterates $\Bheta^{(\tau)}$ on $\Lch$ satisfy
	\begin{align}
		\tn{\bheta_{k}^{(\tau)} - \bheta_k^\star} & \leq \big(1 - \frac{\gamma^4(1+\sigma^2)^2}{128\beta_+^4n^2}\big)^{\tau}\tn{\bheta_{k}^{(0)} - \bheta_k^\star} \nn \\
		&+ \frac{5c}{\gamma^2(1+\sigma^2)}\sigma\beta_+ \log(3(1+\sigma)n + 3\tf{\Bhetas}\beta_+ n^{3/2}\log^{3/2}(2T)N+3)\sqrt{\frac{n}{N}}. \nn \\
		& \leq \big(1 - \frac{\gamma^4(1-\rho)^4}{512C_\rho^4n^2}\big)^{\tau}\tn{\bheta_{k}^{(0)} - \bheta_k^\star} \nn \\
		&+ \frac{10cC_\rho}{\gamma^2(1-\rho)}\sigma\log(3(1+\sigma)n + 3C_\rho(1+\sigma) \tf{\Bhetas} n^{3/2}\log^{3/2}(2T)N/(1-\rho)+3)\sqrt{\frac{n}{N}},\nn	
	\end{align}
	where we get the last inequality by plugging in the value of $\beta_+ = C_\rho\sigma/(1-\rho)$ and using the inequality $(1+\sigma^2) \geq \frac{(1+\sigma)^2}{2}$. 
	We remark that, choosing $N \gtrsim \frac{C_\rho^4}{\gamma^4(1-\rho)^4} \log^2(3(1+\sigma)n + 3C_\rho(1+\sigma) \tf{\Bhetas} n^{3/2}\log^{3/2}(2T)N/(1-\rho)+3)n$, the residual term in the last inequality can be bounded as,
	\[
	\frac{10cC_\rho}{\gamma^2(1-\rho)} \log(3(1+\sigma)n + 3 C_\rho(1+\sigma)\tf{\Bhetas} n^{3/2}\log^{3/2}(2T)N/(1-\rho)+3)\sqrt{\frac{n}{N}} \lesssim  \sigma.
	\] 
	Therefore, to ensure that Theorem~\ref{thrm:main theorem separable} is applicable, we assume that $\sigma \lesssim \tf{\Bhetas}$ (where we choose $\Bheta^{(0)} = 0$ and $r = \tf{\Bhetas}$). This completes the proof.
\end{proof}

\end{document}